\documentclass{article}

\usepackage[main, final]{neurips_2026}

\usepackage[utf8]{inputenc} 
\usepackage[T1]{fontenc}    
\usepackage[colorlinks,
            linkcolor=blue,
            anchorcolor=blue,
            citecolor=teal,
            ]{hyperref}       
\usepackage{url}            
\usepackage{booktabs}       
\usepackage{amsfonts}       
\usepackage{nicefrac}       
\usepackage{microtype}      
\usepackage{xcolor}         

\usepackage{amsmath,amssymb,amsthm}
\usepackage{graphicx}
\usepackage{enumitem}
\usepackage{multirow}
\usepackage{subcaption}
\usepackage{wrapfig}

\theoremstyle{plain}
\newtheorem{theorem}{Theorem}[section]

\newtheorem{lemma}[theorem]{Lemma}

\theoremstyle{definition}

\newtheorem{example}[theorem]{Case Study}
\theoremstyle{remark}

\graphicspath{{./plots/}}
\newcommand{\real}{\mathbb{R}}
\newcommand{\bfA}{\mathbf{A}}

\newcommand{\bfB}{\mathbf{B}}
\newcommand{\bfE}{\mathbf{E}}
\newcommand{\bfI}{\mathbf{I}}
\newcommand{\bfM}{\mathbf{M}}
\newcommand{\bfP}{\mathbf{P}}

\newcommand{\bfv}{\mathbf{v}}
\newcommand{\bfW}{\mathbf{W}}
\newcommand{\bfx}{\mathbf{x}}
\newcommand{\bfX}{\mathbf{X}}
\newcommand{\bfY}{\mathbf{Y}}
\newcommand{\loss}{\mathcal{L}}

\newcommand{\rank}{\mathrm{rank}}

\newcommand{\tr}{\mathrm{Tr}}
\newcommand{\fro}{\mathrm{F}}
\newcommand{\diag}{\mathrm{diag}}
\newcommand{\dd}{\mathsf{d}}
\newcommand{\dt}[1]{\frac{{\mathsf d} {#1}}{ {\mathsf d} t}}

\title{ANCRe: Adaptive Neural Connection Reassignment for Efficient Depth Scaling}

\author{%
  Yilang Zhang\thanks{Equal contribution. Work completed while Y. Zhang was a PhD student at University of Minnesota.} \\
  ML Research \\
  Morgan Stanley\\
  New York, NY 10019 \\
  \hspace{-0.5cm}\texttt{yilang.zhang@morganstanley.com} \hspace{-0.5cm} \\
  \And
  Bingcong Li$^*$ \\
  Department of CS \\
  ETH Z\"urich \\
  8092 Z\"urich, Switzerland \\
  \texttt{bingcong.li@inf.ethz.ch}  \\
  \AND
  Niao He \\
  Department of CS \\
  ETH Z\"urich \\
  8092 Z\"urich, Switzerland \\
  \texttt{niao.he@inf.ethz.ch} \\
  \And
  Georgios B. Giannakis \\
  Department of ECE \\
  University of Minnesota \\
  Minneapolis, MN 55455 \\
  \texttt{georgios@umn.edu} \\
}

\begin{document}

\maketitle

\begin{abstract}
Scaling network depth has been a central driver behind the success of modern foundation models, yet recent investigations suggest that deep layers are often underutilized. This paper revisits the default mechanism for deepening neural networks, namely residual connections, from an optimization perspective. On a motivating example, rigorous analysis proves that the layout of residual connections can fundamentally shape convergence behavior, and even induces an \emph{exponential} gap in convergence rates. Prompted by this insight, we introduce adaptive neural connection reassignment (ANCRe), a principled and lightweight framework that parameterizes and learns residual connectivities from the data. ANCRe adaptively reassigns residual connections with negligible computational and memory overhead ($<1\%$), while enabling more effective utilization of network depth. Extensive numerical tests across pre-training of large language models, diffusion models, and deep ResNets demonstrate consistently accelerated convergence, boosted performance, and enhanced depth efficiency over conventional residual connections.
\end{abstract}

\section{Introduction}
Foundation models have demonstrated remarkable success across a broad spectrum of domains and impactful applications. For instance, large language models (LLMs) exhibit nearly human-level proficiency in various tasks, including conversational interaction~\cite{GPT4}, automated code generation~\cite{Codex}, and complex mathematical reasoning~\cite{Minerva}. Diffusion models have revolutionized vision tasks by enabling high-fidelity and controllable image synthesis~\cite{DDPM,DDIM,score-matching}. Recent work has further extended foundation models to multimodal settings, where joint representations are learned from heterogeneous data modalities~\cite{CLIP,PaLM-E}. 

One of the key factors underlying these advances is the increasing capacity of modern backbone architectures, with a particularly noticeable trend toward greater depth. For instance, the Llama 3.1 family employs 32, 80, and 126 Transformer layers for its 8B, 70B, and 405B variants, respectively~\cite{llama3}. Likewise, Diffusion Transformers (DiTs) scale in depth from 12 (DiT-S) up to 28 layers (DiT-XL)~\cite{DiT}. Moreover,~\cite{depth-efficient} shows a clear positive correlation between depth and performance across 132 open-source LLMs (see their Figure 1). This trend is also supported by theories. It is proved in~\cite{telgarsky2015representation} that polynomially deep networks can express functions that would require exponential width in shallow ones. 

Despite the documented success of deep networks, scaling model depth can be \textit{less efficient} than it first appears. For example,~\cite{depth-efficient} shows that skipping an early layer in the Llama 3.1 70B has a remarkably greater impact on the outputs of subsequent layers than omitting a deep layer, and that deeper layers often behave as near-identity mappings. Since the identity function is essentially available ``for free'', this suggests that late layers are highly underutilized. Complementary evidence also occurs in multimodal foundation models, where the most informative vision embeddings are frequently found in intermediate layers of the Perception Encoder rather than the final layer~\cite{perception-encoder}. Collectively, these observations reveal that the representational potential of depth is not fully exploited yet. 

Given that residual (skip) connections are the default strategy and dominant mechanism for scaling model depth, this work revisits their design to enable more efficient depth scaling. Residual connections were proposed in~\cite{srivastava2015highway,ResNet} to avoid vanishing and exploding gradients, and they have become almost universal across architectures. For example, the backbone architecture of LLMs, i.e., the Transformer~\cite{Transformer}, introduces residual connections around each self-attention and feedforward network modules. 
From an optimization perspective, residual connections are credited with smoothing the loss landscape, which facilitates training by improving the (local) condition number~\cite{li2018visualizing}. 

This work continues on the optimization perspective of residual connections, and demonstrates an intuitive yet often overlooked factor: \emph{where} residual connections are placed within a deep architecture, i.e., the residual \textit{topology}, can play a crucial role in optimization. We provide quantitative theory showing that different topologies can induce an \textit{exponential} gap in convergence for deep linear neural networks. This large gap motivates a principled redesign of residual layout. To this end, we term our approach adaptive neural connection reassignment (ANCRe\footnote{It coincides with “anchor” in French.}), which learns an optimal residual topology from data. Our method not only achieves a linear convergence rate for deep linear networks, but also integrates seamlessly into modern architectures including LLMs, DiTs, and ResNets with consistent empirical gains. In a nutshell, our contributions are as follows:
\begin{itemize}[leftmargin=0.4cm, topsep=0cm, itemsep=0.5\parskip, parsep=0.5\parskip]
    \item A theoretical characterization of residual connection topology is established using deep linear neural networks, showing that different layouts can induce exponential gaps in convergence rates. 
    \item ANCRe is proposed to parameterize residual connections and learn a data-driven topology on the fly with appropriate normalization. It incurs negligible computational and memory overhead. 
    \item Extensive numerical evaluations are conducted under varying data modalities and network depths to systematically examine the efficiency of ANCRe. As an illustrative example, ANCRe achieves a 1.85$\times$ training speedup on LLaMA-1B over conventional residual topology. 
\end{itemize}

\section{Related work}
\label{sec:related-work}
\textbf{Residual connections.} 
Residual (skip) connections are a primary mechanism for scaling neural networks to greater depth~\cite{srivastava2015highway,ResNet,he2016identity}. 
They were rapidly popularized in computer vision, with variants such as~\cite{ReZero} and DenseNet~\cite{huang2017densely}, and have since become a standard component of CNN-based architectures; see e.g.,~\cite{zagoruyko2016wide,xie2017aggregated}.
In contrast, residual connections in large language models (LLMs) have remained relatively stable. Both Transformers and their recent variants adopt the same residual topology~\cite{Transformer,team2025gemma,llama,qwen3,DiT}, while certain shortcuts can be introduced to mitigate over-smoothing~\cite{NeuTRENO} or reduce KV cache size~\cite{ResFormer}. 
More recently, hyper-connections (HC)~\cite{HC} and manifold-constrained hyper-connections (mHC)~\cite{mHC} were proposed for foundation models as architectural alternatives.
Our work is orthogonal to these lines of research: while they focus on intra-layer designs, we emphasize the \textit{inter-layer topology}. 

\textbf{Understanding of residual connections.} 
Residual connections are known to stabilize training by mitigating vanishing and exploding gradients~\cite{haber2017stable}, and make gradients in deep networks less ``shattered'' (i.e., less like white noise)~\cite{balduzzi2017shattered}. Empirical visualizations of ResNets also suggest that residual connections lead to a smoother loss landscape~\cite{li2018visualizing}. 
Theoretical understandings are often acquired by contrasting deep linear neural networks~\cite{converge-analysis-LNN} with their residual counterparts. For example, the convergence of deep linear networks can degrade exponentially with depth~\cite{exp-converge-LNN}, whereas residual connections alleviate this slowdown~\cite{wu2019global}. Moreover, residual connections can relax network width requirements for global convergence in certain regimes, as by comparing results in~\cite{du2019width,zou2020global}. The loss landscape of deep linear residual networks has also been studied in~\cite{hardt2016identity}. Our work enriches this line of results by showing that the \emph{topology} of residual connections can affect convergence \textit{exponentially}. More importantly, we translate this insight into a practical design that yields consistent improvements in modern architectures. Due to space limitation, other related work is deferred to Appendix~\ref{apdx:related-work}. 

\textbf{Notation.} Bold lowercase (capital) letters denote vectors (matrices); $\| \cdot \|$ and $\| \cdot \|_\fro$ stand for $\ell_2$- and Frobenius-norm.

\begin{figure}[t]
   \begin{center}
  \begin{subfigure}[t]{0.325\textwidth}
  	\centerline{\includegraphics[width=\textwidth]{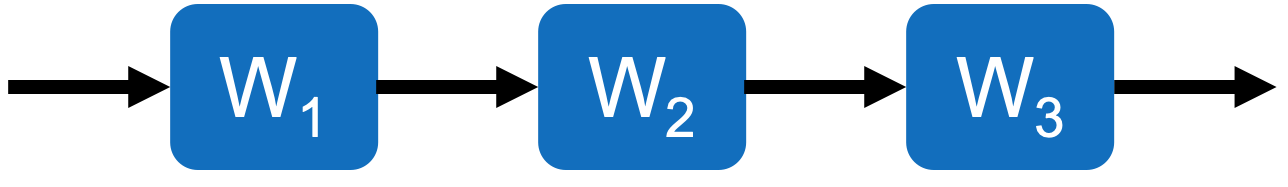}}
  	\caption{No residual connection}
  	\label{subfig:LNN-noskip}
  \end{subfigure}
  \begin{subfigure}[t]{0.325\textwidth}
  	\centerline{\includegraphics[width=\textwidth]{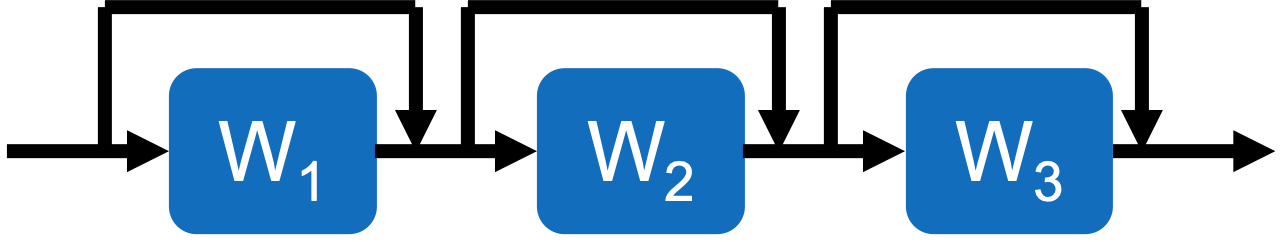}}
  	\caption{Cascaded residual connections}
  	\label{subfig:LNN-cascaded}
  \end{subfigure}
  \begin{subfigure}[t]{0.325\textwidth}
  	\centerline{\includegraphics[width=\textwidth]{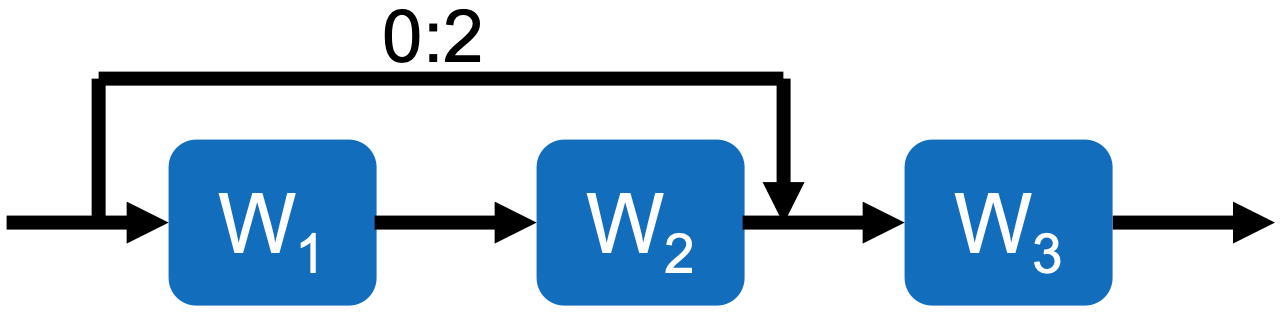}}
  	\caption{Residual connection 0:2}
  	\label{subfig:LNN-1skip}
  \end{subfigure}
    \vspace{-.1cm}
  \caption{Visualization of linear neural network (LNN) of $K = 3$ layers.}
  \label{fig:LNN-visual}
  \end{center}
  \vspace{-.1cm}
\end{figure}

\begin{figure}[t]
   \begin{center}
  \begin{subfigure}[t]{0.33\textwidth}
  	\centerline{\includegraphics[width=\textwidth]{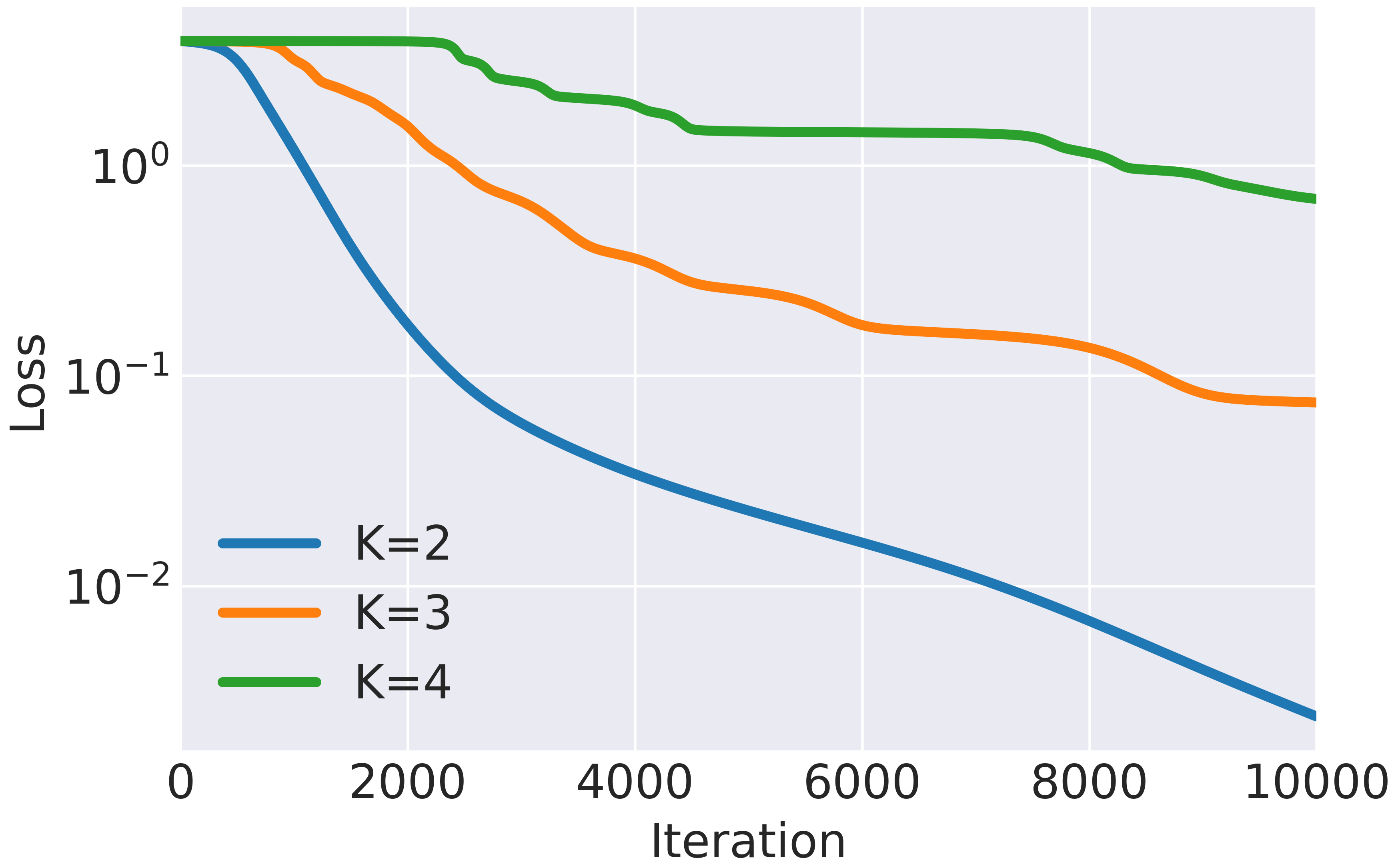}}
  	\caption{Varying depth $K$}
  	\label{subfig:vary-depth}
  \end{subfigure}
  \begin{subfigure}[t]{0.32\textwidth}
  	\centerline{\includegraphics[width=\textwidth]{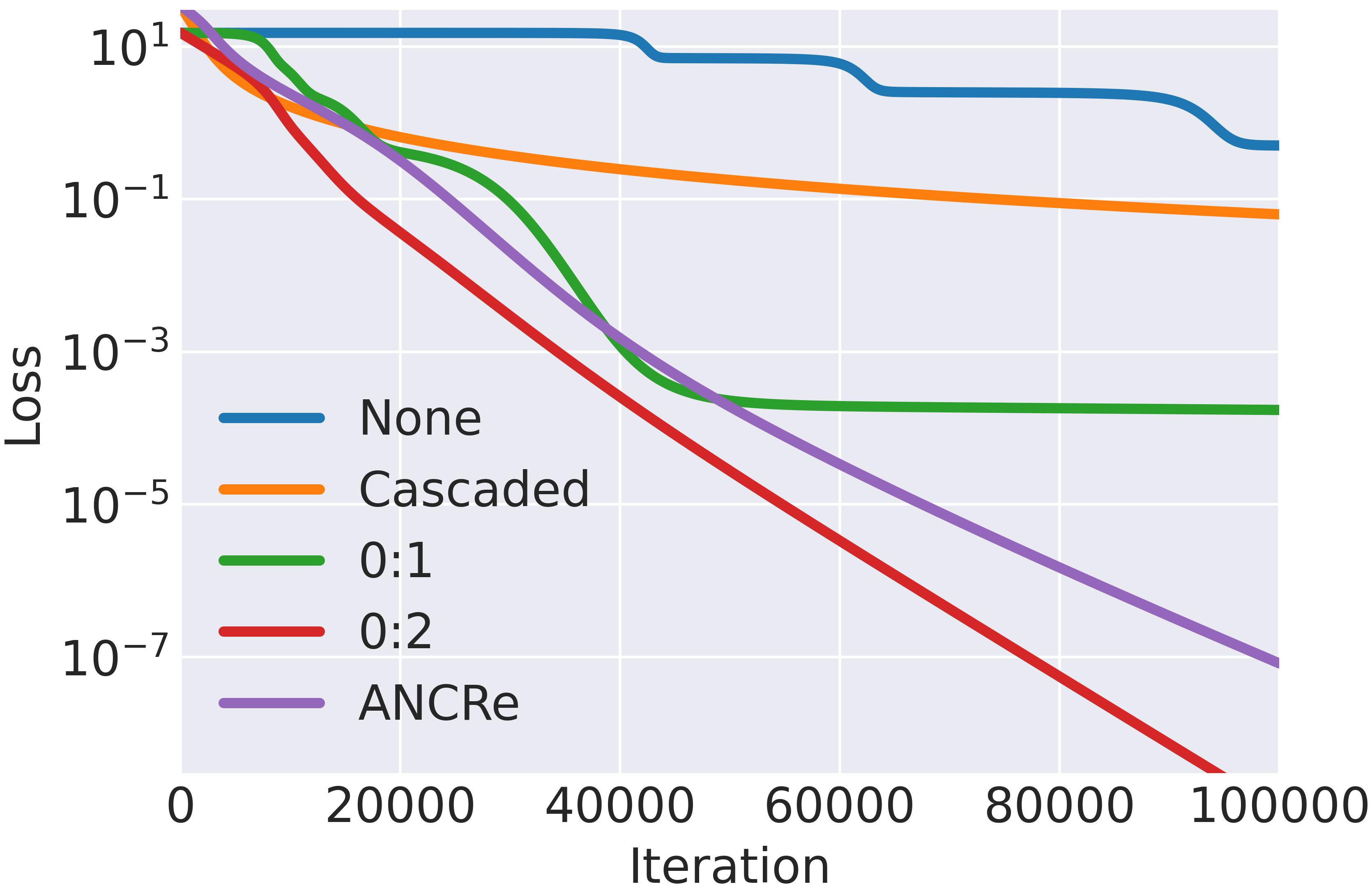}}
  	\caption{$K = 3$ with varying shortcuts}
  	\label{subfig:3layer-vary-place}
  \end{subfigure}
  \begin{subfigure}[t]{0.32\textwidth}
  	\centerline{\includegraphics[width=\textwidth]{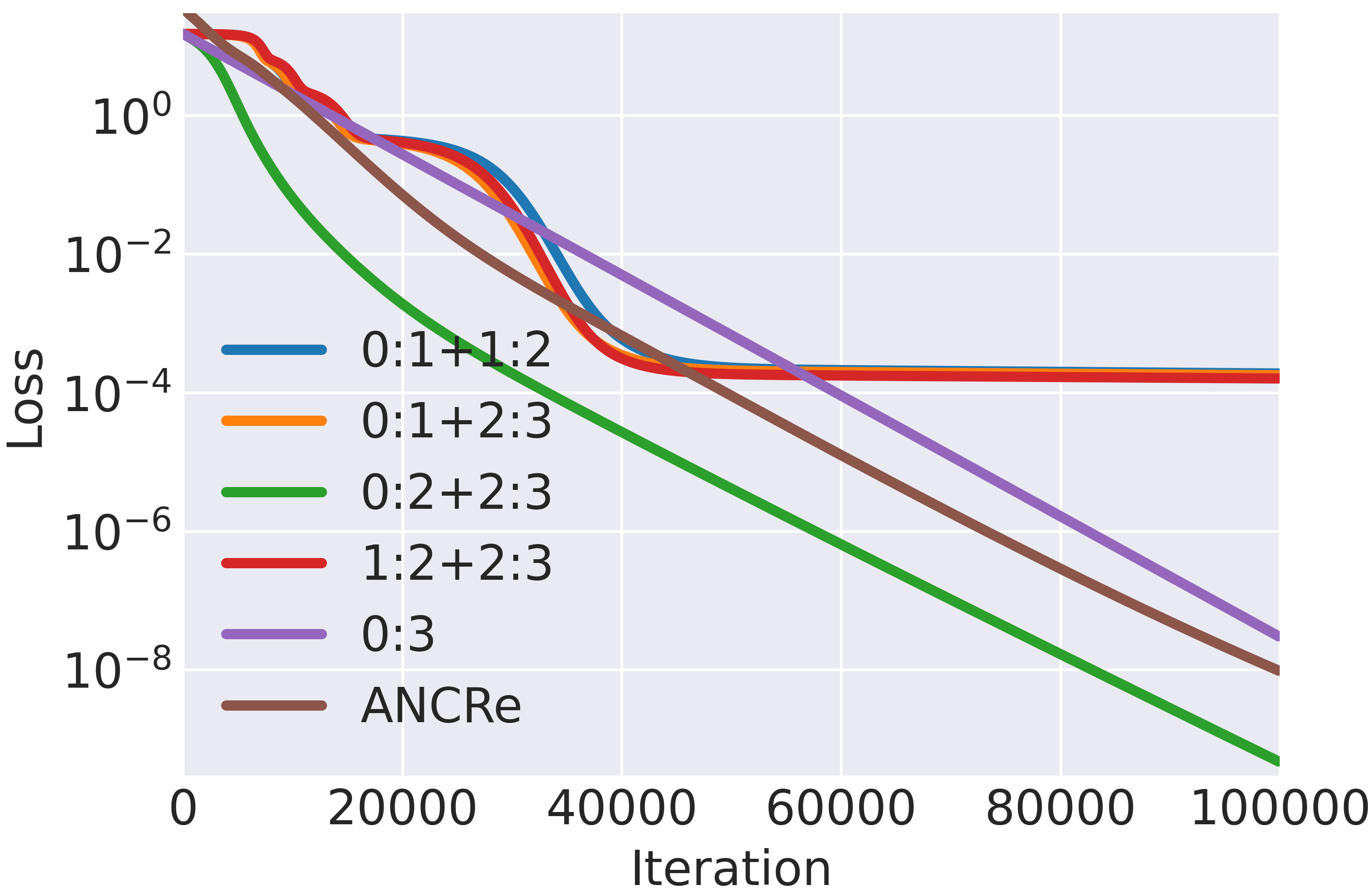}}
  	\caption{$K = 4$ with two shortcuts}
  	\label{subfig:4layer-vary-place}
  \end{subfigure}
\vspace{-.1cm}
  \caption{Convergence comparison of LNN under varying setups.}
  \label{fig:LNN-converge}
  \end{center}
  \vspace{-.1cm}
\end{figure}

\section{Residual topology matters: a case study}
\label{sec:case-study}
Popularized by ResNet~\cite{ResNet}, residual connections allow each layer $k$ to learn a residual mapping $r_k$ relative to its input $\bfx_k$; i.e., $g_k(\bfx_k) := r_k(\bfx_k) + \bfx_k$. By providing an identity shortcut, 
the loss landscape is more well-behaved~\cite{li2018visualizing}, 
thereby enabling training deep networks with hundreds of layers. 
This architecture has been extensively leveraged in foundation models~\cite{Transformer,ViT,CLIP} to train the network at scale.  
Residual connections are typically arranged in a cascaded structure as in Figure~\ref{subfig:LNN-cascaded}, with each shortcut bypassing a single layer or block.

Despite the popularity of residual connections, their topology, i.e., the shortcut layout, is often fixed by default. This section revisits this simple design, and reveals a somewhat surprising message: in a theoretically convenient case study, \textit{topology alone can provably induce exponentially different optimization behaviors}.


\subsection{Linear neural networks with residual connections}
Consider a commonly adopted prototype model, deep linear neural network (LNN), $f(\bfx) := \prod_{k=1}^K \bfW_k \bfx = \bfW_K\ldots \bfW_2 \bfW_1 \bfx$ as sketched by Figure~\ref{subfig:LNN-noskip}, where $\bfW_k \in \real^{d \times d}$ denotes the learnable weights per layer $k$. This idealized model is often adopted for understanding the optimization and analyzing convergence of deep networks; see \cite{converge-analysis-LNN,exp-converge-LNN} and more in Section~\ref{sec:related-work}. In addition, consider also LNN augmented with residual connections. For brevity, we henceforth use $i$:$j$ to represent the shortcut bridging the outputs of layers $i$ and $j$, where $0 \le i < j \le K$, and $0$ represents the input of the first layer; see Figure~\ref{subfig:LNN-1skip} for an example. While NN layers can be highly nonlinear, the residual topology itself remains an additive linear routing mechanism. We therefore study LNN to isolate the role of residual topology from multiple interacting components of modern architectures. 

\begin{example}
\label{ex:LNN}
Given a $K$-layer LNN and input-output pair $\bfX, \bfY \in \real^{d \times n}$ ($n \ge d$), define regression objective
\begin{equation*}
	\min_{\{\bfW_k \}_k} \loss (\{ \bfW_k \}_{k=1}^K) := \frac{1}{2} \bigg\| \prod_{k=1}^K \bfW_k \bfX - \bfY \bigg\|_\fro^2.
\end{equation*}
Using standard gradient descent (GD) for training, Figure~\ref{subfig:vary-depth} outlines the training loss evolution with depth $K = 2, 3, 4$, with vertical axis in log-scale. Without the aid of residual connection, the convergence slows down when the network grows deeper. Figure~\ref{subfig:3layer-vary-place} depicts the convergence behavior of a 3-layer LNN under different residual connection topologies. While the conventional cascaded layout accelerates the convergence relative to no connection, it is not an optimal choice in this setup. In contrast, incorporating a single residual connection 0:1 or 0:2 yields remarkably faster convergence, where the latter further exhibits an exponential improvement over the former. Figure~\ref{subfig:4layer-vary-place} extends the case study to a 4-layer LNN with 2 residual connections, and compares it with the optimal choice of a single connection 0:3. Notably, the model does not necessarily benefit from the additional residual connections, unless they are appropriately located. 
\end{example}

Case Study~\ref{ex:LNN} suggests that the configuration of residual connections can induce exponential discrepancies in convergence behaviors, and thus must be chosen cautiously. Nevertheless, there are currently neither rigorous theories supporting these observations, nor a principled approach for selecting an optimal topology. These two challenges call for an analytical investigation and numerical evaluations, which will be provided in the ensuing sections. 

\subsection{Convergence analysis of exponential discrepancies}
This subsection establishes theories characterizing the exponentially different convergence behaviors induced by residual connections. To simplify the setup, our analysis considers the 3-layer LNN with a single residual connection (0:1 or 0:2), as illustrated in Figure~\ref{subfig:3layer-vary-place}, while extension to nonlinear neural networks is left for future work. 
To eliminate the influence of step size, the analysis is performed under gradient flow (GF), the continuous-time counterpart of GD.
Specifically, GF indexes the optimization time with a continuous variable $t \ge 0$, and updates the weight matrix through ordinary differential equation
\begin{equation}
\label{eq:GF}
	\frac{\dd \mathbf{W}_k(t)}{\dd t} = - \nabla_{\bfW_k} \loss (t), ~k = 1,\ldots,K
\end{equation}
where the dependency of $\loss$ on $\{ \bfW_k(t) \}_k$ is omitted for compactness. Clearly, GF can be viewed as an approximation of GD in the limit of infinitely small step sizes. 
Moreover, we assume the input data matrix $\bfX$ is whitened, so that it's orthogonal; i.e., $\bfX \bfX^\top = \bfI_d$. 
This is a standard modeling choice that avoids unnecessarily over-complicating the analysis while still capturing the essential dynamics \cite{converge-analysis-LNN,wu2019global,kai2025}.
When $\bfX$ is a random matrix, this assumption corresponds to the data vectors being normalized and uncorrelated, which can be readily satisfied in practice without loss of generality. 

Under these simplifications, the following two theorems establish respectively a lower bound for the LNN with the 0:1 shortcut, and an upper bound for the 0:2 one. 
\begin{theorem}[Existence lower bound]
\label{thm:LB}
Consider a 3-layer LNN with shortcut 0:1, and $\ell_2$ loss
\begin{equation}
\label{eq:obj-LB}
	\loss_1(t) := \frac{1}{2} \Bigg\| \bfW_3(t) \bfW_2(t) \big( \bfW_1(t) + \bfI_d \big) \bfX - \bfY \bigg\|_\fro^2.
\end{equation}
There \underline{exists} sufficiently small initialization that, GF~\eqref{eq:GF} cannot converge faster than a \underline{sublinear} rate
\begin{equation}
\label{eq:lower-bound}
	\loss_1 (t) \ge \Omega (1/t^2).
\end{equation}
\end{theorem}

\begin{theorem}[Universal upper bound]
\label{thm:UB}
Consider a 3-layer LNN with shortcut 0:2, and $\ell_2$ loss
\begin{equation}
\label{eq:obj-UB}
	\loss_2 (t) := \frac{1}{2} \Bigg\| \bfW_3(t) \big( \bfW_2(t) \bfW_1(t) + \bfI_d \big) \bfX - \bfY \bigg\|_\fro^2.
\end{equation}
Under \underline{any} sufficiently small initialization, GF~\eqref{eq:GF} ensures \underline{linear} convergence
\begin{equation}
\label{eq:upper-bound}
	\loss_2 (t) \le \loss_2(0) e^{-2(1-\lambda)^2 t}
\end{equation}
where $\lambda \in (0,1)$ is a constant associated with the initialization $\{ \bfW_k (0) \}_{k=1}^3$. 
\end{theorem}
The formal statements of Theorems~\ref{thm:LB} and~\ref{thm:UB} involving all details can be found in Appendix~\ref{apdx:proof}. We remark that small random initialization is standard in deep networks for numerical stability. While the 0:1 residual connection cannot achieve the universal upper bound faster than a sublinear convergence rate, the 0:2 shortcut yields an exponential improvement. This corroborates our observation in Figure~\ref{subfig:3layer-vary-place}, and validates that proper placements of residual connections can \emph{exponentially} accelerate the training convergence. 

It is worth stressing that these results can be readily extended to $K > 3$; see Appendix~\ref{apdx:extension-morelayers}. To be specific, the 0:1 shortcut always corresponds to the slow case of sublinear convergence rate of $\Omega(1/t^\frac{K-1}{K-2}) \ge \Omega (1/t^2)$, whereas the 0:$K\!-\!1$ one guarantees fast linear convergence. This suggests: \textit{increasing depth can impede optimization when residual connections are not appropriately placed}. For the simple LNNs, the lower bound approaches $\Omega(1/t)$ as $K\to\infty$. An example of the fast case can be found in the 0:3 (violet) curve of Figure~\ref{subfig:4layer-vary-place}. 

Our case study and convergence analysis highlight \emph{where} residual connections are placed could affect the convergence rate exponentially. 
This pronounced gap suggests that residual topology should be treated as a design lever for faster optimization. In the next section, we turn to the practical question of \textit{how} to find a performant topology.

\begin{figure}[t]
   \begin{center}
  \begin{subfigure}[t]{0.23\textwidth}
  	\centerline{\includegraphics[width=\textwidth]{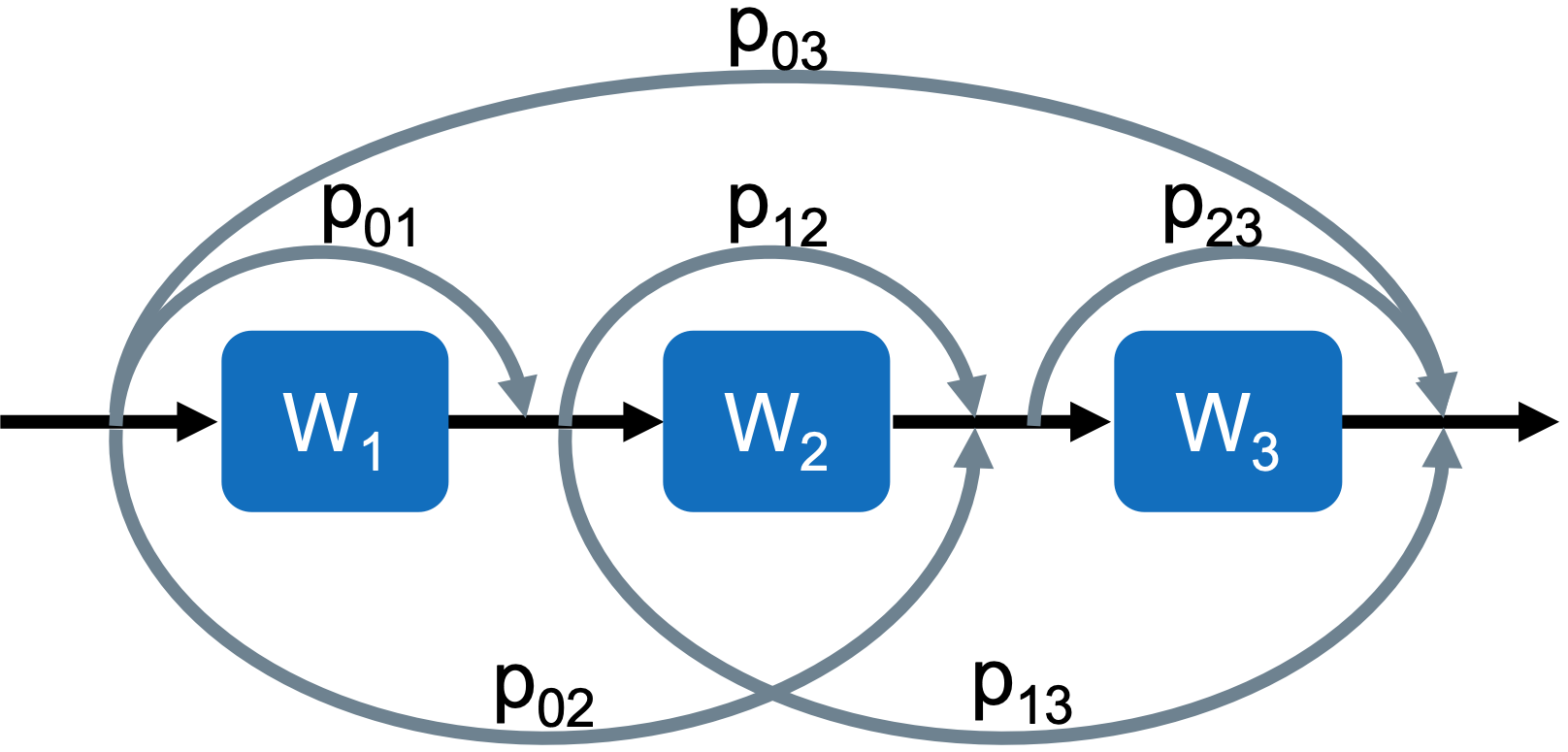}}
  	\caption{ANCRe}
  	\label{subfig:ANCRe-connect}
  \end{subfigure}
  \begin{subfigure}[t]{0.26\textwidth}
  	\centerline{\includegraphics[width=\textwidth]{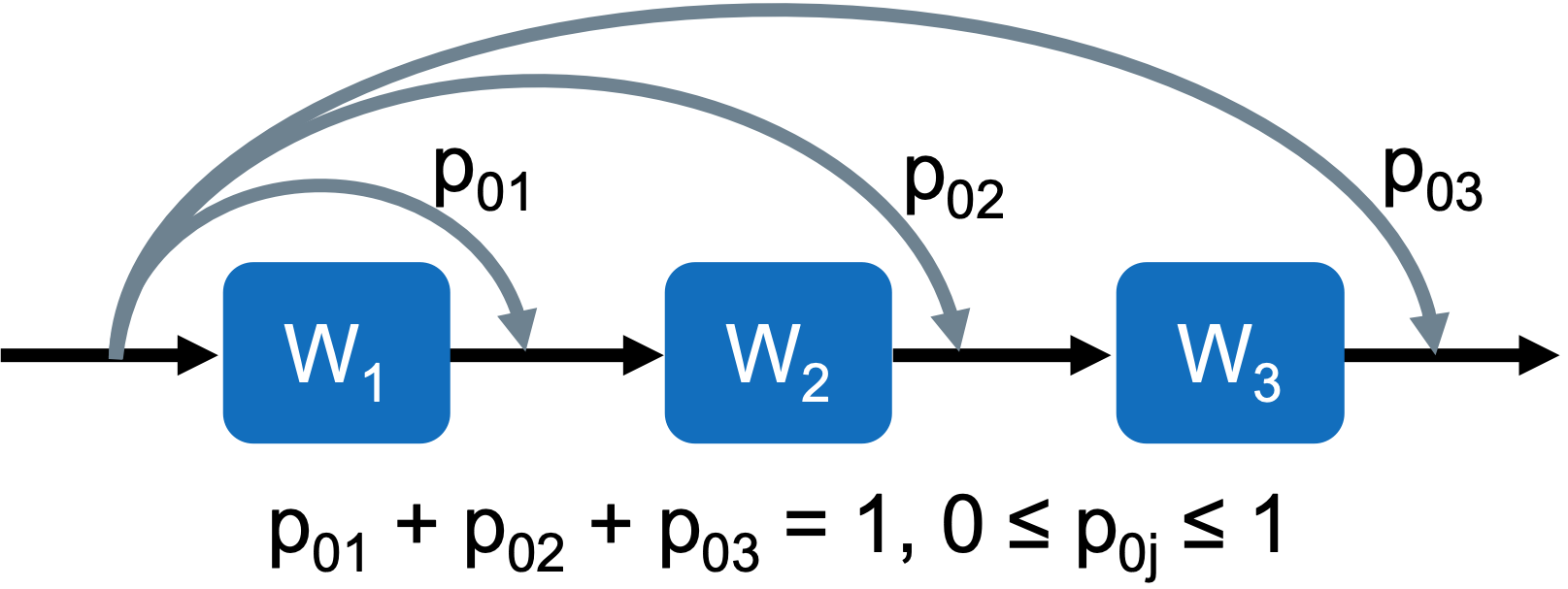}}
  	\caption{Outgoing normalization}
  	\label{subfig:out-normalize}
  \end{subfigure}
  \begin{subfigure}[t]{0.26\textwidth}
  	\centerline{\includegraphics[width=\textwidth]{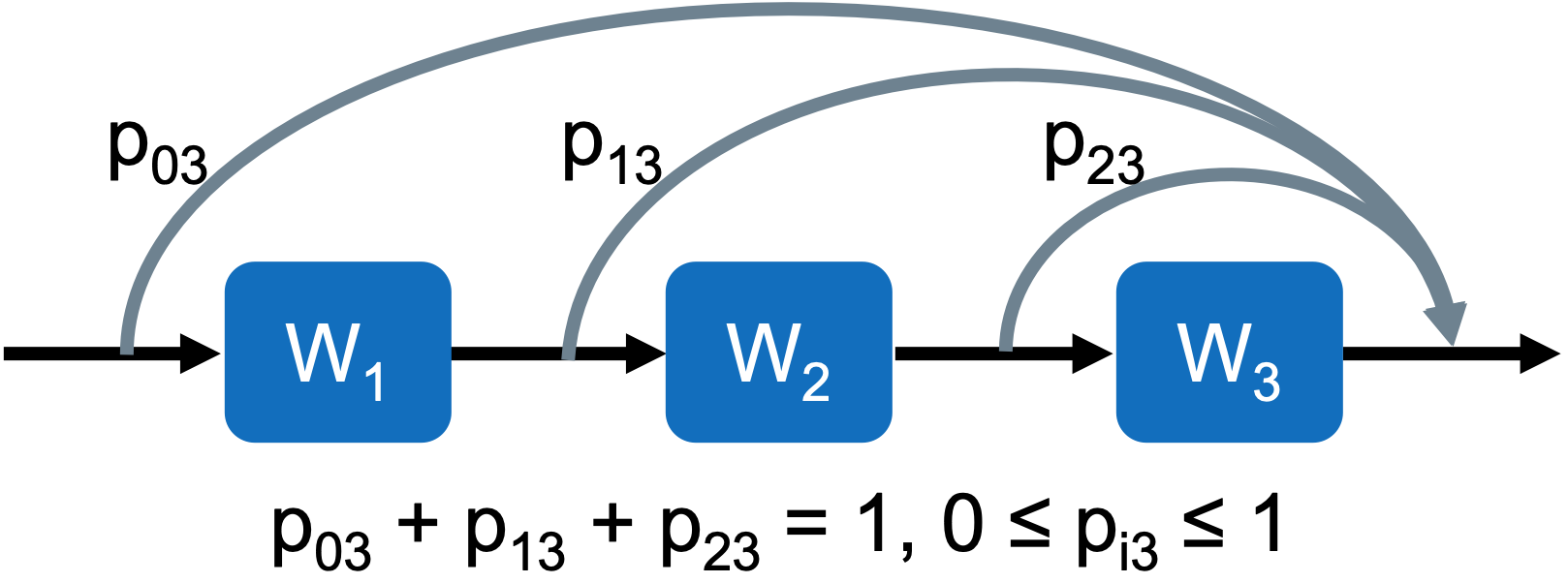}}
  	\caption{Ingoing normalization}
  	\label{subfig:in-normalize}
  \end{subfigure}
    \begin{subfigure}[t]{0.23\textwidth}
  	\centerline{\includegraphics[width=\textwidth]{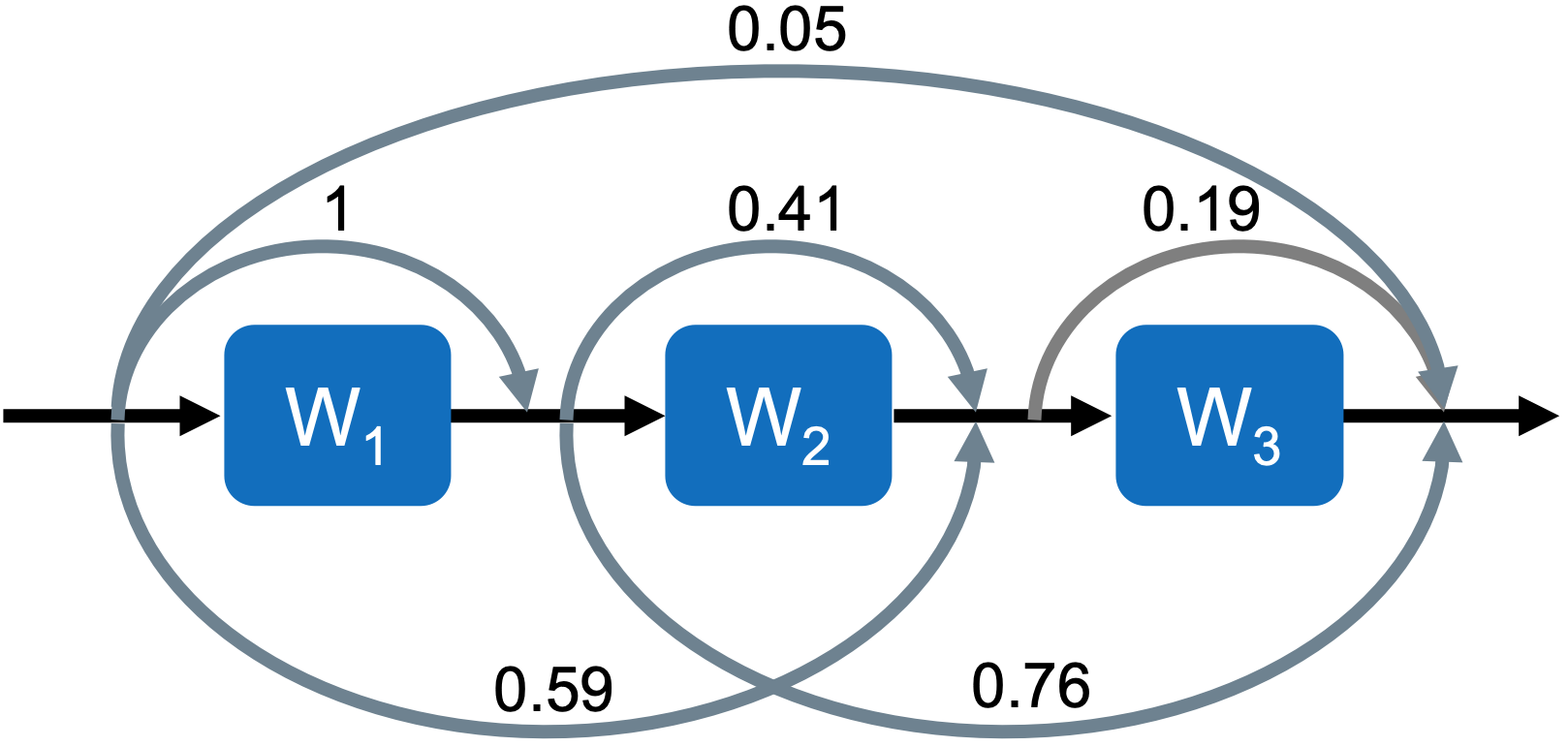}}
  	\caption{Learned connectivity}
  	\label{subfig:ANCRe-coeff}
  \end{subfigure}
  \vspace{-.1cm}
  \caption{Visualization of ANCRe and two normalization schemes on a 3-layer LNN.}
  \label{fig:ANCRe-n-normalization}
  \end{center}
  \vspace{-.1cm}
\end{figure}

\section{Learning residual connections from data}
As observed in the previous section, even with the same network architecture, the residual topology pattern for fast convergence can change with depth (e.g., 0:2 for 3-layer and 0:2+2:3 for 4-layer LNNs in Figures~\ref{subfig:3layer-vary-place} and~\ref{subfig:4layer-vary-place}). More broadly, an optimal layout can depend on other architectural choices. It is thus difficult to prescribe a ``panacea'' of residual topology, especially for complex foundation models. This motivates learning the desirable topology on the fly.

\subsection{ANCRe: Adaptive neural connection reassignment}
\label{subsec:ANCRe}
Instead of exhaustively searching over all the $2^{(K+1)K/2}$ possible topologies, our key idea is to parameterize the layout, and adaptively reassign the connectivity based on the network architecture and data distribution. Thus, we term our approach adaptive neural connection reassignment (ANCRe). In particular, ANCRe considers all possible shortcut connections $i$:$j$ ($i < j$), each associated with a coefficient $p_{ij}$ that is optimized alongside the model parameters; cf. Figure~\ref{subfig:ANCRe-connect}. We further probe two schemes for normalizing these weights, which we refer to as outgoing and ingoing normalizations, as exemplified in Figures~\ref{subfig:out-normalize} and~\ref{subfig:in-normalize}. 
This normalization ensures that the coefficients over candidate positions sum to 1. In particular, if some $p_{ij}=1$, the construction reduces to adding a standard residual connection between layers $i$ and $j$. As shown later in Section~\ref{sec:ablation}, this normalization is essential for stabilizing  training at scale.

Outgoing normalization restricts the coefficients originating from layer $i$ via a convex combination
\begin{equation}
\label{eq:out-normalize}
	\sum_{j=i+1}^K p_{ij} = 1, ~ 0 \le p_{ij} \le 1, ~\forall 0 \le i < j \le K.
\end{equation}
That says, the feature embedding from layer $i$ is routed to downstream layers through weighted shortcuts $p_{ij} \bfx_i$. This ensures that the outgoing information per layer is bounded. 
Conversely, ingoing normalization imposes the following constraint
\begin{equation}
\label{eq:in-normalize}
	\sum_{i=0}^{j-1} p_{ij} = 1, ~ 0 \le p_{ij} \le 1, ~\forall 0 \le i < j \le K
\end{equation}
which preserves the input signal magnitude to layer $j$ and enhances numerical stability. An alternative interpretation is that it forms an \emph{ensemble over network depths}. Letting $f_0(\bfx) := \bfx$ denote the initial input and $f_i(\bfx)$ the output of layer $i$, the aggregated shortcuts to layer $j$ incur $\sum_{i=0}^{j-1} p_{ij} f_i(\bfx)$, which ensembles subnetworks of depths $0,\ldots,j-1$, with $p_{ij}$ serving as the routing coefficient. 

In practice, both normalization constraints can be conveniently enforced via a softmax reparametrization. For example, the ingoing normalization can be formulated as
\begin{equation}
\label{eq:softmax}
	p_{ij} = \frac{\exp(c_{ij}/\tau)}{\sum_{k=0}^{j-1} \exp(c_{kj}/\tau)},~\forall 0 \le i < j \le K
\end{equation}
where $\tau > 0$ is a temperature hyperparameter, with a default value of $0.01$ in practice. This softmax allows the model to emphasize beneficial shortcuts by driving $c_{ij} \to +\infty$, while suppressing detrimental ones with $c_{ij} \to -\infty$. 
Furthermore, if an effective configuration is known a priori, it can be readily incorporated through a regularizer. For instance, adding Tsallis entropy regularization~\cite{tsallis-entropy} to the loss encourages sparse connectivity for softmax reparameterization. It is worth stressing that ANCRe introduces merely $K (K + 1) / 2$ additional parameters $c_{ij}$, which is often fewer than a single feature dimension, and is negligible in modern deep networks. As a result, ANCRe preserves the highly desired computational efficiency; cf. Table~\ref{tab:time}. 

While both normalizations work well in practice, we often find that ingoing normalization slightly outperforms the other at larger scale; see the ablation study in Section~\ref{sec:ablation}. We conjecture that this is due to the more stable input magnitude to each layer. For this reason, we stick to \textit{ingoing normalization}~\eqref{eq:in-normalize} throughout the paper.

Regarding Case Study~\ref{ex:LNN}, ANCRe learns a data-driven connectivity, and hence achieves a linear convergence rate on par with the optimal case; cf. Figure~\ref{subfig:3layer-vary-place} and~\ref{subfig:4layer-vary-place}. The learned coefficients for the 3-layer LNN under ingoing normalization are visualized in Figure~\ref{subfig:ANCRe-coeff}, where certain shortcuts such as 0:3 and 2:3 are suppressed by the ingoing normalization. Aside from LNNs, the next subsection investigates how to effectively incorporate ANCRe into Transformers.

\begin{figure}[t]
  \centering
  \begin{minipage}[t]{0.5\linewidth}
  \vspace{-.1cm}
      \begin{center}
      	\centerline{\includegraphics[width=.8\columnwidth]{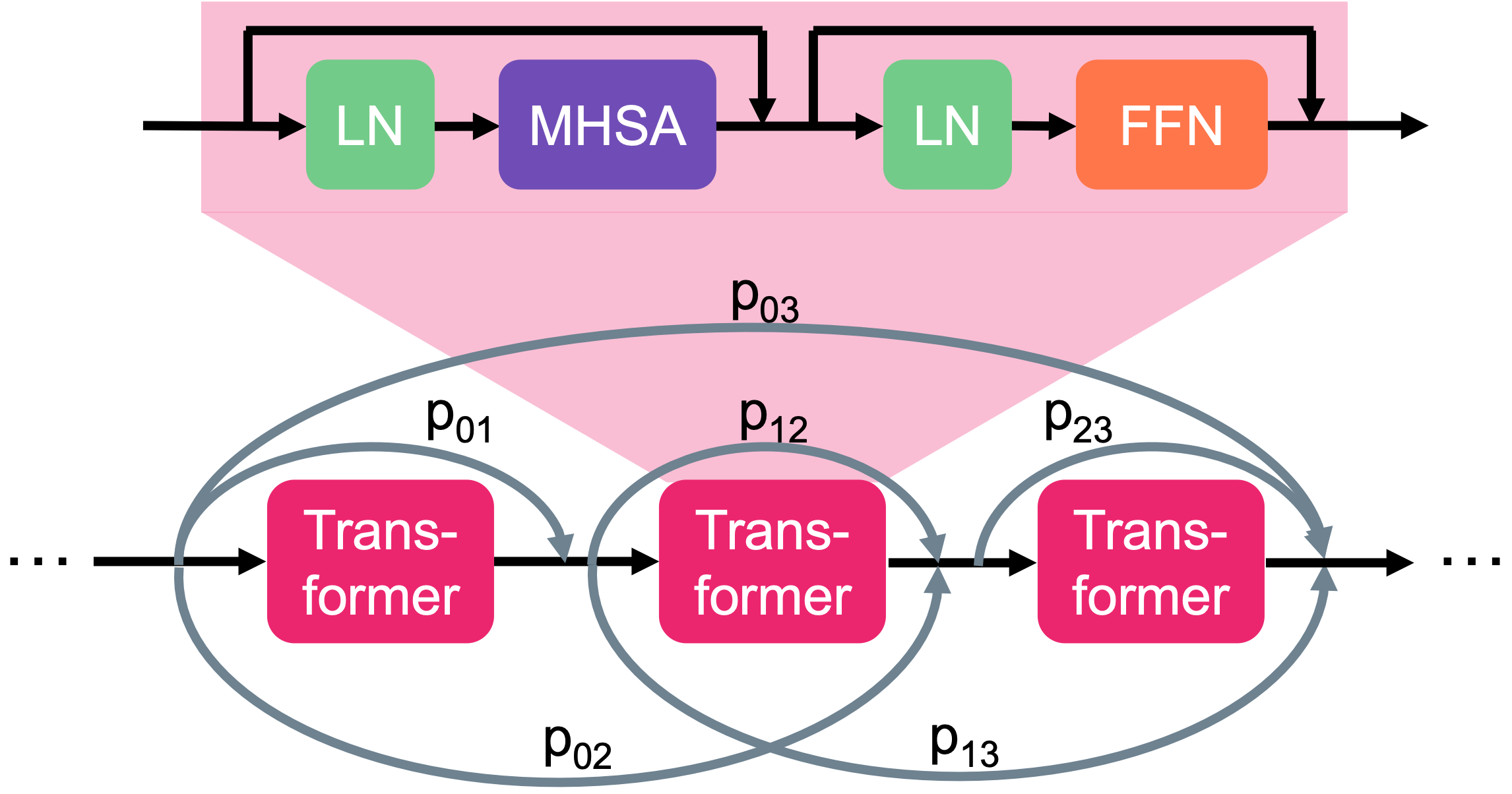}}
        \vspace{-.1cm}
      \caption{ANCRe applied to standard Transformer comprising LN, MHSA, and FFN modules.}
      \label{fig:Transformers}
      \end{center}
  \end{minipage}
  \hfill
  \begin{minipage}[t]{0.45\linewidth}
   \captionof{table}{Validation perplexity ($\downarrow$) by adding ANCRe to inputs of different modules.}
   \vspace{-.1cm}
   \label{tab:granularity}
   \begin{center}
     \begin{tabular}{cccc}
       \toprule
       MHSA  & FFN  &  60M  &  130M \\
       \midrule
       $\times$ & $\times$ & 30.39 & 25.07 \\
       $\surd$  & $\times$ & \textbf{29.62} & 24.48  \\
       $\times$ & $\surd$ & 30.43 & 24.98 \\
       $\surd$  &  $\surd$ & 29.73 & \textbf{24.41} \\
       \bottomrule
     \end{tabular}
   \end{center}
  \end{minipage}
  \vspace{-.2cm}
\end{figure}

\subsection{Applying ANCRe to Transformers}
In modern Transformers, residual connections are applied separately to the multi-head self-attention (MHSA) and feedforward network (FFN) modules in a cascaded way, as illustrated in the upper part of Figure~\ref{fig:Transformers}. We consider two granularities for placing ANCRe: module-level and block-level. 
The former introduces ANCRe connections from the input of the $i$-th FFN (resp. MHSA) to the input of the $j$-th FFN (resp. MHSA), while the latter bridges the inputs of entire Transformer blocks. Note that the block-level granularity coincides with the MHSA module-level one.


To determine the optimal granularity, a lightweight ablation study is performed by pre-training LLaMA-60M and LLaMA-130M~\cite{llama} on the C4 dataset~\cite{C4}. As summarized in Table~\ref{tab:granularity}, establishing connectivity between MHSA inputs improves the validation perplexity (lower the better), whereas shortcuts originating from FFN inputs are less informative. 
Moreover, combining both connections yields performance comparable to using MHSA-only connections. 
For this reason, we adopt the MHSA-only (i.e. block-level) ANCRe for all subsequent tests involving Transformers; cf. Figure~\ref{fig:Transformers}. 

Aside from granularity, another critical consideration for LLM integration is the compatibility with gradient checkpointing and pipeline parallelism, where long-range shortcuts can markedly increase the recomputation and communication costs. To address this challenge, Appendix~\ref{apdx:gradckpt-pipeline} presents a practical variant that substantially lowers the associated overhead through sparse residual connectivity. 

Next, tests are conducted on foundation models across domains to showcase ANCRe's effectiveness. 

\section{Numerical experiments}
\label{sec:experiments}
This section implements numerical tests to assess the performance of ANCRe, and compare it with vanilla cascaded residual connections (i.e., the standard architecture). All tests are conducted on servers with NVIDIA A100 and H100 GPUs. Additional setups including models, hyperparameters, and datasets are detailed in Appendix~\ref{apdx:experiment-setups}. The code is available at~\url{https://github.com/zhangyilang/ANCRe}.

\begin{figure}[t]
   \begin{center}
   \centerline{\includegraphics[width=.5\textwidth]{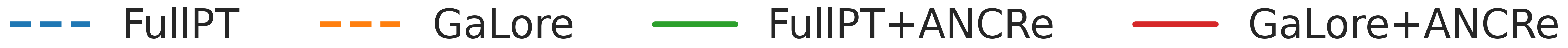}}
  \begin{subfigure}[t]{0.245\textwidth}
  	\centerline{\includegraphics[width=\textwidth]{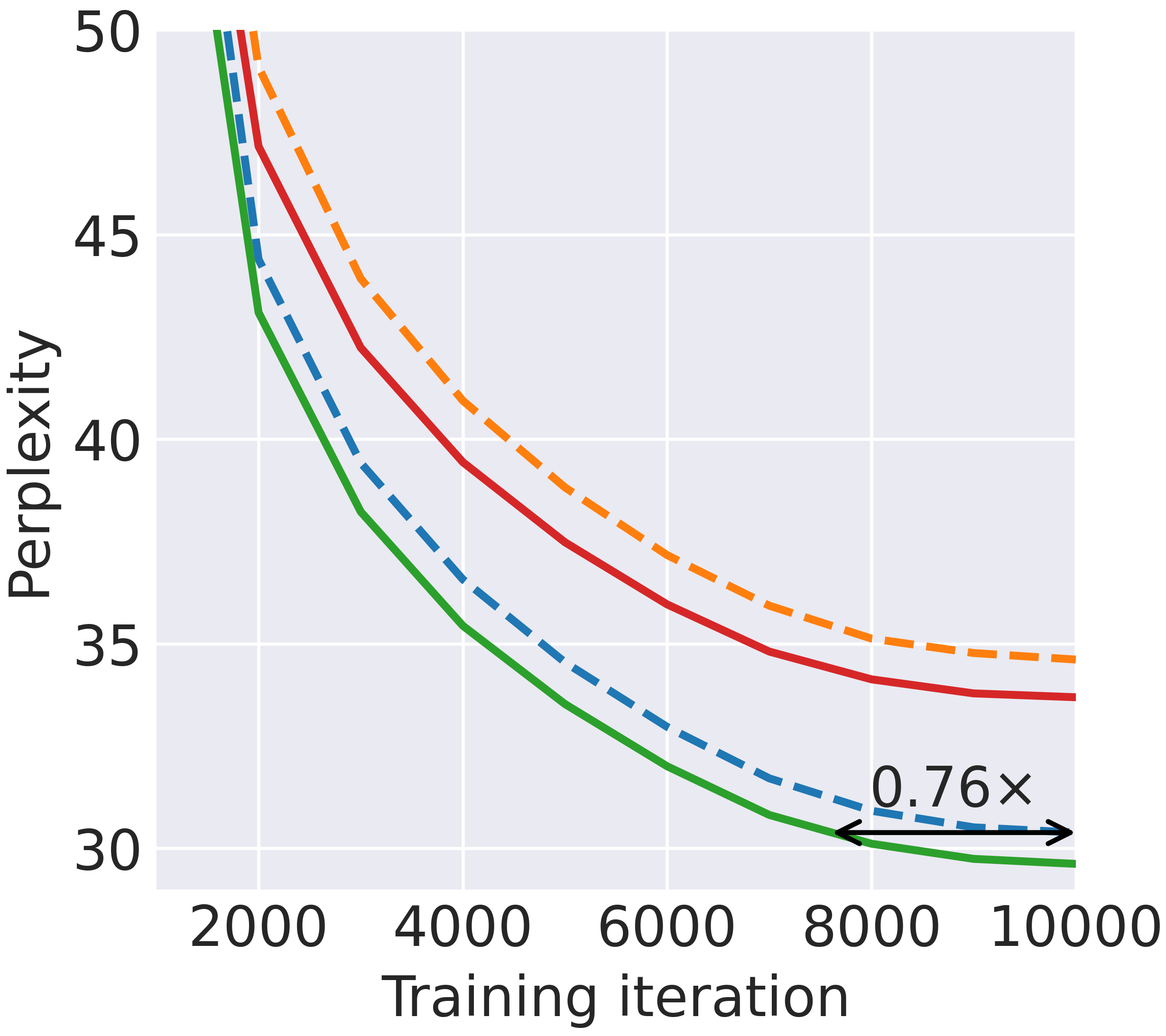}}
  	\caption{60M}
  	\label{subfig:PPL-60m}
  \end{subfigure}
  \begin{subfigure}[t]{0.245\textwidth}
  	\centerline{\includegraphics[width=\textwidth]{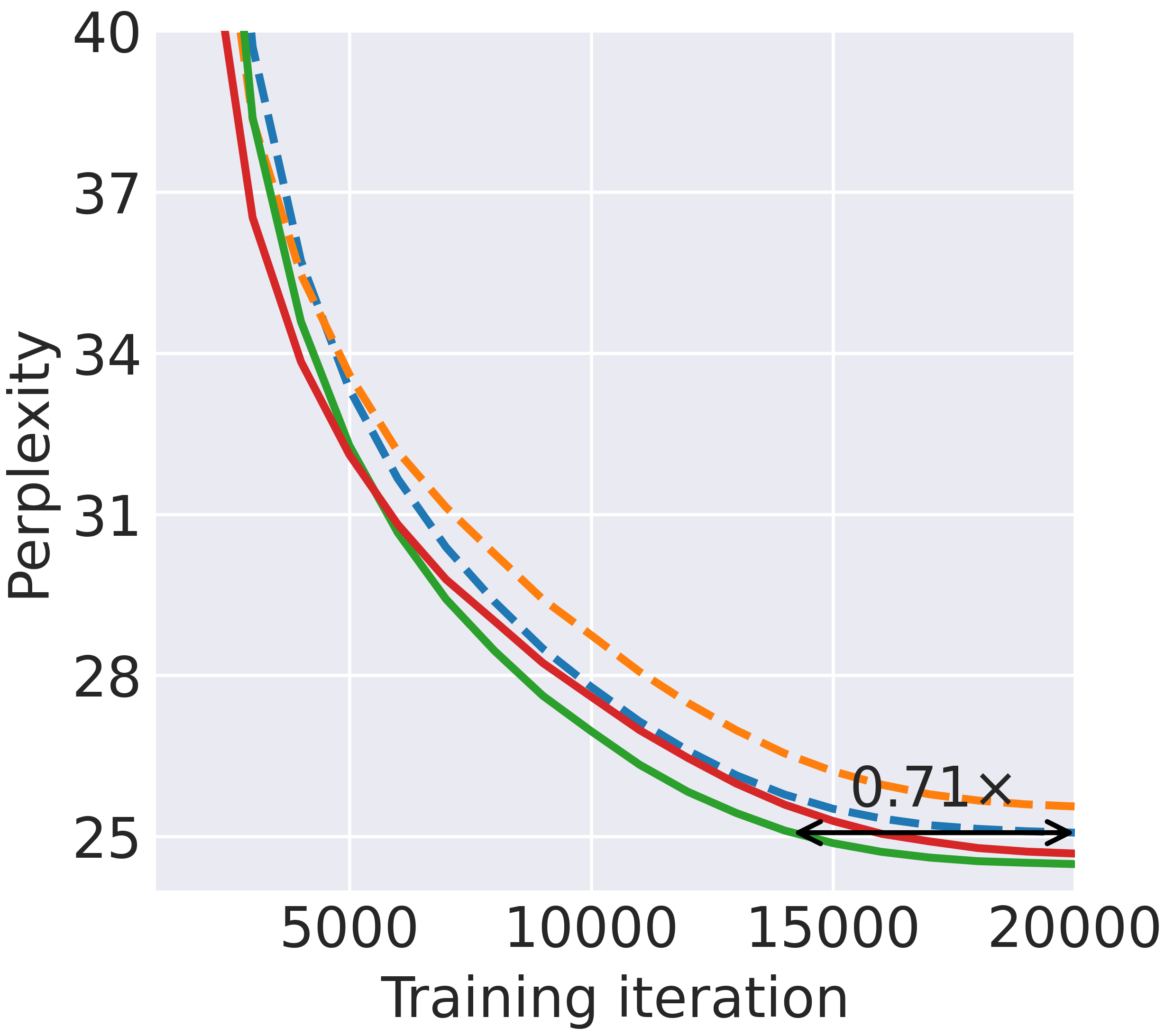}}
  	\caption{130M}
  	\label{subfig:PPL-130m}
  \end{subfigure}
  \begin{subfigure}[t]{0.245\textwidth}
  	\centerline{\includegraphics[width=\textwidth]{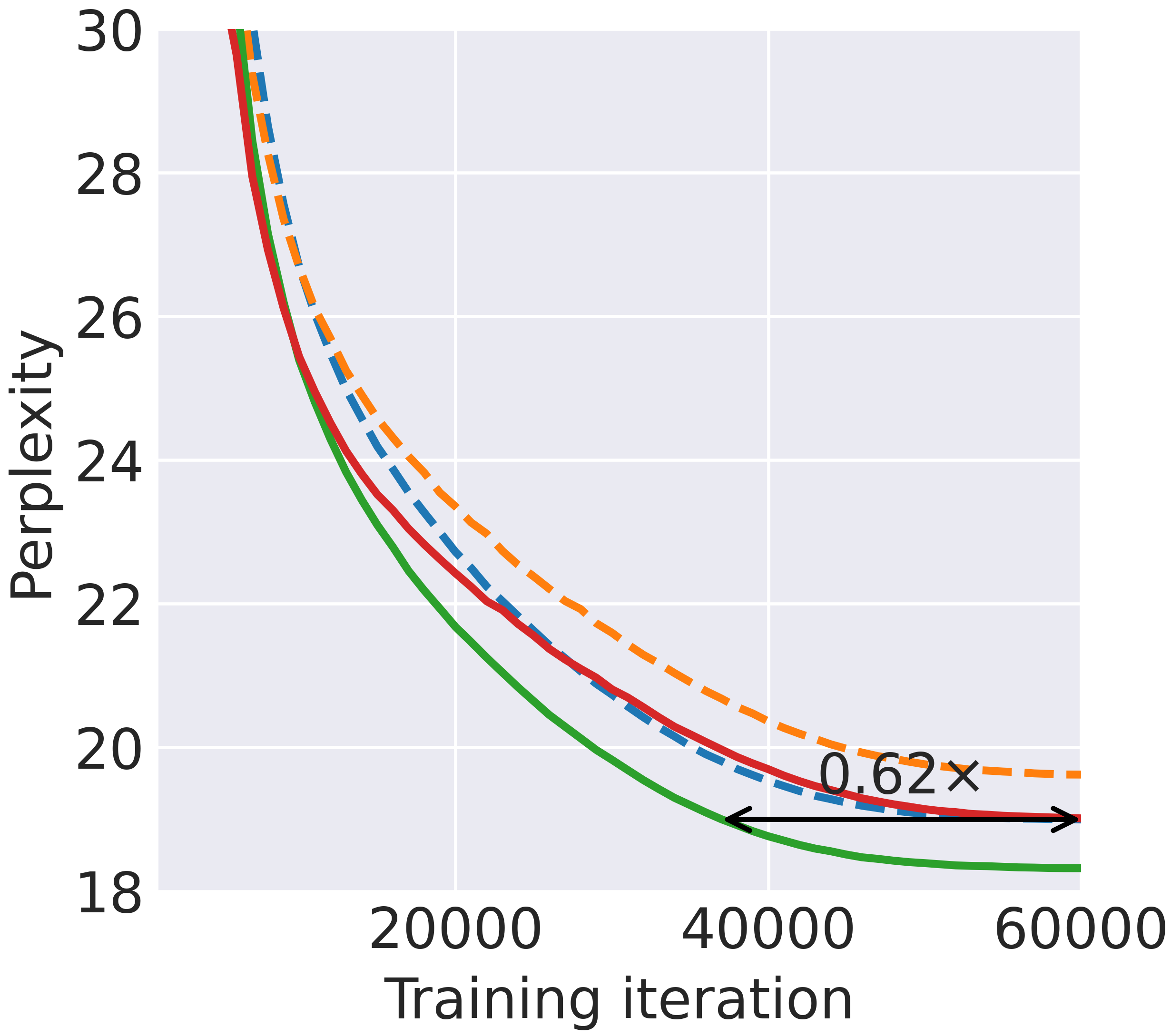}}
  	\caption{350M}
  	\label{subfig:PPL-350m}
  \end{subfigure}
    \begin{subfigure}[t]{0.245\textwidth}
  	\centerline{\includegraphics[width=\textwidth]{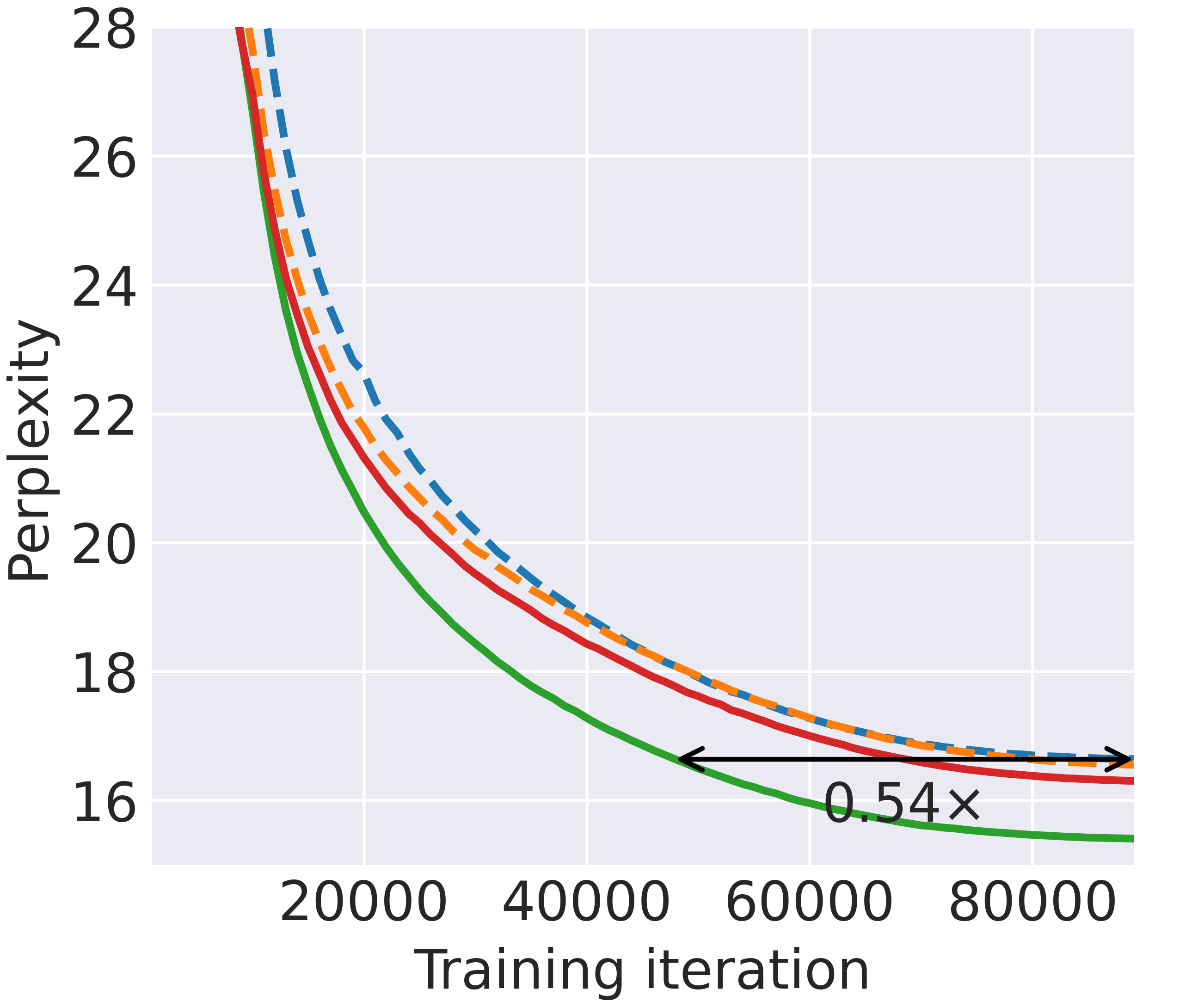}}
  	\caption{1B}
  	\label{subfig:PPL-1b}
  \end{subfigure}
 \vspace{-.1cm}
  \caption{Validation perplexity evolution during pre-training of LLaMA models of different sizes.}
  \label{fig:LLM-PPL}
  \end{center}
  \vspace{-.1cm}
\end{figure}

\begin{figure}[t]
   \begin{center}
  \begin{subfigure}[t]{0.245\textwidth}
  	\centerline{\includegraphics[width=\textwidth]{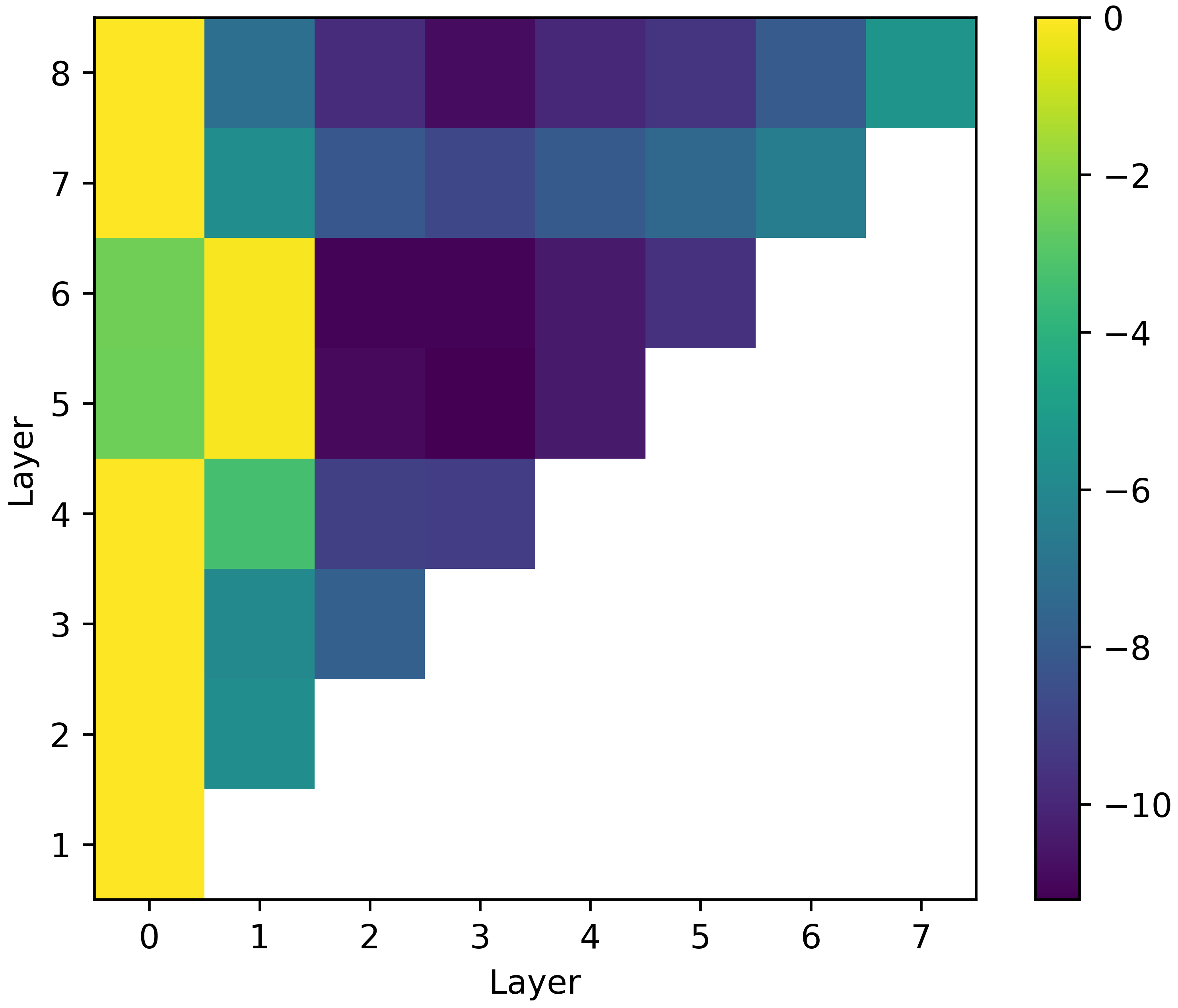}}
  	\caption{60M}
  	\label{subfig:heatmap-60m}
  \end{subfigure}
  \begin{subfigure}[t]{0.245\textwidth}
  	\centerline{\includegraphics[width=\textwidth]{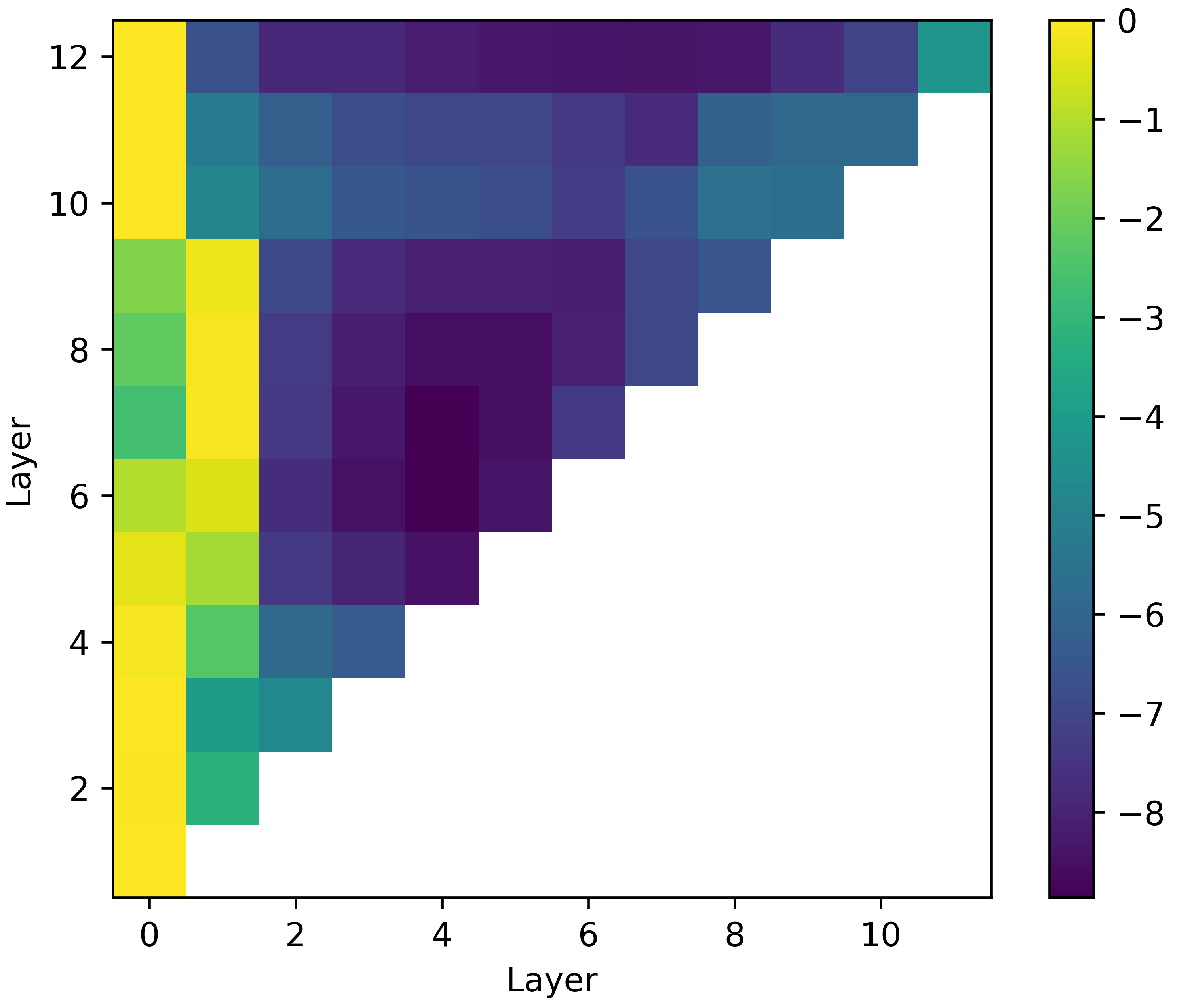}}
  	\caption{130M}
  	\label{subfig:heatmap-130m}
  \end{subfigure}
  \begin{subfigure}[t]{0.245\textwidth}
  	\centerline{\includegraphics[width=\textwidth]{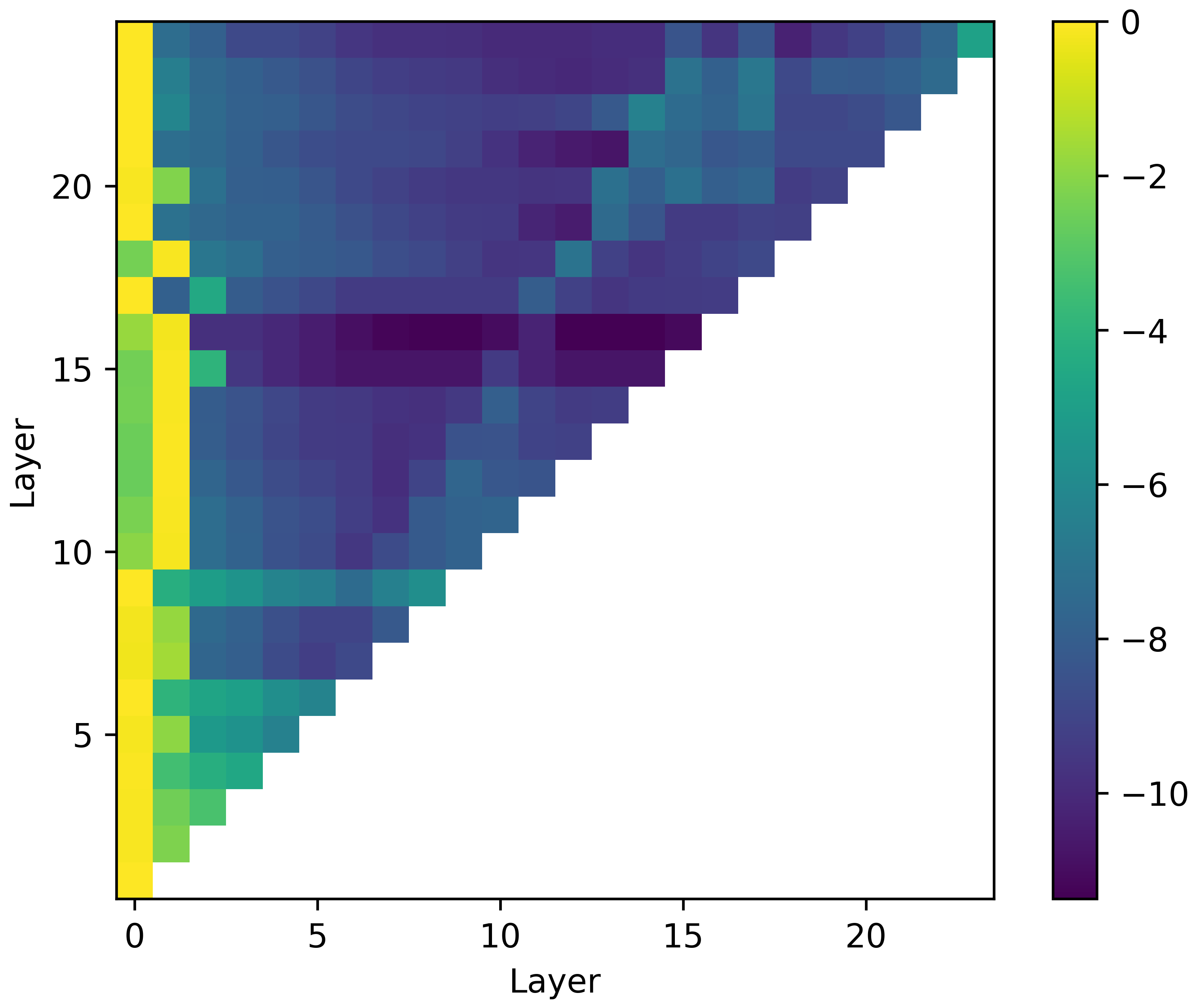}}
  	\caption{350M}
  	\label{subfig:heatmap-350m}
  \end{subfigure}
    \begin{subfigure}[t]{0.245\textwidth}
  	\centerline{\includegraphics[width=\textwidth]{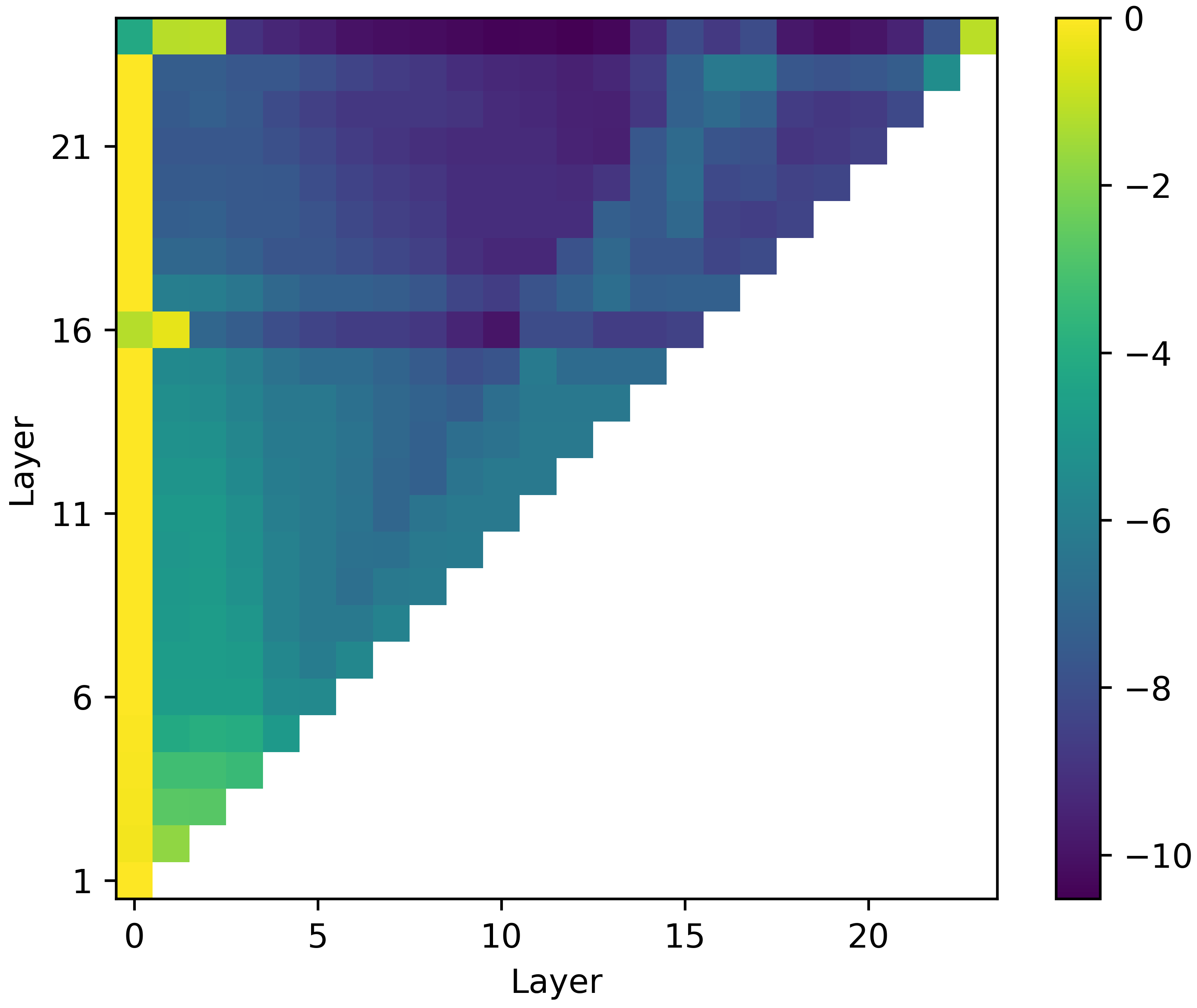}}
  	\caption{1B}
  	\label{subfig:heatmap-1b}
  \end{subfigure}
 \vspace{-.1cm}
  \caption{Visualization of residual coefficients $\log (p_{ij})$ learned by ANCRe, where horizontal and vertical axes stand for $i$ and $j$.}
  \label{fig:LLM-skip-heatmap}
  \end{center}
  \vspace{-.1cm}
\end{figure}

\subsection{Pre-training of LLMs}
\label{subsec:exp-LLM}
The evaluation of ANCRe begins with pre-training LLMs. Following the setups in~\cite{ReLoRA,GaLore}, a series of LLaMA models~\cite{llama} are trained on the widely adopted C4 dataset~\cite{C4}, which consists of cleaned English text collected across the Internet. 
The models span parameter counts of \{60M, 130M, 350M, 1B\}, with corresponding depths of \{8, 12, 24, 24\}. All experiments are run in BF16 precision. Two optimization schemes are considered: standard full pre-training (FullPT) and GaLore~\cite{GaLore}, which applies low-rank gradient projection for memory efficiency. For both schemes, learning rates are selected from $\{5\times10^{-4},10^{-3}, 5\times10^{-3}, 10^{-2}, 5\times10^{-2}\}$. ANCRe directly adopts the learning rate optimized for cascaded residual connections without additional tuning. HC~\cite{HC} and mHC~\cite{mHC} with 4 parallel residual streams are included as strong residual-learning baselines. 

Figure~\ref{fig:LLM-PPL} plots the validation perplexity over pre-training iterations, while Table~\ref{tab:LLM-PPL} reports quantitative results evaluated with the pre-trained models. Across all eight combinations of model sizes and optimization schemes, ANCRe consistently converges faster, and reduces perplexities by an average ($\pm$ standard deviation) of $0.72 \pm 0.33$ over conventional residual topology. Notably, ANCRe matches the perplexity of cascaded residual connections using on average $34.3\%$ fewer training iterations. The advantage grows from $24\%$ to $46\%$ as the network scales deeper. This underscores not only the significance of residual connection topology, but also the effectiveness of ANCRe in leveraging the increased depth. Additional runs with multiple random seeds are reported in Appendix~\ref{apdx:seeds}. 

In addition, Figure~\ref{fig:LLM-skip-heatmap} sketches the residual coefficients learned by ANCRe. For clearer visualization, the heatmap is displayed on a logarithmic scale. Rather than relying solely on the shortcut from immediately preceding layer, deeper layers under ANCRe ensemble subnetworks of different depths. 
The learned connectivity patterns are also consistent across model sizes, with a predominant emphasis on shortcuts originating from the first two layers. Furthermore, larger models (350M and 1B) exhibit denser connectivity than the smaller ones, indicating more effective exploitation of network depth.

\begin{figure}[t]
  \centering
  \begin{minipage}[t]{0.52\linewidth}
  \vspace{0cm}
   \captionof{table}{Perplexity (PPL$\downarrow$) comparison by pre-training LLaMA models of varying sizes. The lowest PPLs are marked in solid lines.}
   \vspace{-.2cm}
   \label{tab:LLM-PPL}
   \begin{center}
        \begin{small}
            \begin{tabular}{lcccc}
           \toprule
           Method  & 60M         & 130M      & 350M	& 1B  \\
           \midrule
           FullPT    & 30.39 & 25.07 & 19.00 & 16.64 \\
           FullPT+HC & 29.86 & 24.51 & 18.58 & 16.38 \\
           FullPT+mHC & 30.21 & 24.62 & 18.93 & 16.14 \\
           FullPT+ANCRe & \textbf{29.62} & \textbf{24.48} & \textbf{18.32} & \textbf{15.41} \\
           \midrule
           GaLore    & 34.61 & 25.56 & 19.62 & 16.55 \\
           GaLore+ANCRe    & \textbf{33.69} & \textbf{24.68} & \textbf{19.01} & \textbf{16.45} \\
           \midrule
           Training tokens & 1.1B & 2.2B & 6.4B & 13.1B \\
           \bottomrule
         \end{tabular}
        \end{small}
   \end{center}
  \end{minipage}
  \hfill
  \begin{minipage}[t]{0.42\linewidth}
  \vspace{0cm}
   \begin{center}
  	\centerline{\includegraphics[width=0.99\columnwidth]{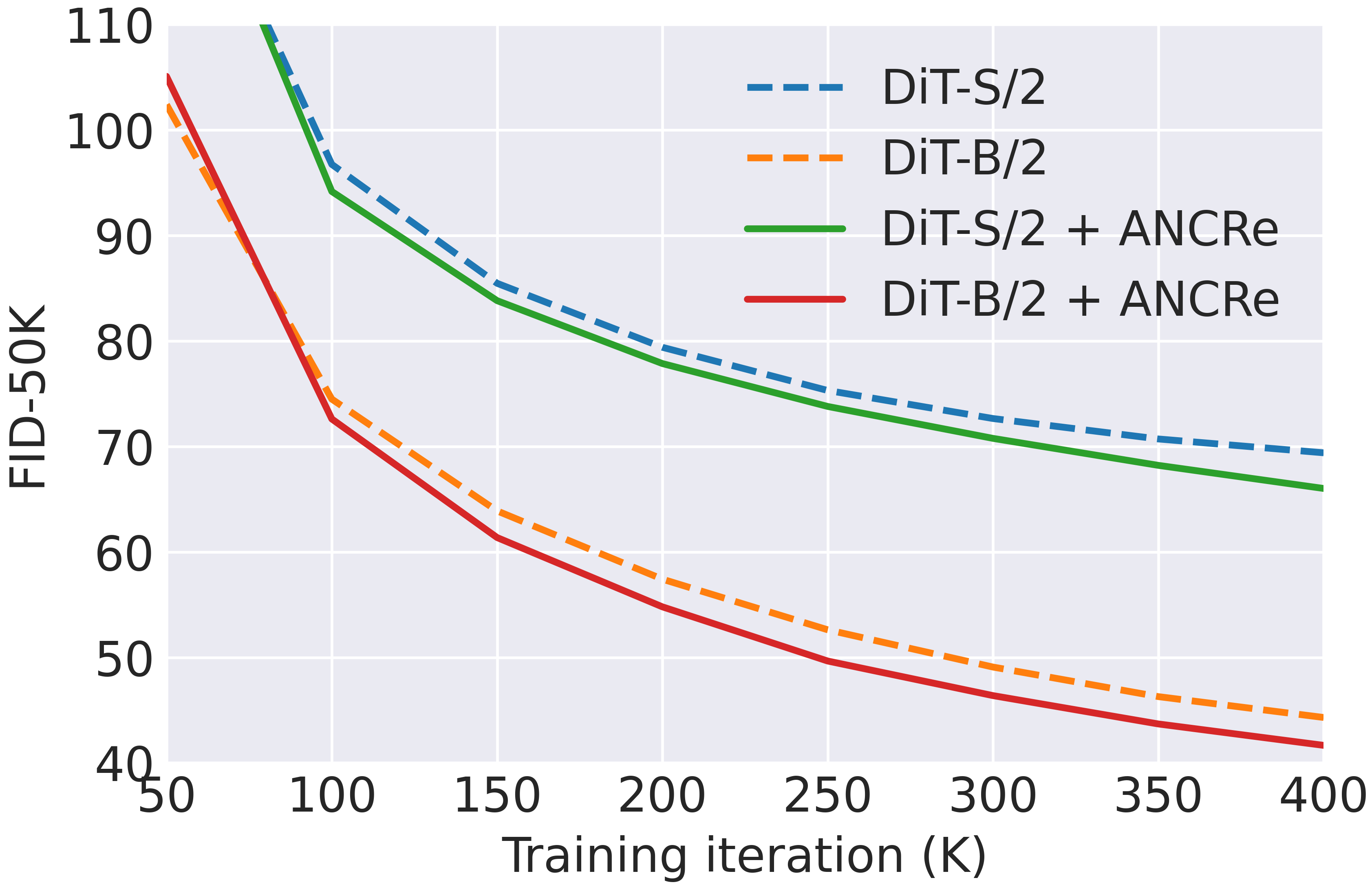}}
  \caption{FID-50K evolution when pre-training DiT models.}
  \label{fig:DiT-FID}
  \end{center}
  \end{minipage}
  \vspace{-.2cm}
\end{figure}
 
\begin{table*}[t]
   \caption{Comparison of ANCRe and cascaded residual connections by pre-training DiT models. The better of the two are marked in bold.}
   \vspace{-.2cm}
   \label{tab:DiT}
   \begin{center}
     \begin{tabular}{lccccc}
       \toprule
       Model  & FID$\downarrow$         & sFID$\downarrow$      & IS$\uparrow$ 	& Precision$\uparrow$	& Recall$\uparrow$  \\
       \midrule
       DiT-S/2    & 69.40 &	12.45 &	19.66 &	35.75\% & 55.82\% \\
       DiT-S/2+ANCRe & \textbf{66.01} &	\textbf{11.68} & \textbf{20.70} & \textbf{37.90\%} & \textbf{57.80\%} \\
       DiT-S/2 (cfg=1.5)    & 45.78 & 	9.08 &	33.48 &	46.08\% &	54.25\% \\
       DiT-S/2+ANCRe (cfg=1.5) & \textbf{42.99} & \textbf{8.61} &	\textbf{35.04} &	\textbf{48.53\%} &	\textbf{55.29\%} \\
       \midrule
       DiT-B/2    & 44.31 &	8.42 &	32.89 &	47.93\% &	61.55\% \\
       DiT-B/2+ANCRe    & \textbf{41.66} & \textbf{7.89} & \textbf{34.40} & \textbf{50.42\%} &	\textbf{64.20\%} \\
       DiT-B/2 (cfg=1.5)    & 22.41 &	6.35 &	65.27 &	60.75\%	& 52.41\% \\
       DiT-B/2+ANCRe (cfg=1.5)    & \textbf{20.53} & \textbf{5.81} & \textbf{70.45} & \textbf{65.91\%} & \textbf{56.51\%} \\
       \bottomrule
     \end{tabular}
   \end{center}
  \vspace{-.1cm}
 \end{table*}

\begin{figure}[t]
   \begin{center}
   \centerline{\includegraphics[width=.6\textwidth]{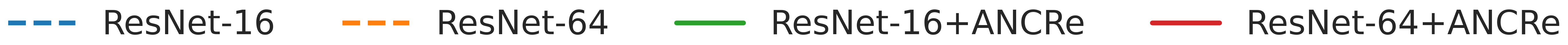}}
  \begin{subfigure}[t]{0.245\textwidth}
  	\centerline{\includegraphics[width=\textwidth]{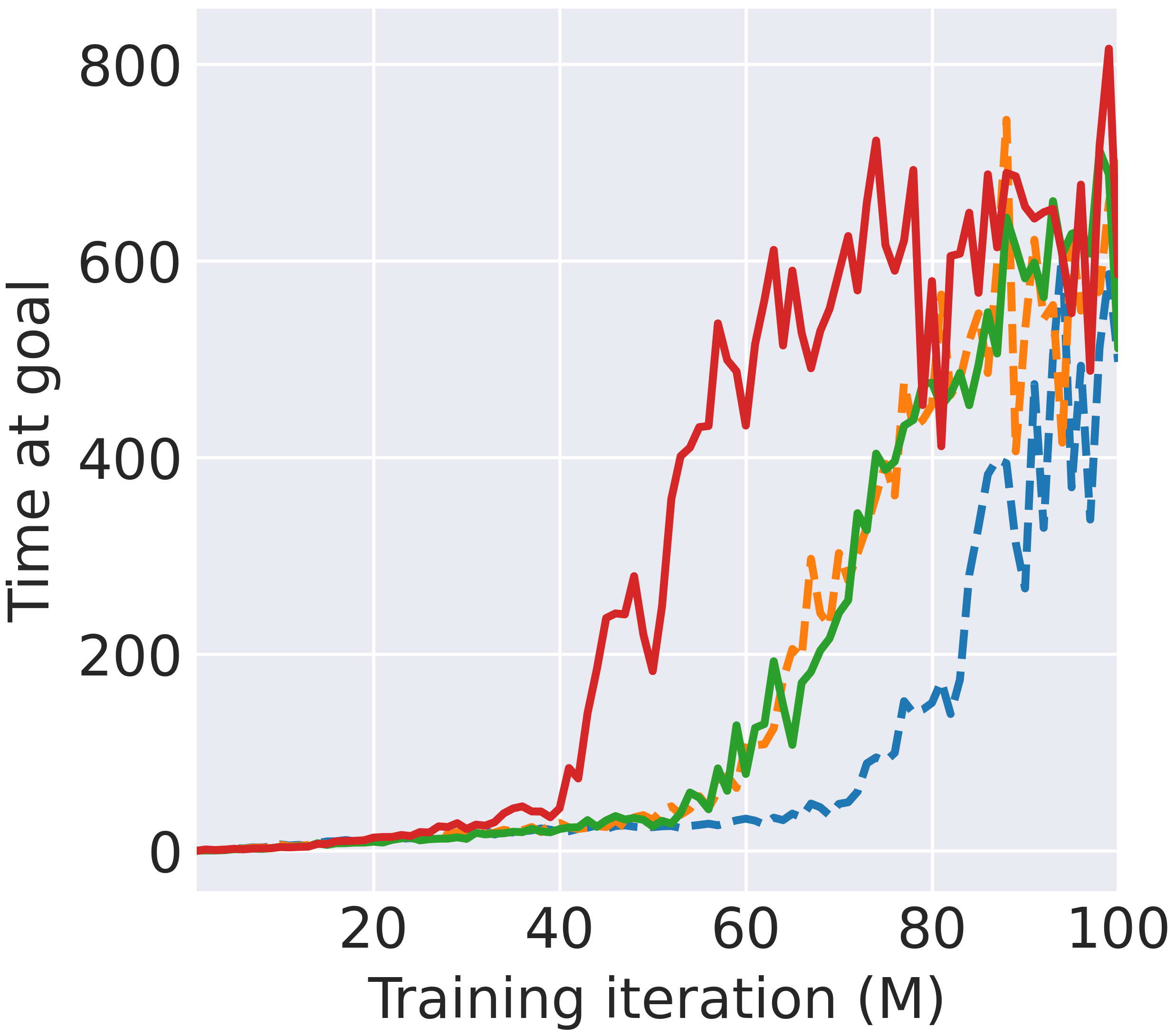}}
  	\caption{Humanoid}
  	\label{subfig:RL-humanoid}
  \end{subfigure}
  \begin{subfigure}[t]{0.245\textwidth}
  	\centerline{\includegraphics[width=\textwidth]{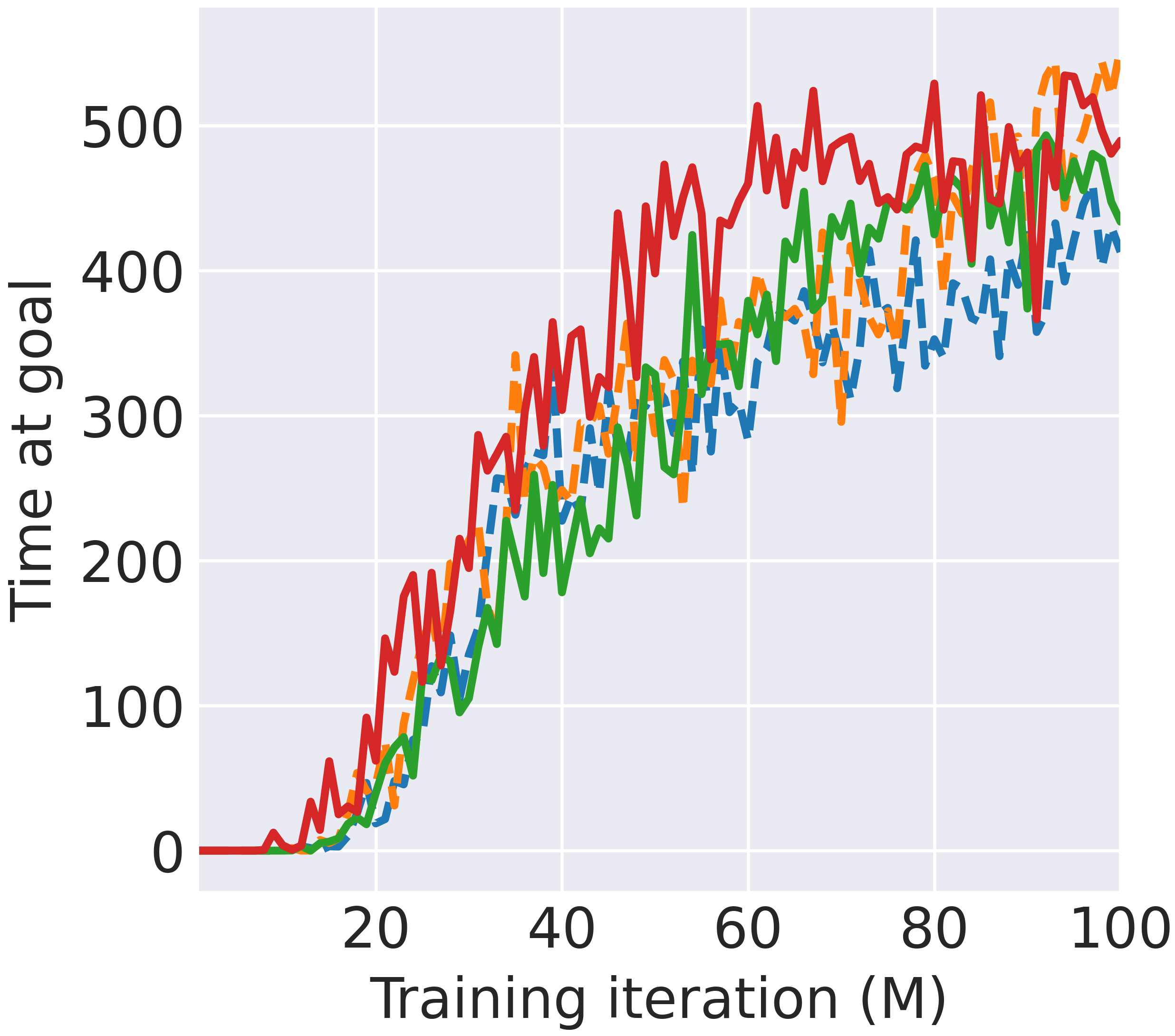}}
  	\caption{Ant Big Maze}
  	\label{subfig:RL-ant_big_maze}
  \end{subfigure}
  \begin{subfigure}[t]{0.245\textwidth}
  	\centerline{\includegraphics[width=\textwidth]{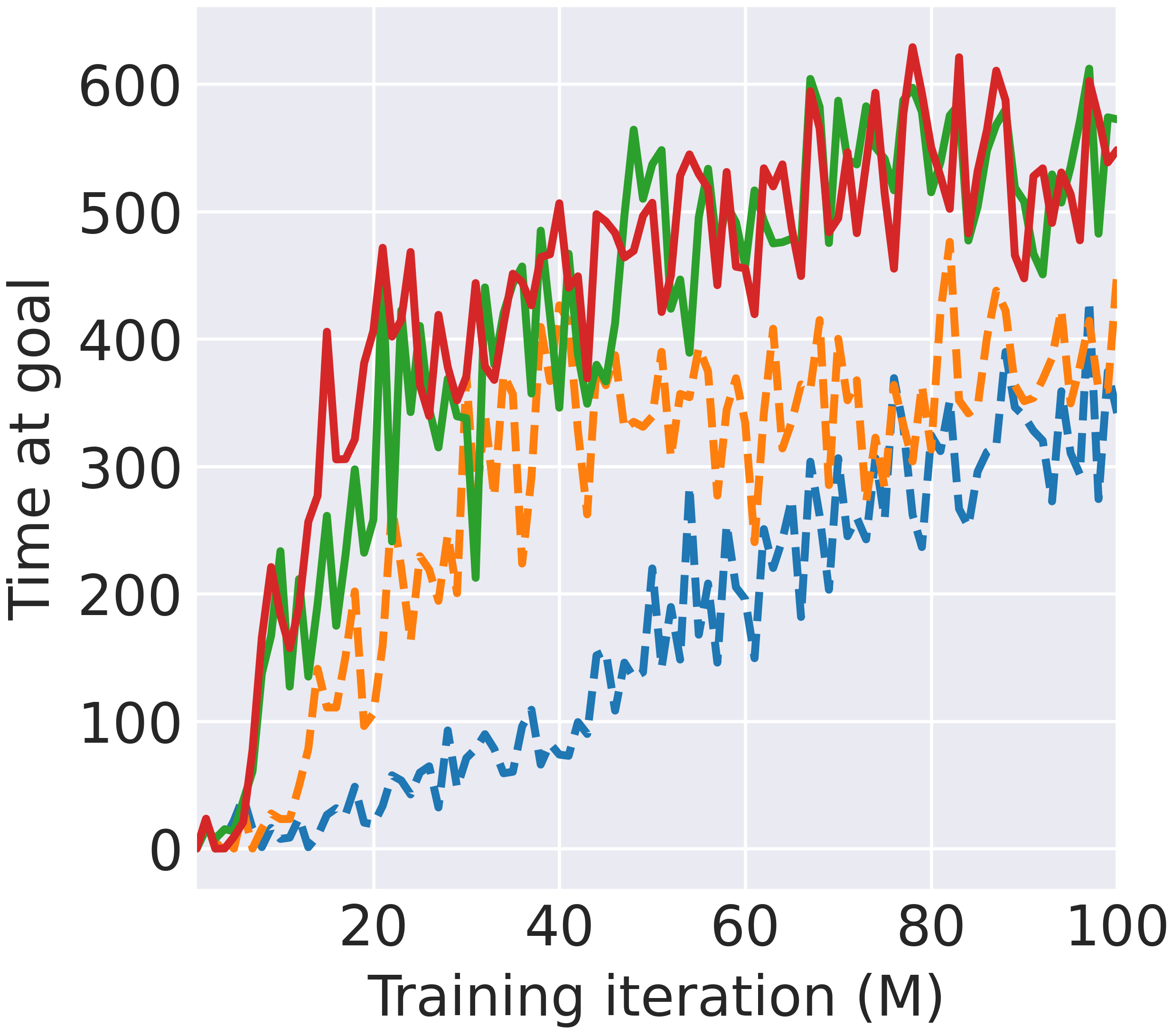}}
  	\caption{Arm Push Hard}
  	\label{subfig:RL-arm_push_hard}
  \end{subfigure}
    \begin{subfigure}[t]{0.245\textwidth}
  	\centerline{\includegraphics[width=\textwidth]{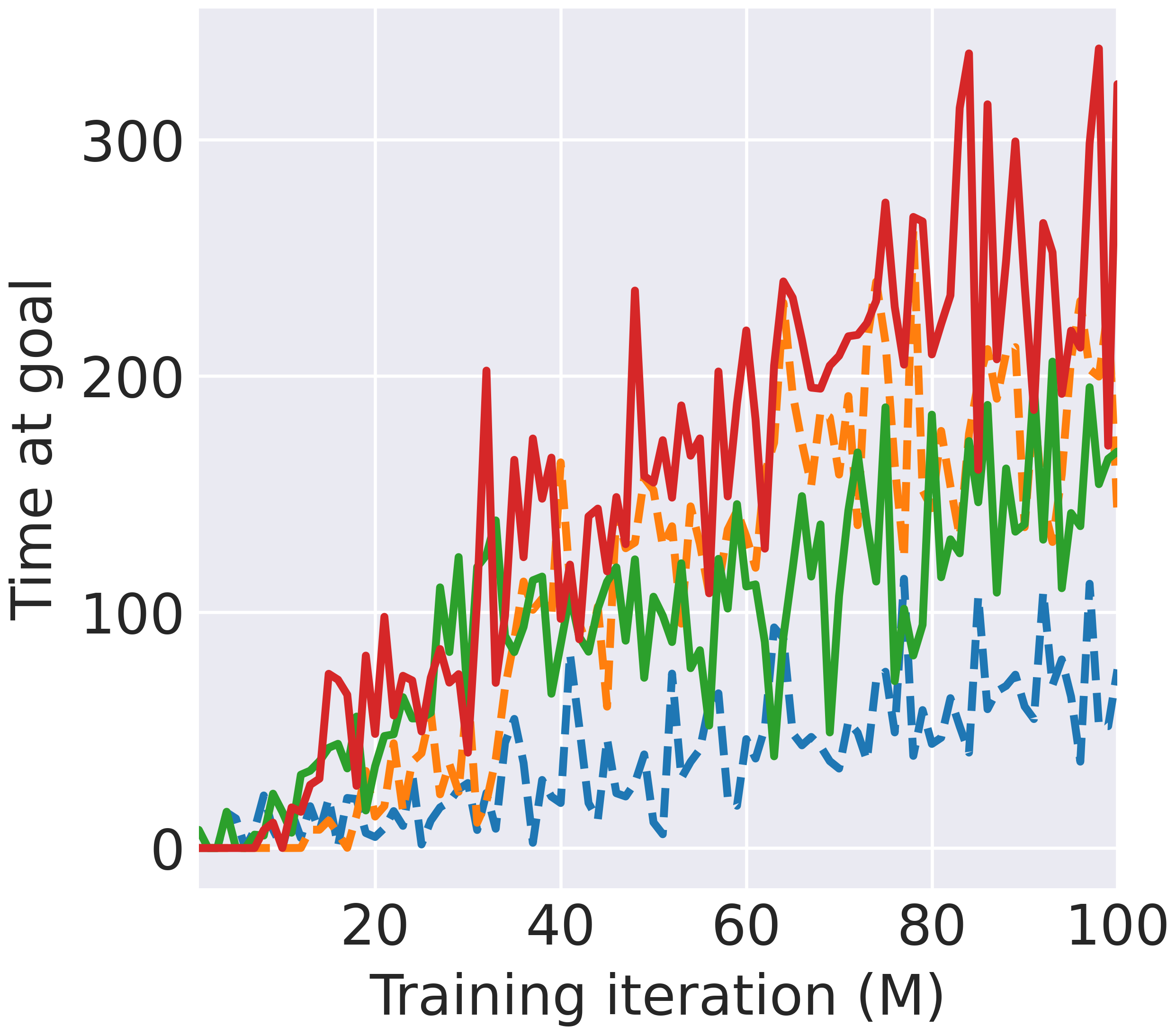}}
  	\caption{Arm Binpick Hard}
  	\label{subfig:RL-arm_binpick_hard}
  \end{subfigure}
 \vspace{-.1cm}
  \caption{Reinforcement learning using ResNets of 16 and 64 layers.}
  \label{fig:RL}
  \end{center}
  \vspace{-.1cm}
\end{figure}

\subsection{Pre-training of diffusion models}
Beyond LLMs, pre-training of diffusion models is also inspected following the experimental protocols in Diffusion Transformers (DiT)~\cite{DiT}. Specifically, class-conditional latent DiT-S/2 and DiT-B/2 are trained on the ImageNet-1K dataset~\cite{ImageNet}, which contains over 1.3 million images across 1,000 classes, at 256$\times$256 resolution for 400,000 iterations. All hyperparameters including the learning rate are set to the defaults in~\cite{DiT} without tuning. For evaluation, the standard Fr\'echet inception distance (FID)~\cite{FID} is calculated using 50,000 randomly sampled images from 250 DDPM steps~\cite{DDPM}, without classifier-free guidance (cfg) unless otherwise stated. Secondary metrics include spatial FID (sFID)~\cite{sFID}, inception score (IS)~\cite{IS}, and improved precision/recall~\cite{prec-rec}. 

The evolution of FID-50K over the pre-training iterations is outlined in Figure~\ref{fig:DiT-FID}. It is observed that the incorporation of ANCRe remarkably accelerates the convergence, and leads to improved final performance, thanks to the effective utilization of network depth. 
It is worth noting that the performance gains are relatively smaller than those observed for LLMs, as DiT are generally shallower. 
In addition to the FID-50K, Table~\ref{tab:DiT} reports auxiliary metrics computed using the pre-trained models for a more comprehensive comparison. Again, ANCRe consistently outperforms cascaded residual connections by an average of approximately $6\%$ among metrics, demonstrating its effectiveness for different data modalities.

\subsection{Reinforcement learning with ResNets}
The third test examines reinforcement learning (RL) with ``actions'' as the modality. The test settings follow from~\cite{GCRL}, which demonstrates that scaling model depth in RL can unlock new goal-reaching capabilities. The test involves an unsupervised goal-conditioned regime, in which rewards are only available upon reaching the commanded goals. 
Broadly speaking, this setting follows the same spirit as RLVR \cite{lambert2024tulu}.
Given that the sparse rewards provide rather limited feedback, the problem is considerably more challenging compared to standard RL scenarios supervised by dense rewards or expert demonstrations. Consequently, it necessitates deeper models of sufficient representational capacity. The models are ResNets~\cite{ResNet} of varying depths, with cascaded residual connections spanning consecutive blocks of four layers. ANCRe is applied at the same granularity of blocks, after removing the original cascaded residual connections. The RL benchmark encompasses four locomotion and manipulation tasks \{\texttt{Humanoid}, \texttt{Ant\_Big\_Maze}, \texttt{Arm\_Push\_Hard}, \texttt{Arm\_Binpick\_Hard}\} from the Gymnasium environments~\cite{gym}. Training is performed with the simple contrastive RL (CRL) algorithm~\cite{CRL} for 100 million environment steps, and all hyperparameters are consistent with~\cite{GCRL}. 

Figure~\ref{fig:RL} sketches the time-at-goal metric ($\uparrow$) averaged across episodes. For all four tasks, ANCRe consistently leads to considerably accelerated convergence than the vanilla ResNet baselines. Notably, when equipped with ANCRe, a 16-layer ResNet can match or even surpass the performance of a $4 \times$ deeper ResNet. These results further corroborate that standard ResNet architectures tend to underutilize increased network depth, and demonstrate that ANCRe provides an effective plug-and-play remedy for improving depth efficiency.

\begin{table}[t]
\begin{minipage}{0.4\textwidth}
    \caption{Perplexity under different normalization schemes.}
   \label{tab:normalization}
   \begin{center}
     \begin{tabular}{lcc}
       \toprule
       Normalization  & 130M         & 350M  \\
       \midrule
       None & diverge & diverge \\
       Outgoing~\eqref{eq:out-normalize} & 25.03 &	18.63 \\
       Ingoing~\eqref{eq:in-normalize}    & \textbf{24.48} &	 \textbf{18.32} \\
       \bottomrule
     \end{tabular}
   \end{center}
\end{minipage}
\qquad
\begin{minipage}{0.45\textwidth}
    \caption{Ablation study using all or learnable shortcuts.}
   \vspace{-.1cm}
   \label{tab:ablation}
   \begin{center}
     \begin{tabular}{cccc}
       \toprule
       All $i$:$j$  & Learnable  &  130M  &  350M \\
       \midrule
       $\surd$  & $\times$ & diverge & diverge  \\
       $\times$ & $\surd$ & 27.18 & 21.72 \\
       $\surd$  &  $\surd$ & \textbf{24.48} & \textbf{18.32} \\
       \bottomrule
     \end{tabular}
   \end{center}
\end{minipage}
   \vspace{-.1cm}
\end{table}

\subsection{Ablation study}
\label{sec:ablation}
The last group of tests perform ablation studies to justify ANCRe's design choices. The tests employ the setups as in Section~\ref{subsec:exp-LLM} using the LLaMA models with FullPT. 

The first test compares the two normalization schemes discussed in Section~\ref{subsec:ANCRe}. From Table~\ref{tab:normalization}, omitting normalization leads to divergence, since the coefficients can grow unbounded. It is further observed that ingoing normalization yields smaller fluctuations in the training loss and consequently achieves lower perplexity than outgoing normalization. This is because the former normalizes the inputs to each layer, so that computations are numerically stabler. 

The second ablation study showcases that the effectiveness of ANCRe does not stem from simply adding more connections or making coefficients learnable. To see this, we consider two variants of ANCRe: one containing all possible shortcuts $i$:$j$ ($i < j$) but with constant coefficients $c_{ij} = 1/j$, and another with learnable cascaded residual connection on each Transformer block. As shown in Table~\ref{tab:ablation}, these variants either diverges or achieves less competitive performance. This highlights that ANCRe’s gains arise from the synergistic interaction between its connection topology and adaptive coefficient learning rather than either component alone. 

Finally, we measure training runtime and GPU memory usage to assess the efficiency of ANCRe. As shown in Table~\ref{tab:time}, incorporating ANCRe introduces negligible computational and memory overhead across all model scales. In particular, the training throughput drop by $<1\%$, while the increase in peak GPU memory is consistently below 0.01GB. This efficiency arises because the forward and backward passes of the residual connections involve only a scalar–matrix multiplication and a matrix addition, which are ignorable compared to the dominant computations in Transformers. 
Together, these results demonstrate that ANCRe delivers performance gains with minimal overheads, ensuring its practical merits of extracting more performance from deep foundation models.


\begin{table}[t]
\begin{minipage}{0.6\textwidth}
   \caption{Training overhead comparison.}
   \vspace{-.2cm}
   \label{tab:time}
   \begin{center}
     \begin{small}
         \begin{tabular}{@{}lcccc@{}}
           \toprule
           Model, depth,    & \multirow{2}{*}{Method} & \multirow{2}{*}{Param} &  Mem & Throughput  \\
           GPU, batch &  &  & (GB) & (tok/s) \\
           \midrule
            &  Cascaded & 0 & 26.60 & 99.9k \\
           350M, 24,  & HC & 172k &  46.76 & 43.6k \\
           H100, 64 & mHC & 172k & 55.55 & 33.6k \\
               & ANCRe & 300 &  26.60 & 99.2k \\
           \midrule
            &  Cascaded & 0 & 30.38 & 17.3k \\
           1B, 24,  & HC & 345k &  50.52 & 12.0k \\
           H100, 32 & mHC & 345k & 58.50 & 10.3k \\
               & ANCRe & 300 &  30.38 & 17.3k \\
           \bottomrule
         \end{tabular}
     \end{small}
   \end{center}
\end{minipage}
~
\begin{minipage}{0.35\textwidth}
    \caption{Ablation study of learned topology pattern.}
   \label{tab:pattern}
   \begin{center}
     \begin{tabular}{lcc}
       \toprule
       Shortcuts  & 130M         & 350M  \\
       \midrule
       ANCRe & 24.48 & 18.32 \\
       First and prev. & 25.44 &	19.94 \\
       Retraining & 24.48 & 18.49 \\
       \bottomrule
     \end{tabular}
   \end{center}
\end{minipage}
   \vspace{-.1cm}
\end{table}

\subsection{Analysis of learned topology}
Figure~\ref{fig:LLM-skip-heatmap} suggests that the learned topology exhibits a consistent pattern, underscoring shortcut from the first layer. We next demonstrate that the advantage of ANCRe does not come from a trivial architectural modification. Specifically, we evaluate a fixed shortcut design that incorporates inputs from both the previous layer and the first layer. As shown in Table~\ref{tab:pattern}, this fixed topology performs remarkably worse than ANCRe, indicating that the observed gains cannot be attributed merely to the additional shortcut connections from the first layer. Instead, the ability to learn the connectivity pattern is crucial. 

To further investigate the role of the learned topology, we froze the final learned coefficients and retrained a new model from scratch. The resulting models retained nearly all of the original performance gains. This result suggests that the learned topology together with its routing coefficients, rather than the mere presence of additional aggregation pathways, captures a transferable and beneficial connectivity structure.

We conjecture that early-layer representations encode reusable low-level features that remain valuable throughout the network when appropriately scaled, whereas intermediate-layer representations become increasingly specialized and task-dependent with the growth of depth.

\section{Conclusion and outlook}
\label{sec:conclusion}
This paper investigated how the topology of residual connections fundamentally influences the convergence and depth efficiency of deep networks, and introduced adaptive neural connection reassignment (ANCRe) as a principled approach to learn the optimal residual connectivity. 
One major limitation of this work is that our analysis is limited to LNNs. Additionally, empirical experiments indicate that the benefits are particularly more effective for deep networks. Other potential cases where ANCRe can be less effective include architectures whose residual topology has already been extensively optimized, and tasks where intermediate representations are less reusable across depths.

Looking ahead, several directions have been added to the research agenda. First, to capture the optimization dynamics of nonlinear neural networks, the convergence analysis will be broadened to more general models beyond LNNs. Second, more effective schemes for parameterizing and normalizing residual connections tailored to foundation models such as Transformers will be explored, with an emphasis on improving stability and depth efficiency. Finally, pre-training tests will be scaled to larger models, for which efficient depth utilization becomes increasingly critical. 

\section*{Acknowledgments}
This work is partly supported by NSF grants 2212318, 2220292, and 2312547. B. Li and N. He are supported by Swiss National Science Foundation (SNSF) Starting Grant.

\clearpage
\bibliography{ref}

@STRING{IEEEPAMI = {IEEE Trans. Pattern Anal. Mach. Intel.}}

@STRING{JMLR = {J. Mach. Learn. Res.} }

@STRING{ICML = {Proc. Int. Conf. on Machine Learning (ICML)} }

@STRING{UAI = {Proc. Conf. on Uncertainty in Artificial Intelligence (UAI)}}

@STRING{ICLR = {Proc. Int. Conf. on Learning Representations (ICLR)} }

@STRING{NIPS = {Proc. Neural Information Processing Systems (NeurIPS)}}

@STRING{COLT = {Proc. Annual Conf. on Learning Theory (COLT)}}

@STRING{CVPR = {Proc. Conf. Computer Vision and Pattern Recognition (CVPR)}}

@STRING{ICCV = {Proc. Int. Conf. on Computer Vision (ICCV)}}

@string{ECCV = {Proc. European Conf. on Computer Vision (ECCV)}}

@string{ACL = {Proc. Conf. Assoc. Comput. Linguist. Meet. (ACL)}}

@article{GPT4,
  title={{GPT}-4 technical report},
  author={Achiam, Josh and Adler, Steven and Agarwal, Sandhini and Ahmad, Lama and Akkaya, Ilge and Aleman, Florencia Leoni and Almeida, Diogo and Altenschmidt, Janko and Altman, Sam and Anadkat, Shyamal and others},
  journal={arXiv preprint arXiv:2303.08774},
  year={2023}
}

@article{Codex,
  title={Evaluating large language models trained on code},
  author={Chen, Mark and Tworek, Jerry and Jun, Heewoo and Yuan, Qiming and Pinto, Henrique Ponde De Oliveira and Kaplan, Jared and Edwards, Harri and Burda, Yuri and Joseph, Nicholas and Brockman, Greg and others},
  journal={arXiv preprint arXiv:2107.03374},
  year={2021}
}

@inproceedings{Minerva,
 author = {Lewkowycz, Aitor and Andreassen, Anders and Dohan, David and Dyer, Ethan and Michalewski, Henryk and Ramasesh, Vinay and Slone, Ambrose and Anil, Cem and Schlag, Imanol and Gutman-Solo, Theo and Wu, Yuhuai and Neyshabur, Behnam and Gur-Ari, Guy and Misra, Vedant},
 booktitle = NIPS,
 pages = {3843--3857},
 title = {Solving Quantitative Reasoning Problems with Language Models},
 volume = {35},
 year = {2022}
}

@inproceedings{DDPM,
 author = {Ho, Jonathan and Jain, Ajay and Abbeel, Pieter},
 booktitle = {Advances in Neural Information Processing Systems},
 pages = {6840--6851},
 title = {Denoising Diffusion Probabilistic Models},
 volume = {33},
 year = {2020}
}

@inproceedings{DDIM,
title={Denoising Diffusion Implicit Models},
author={Jiaming Song and Chenlin Meng and Stefano Ermon},
booktitle=ICLR,
year={2021},
}

@inproceedings{score-matching,
title={Score-Based Generative Modeling through Stochastic Differential Equations},
author={Yang Song and Jascha Sohl-Dickstein and Diederik P Kingma and Abhishek Kumar and Stefano Ermon and Ben Poole},
booktitle=ICLR,
year={2021},
}

@InProceedings{CLIP,
  title = 	 {Learning Transferable Visual Models From Natural Language Supervision},
  author =       {Radford, Alec and Kim, Jong Wook and Hallacy, Chris and Ramesh, Aditya and Goh, Gabriel and Agarwal, Sandhini and Sastry, Girish and Askell, Amanda and Mishkin, Pamela and Clark, Jack and Krueger, Gretchen and Sutskever, Ilya},
  booktitle = 	 ICML,
  pages = 	 {8748--8763},
  year = 	 {2021},
  volume = 	 {139},
  month = 	 {18--24 Jul},
}

@InProceedings{PaLM-E,
  title = 	 {{P}a{LM}-E: An Embodied Multimodal Language Model},
  author =       {Driess, Danny and Xia, Fei and Sajjadi, Mehdi S. M. and Lynch, Corey and Chowdhery, Aakanksha and Ichter, Brian and Wahid, Ayzaan and Tompson, Jonathan and Vuong, Quan and Yu, Tianhe and Huang, Wenlong and Chebotar, Yevgen and Sermanet, Pierre and Duckworth, Daniel and Levine, Sergey and Vanhoucke, Vincent and Hausman, Karol and Toussaint, Marc and Greff, Klaus and Zeng, Andy and Mordatch, Igor and Florence, Pete},
  booktitle = 	 ICML,
  pages = 	 {8469--8488},
  year = 	 {2023},
  volume = 	 {202},
  month = 	 {23--29 Jul},
}

@inproceedings{GCRL,
title={1000 Layer Networks for Self-Supervised {RL}: Scaling Depth Can Enable New Goal-Reaching Capabilities},
author={Kevin Wang and Ishaan Javali and Micha{\l} Bortkiewicz and Tomasz Trzcinski and Benjamin Eysenbach},
booktitle=NIPS,
year={2025},
}

@InProceedings{ResNet,
author = {He, Kaiming and Zhang, Xiangyu and Ren, Shaoqing and Sun, Jian},
title = {Deep Residual Learning for Image Recognition},
booktitle = CVPR,
month = {June},
year = {2016}
}

@InProceedings{ReZero,
  title = 	 {Re{Z}ero is all you need: fast convergence at large depth},
  author =       {Bachlechner, Thomas and Majumder, Bodhisattwa Prasad and Mao, Henry and Cottrell, Gary and McAuley, Julian},
  booktitle = 	 UAI,
  pages = 	 {1352--1361},
  year = 	 {2021},
  volume = 	 {161},
  month = 	 {27--30 Jul},
}

@inproceedings{Transformer,
 author = {Vaswani, Ashish and Shazeer, Noam and Parmar, Niki and Uszkoreit, Jakob and Jones, Llion and Gomez, Aidan N and Kaiser, {\L}ukasz and Polosukhin, Illia},
 booktitle = NIPS,
 pages = {},
 title = {Attention is All you Need},
 volume = {30},
 year = {2017}
}

@article{depth-efficient,
  title={Do Language Models Use Their Depth Efficiently?},
  author={Csord{\'a}s, R{\'o}bert and Manning, Christopher D and Potts, Christopher},
  journal={arXiv preprint arXiv:2505.13898},
  year={2025}
}

@article{perception-encoder,
  title={Perception encoder: The best visual embeddings are not at the output of the network},
  author={Bolya, Daniel and Huang, Po-Yao and Sun, Peize and Cho, Jang Hyun and Madotto, Andrea and Wei, Chen and Ma, Tengyu and Zhi, Jiale and Rajasegaran, Jathushan and Rasheed, Hanoona and others},
  journal={arXiv preprint arXiv:2504.13181},
  year={2025}
}

@inproceedings{HC,
title={Hyper-Connections},
author={Defa Zhu and Hongzhi Huang and Zihao Huang and Yutao Zeng and Yunyao Mao and Banggu Wu and Qiyang Min and Xun Zhou},
booktitle=ICLR,
year={2025},
}

@article{mHC,
  title={mHC: Manifold-Constrained Hyper-Connections},
  author={Xie, Zhenda and Wei, Yixuan and Cao, Huanqi and Zhao, Chenggang and Deng, Chengqi and Li, Jiashi and Dai, Damai and Gao, Huazuo and Chang, Jiang and Zhao, Liang and others},
  journal={arXiv preprint arXiv:2512.24880},
  year={2025}
}

@inproceedings{ViT,
 author = {Wang, Yulin and Huang, Rui and Song, Shiji and Huang, Zeyi and Huang, Gao},
 booktitle = NIPS,
 pages = {11960--11973},
 title = {Not All Images are Worth 16x16 Words: Dynamic Transformers for Efficient Image Recognition},
 volume = {34},
 year = {2021}
}

@inproceedings{converge-analysis-LNN,
title={A Convergence Analysis of Gradient Descent for Deep Linear Neural Networks},
author={Sanjeev Arora and Nadav Cohen and Noah Golowich and Wei Hu},
booktitle=ICLR,
year={2019},
}

@InProceedings{exp-converge-LNN,
  title = 	 {Exponential Convergence Time of Gradient Descent for One-Dimensional Deep Linear Neural Networks},
  author =       {Shamir, Ohad},
  booktitle = 	 COLT,
  pages = 	 {2691--2713},
  year = 	 {2019},
  volume = 	 {99},
  month = 	 {25--28 Jun},
}

@InProceedings{GaLore,
  title = 	 {{G}a{L}ore: Memory-Efficient {LLM} Training by Gradient Low-Rank Projection},
  author =       {Zhao, Jiawei and Zhang, Zhenyu and Chen, Beidi and Wang, Zhangyang and Anandkumar, Anima and Tian, Yuandong},
  booktitle = 	 ICML,
  pages = 	 {61121--61143},
  year = 	 {2024},
  volume = 	 {235},
  series = 	 {Proceedings of Machine Learning Research},
  month = 	 {21--27 Jul},
}

@article{C4,
  author  = {Colin Raffel and Noam Shazeer and Adam Roberts and Katherine Lee and Sharan Narang and Michael Matena and Yanqi Zhou and Wei Li and Peter J. Liu},
  title   = {Exploring the Limits of Transfer Learning with a Unified Text-to-Text Transformer},
  journal = JMLR,
  year    = {2020},
  volume  = {21},
  number  = {140},
  pages   = {1--67},
}

@inproceedings{ReLoRA,
title={Re{L}o{RA}: High-Rank Training Through Low-Rank Updates},
author={Vladislav Lialin and Sherin Muckatira and Namrata Shivagunde and Anna Rumshisky},
booktitle=ICLR,
year={2024},
}

@article{llama,
  title={L{LaMA}: Open and Efficient Foundation Language Models},
  author={Touvron, Hugo and Lavril, Thibaut and Izacard, Gautier and Martinet, Xavier and Lachaux, Marie-Anne and Lacroix, Timoth{\'e}e and Rozi{\`e}re, Baptiste and Goyal, Naman and Hambro, Eric and Azhar, Faisal and others},
  journal={arXiv preprint arXiv:2302.13971},
  year={2023}
}

@InProceedings{DiT,
    author    = {Peebles, William and Xie, Saining},
    title     = {Scalable Diffusion Models with Transformers},
    booktitle = ICCV,
    month     = {October},
    year      = {2023},
    pages     = {4195-4205}
}

@inproceedings{ImageNet,
  author={Deng, Jia and Dong, Wei and Socher, Richard and Li, Li-Jia and Kai Li and Li Fei-Fei},
  booktitle=CVPR, 
  title={Image{N}et: A large-scale hierarchical image database}, 
  year={2009},
  volume={},
  number={},
  pages={248-255},
  }

@inproceedings{FID,
 author = {Heusel, Martin and Ramsauer, Hubert and Unterthiner, Thomas and Nessler, Bernhard and Hochreiter, Sepp},
 booktitle = NIPS,
 pages = {},
 title = {G{AN}s Trained by a Two Time-Scale Update Rule Converge to a Local Nash Equilibrium},
 volume = {30},
 year = {2017}
}

@InProceedings{sFID,
  title = 	 {Generating images with sparse representations},
  author =       {Nash, Charlie and Menick, Jacob and Dieleman, Sander and Battaglia, Peter},
  booktitle = 	 ICML,
  pages = 	 {7958--7968},
  year = 	 {2021},
  volume = 	 {139},
  month = 	 {18--24 Jul},
}

@inproceedings{IS,
 author = {Salimans, Tim and Goodfellow, Ian and Zaremba, Wojciech and Cheung, Vicki and Radford, Alec and Chen, Xi and Chen, Xi},
 booktitle = NIPS,
 pages = {},
 title = {Improved Techniques for Training GANs},
 volume = {29},
 year = {2016}
}

@inproceedings{prec-rec,
 author = {Kynk\"{a}\"{a}nniemi, Tuomas and Karras, Tero and Laine, Samuli and Lehtinen, Jaakko and Aila, Timo},
 booktitle = NIPS,
 pages = {},
 title = {Improved Precision and Recall Metric for Assessing Generative Models},
 volume = {32},
 year = {2019}
}

@inproceedings{CRL,
 author = {Eysenbach, Benjamin and Zhang, Tianjun and Levine, Sergey and Salakhutdinov, Russ R},
 booktitle = NIPS,
 pages = {35603--35620},
 title = {Contrastive Learning as Goal-Conditioned Reinforcement Learning},
 volume = {35},
 year = {2022}
}

@inproceedings{gym,
title={Gymnasium: A Standard Interface for Reinforcement Learning Environments},
author={Mark Towers and Ariel Kwiatkowski and John U. Balis and Gianluca De Cola and Tristan Deleu and Manuel Goul{\~a}o and Kallinteris Andreas and Markus Krimmel and Arjun KG and Rodrigo De Lazcano Perez-Vicente and J K Terry and Andrea Pierr{\'e} and Sander V Schulhoff and Jun Jet Tai and Hannah Tan and Omar G. Younis},
booktitle={Proc. Neural Information Processing Systems Datasets and Benchmarks Track},
year={2025},
}

@inproceedings{AdamW,
  title={Decoupled weight decay regularization},
  author={Loshchilov, Ilya and Hutter, Frank},
  booktitle=ICLR,
  year={2017}
}

@ARTICLE{DeepNet,
  author={Wang, Hongyu and Ma, Shuming and Dong, Li and Huang, Shaohan and Zhang, Dongdong and Wei, Furu},
  journal=IEEEPAMI, 
  title={Deep{N}et: Scaling Transformers to 1,000 Layers}, 
  year={2024},
  volume={46},
  number={10},
  pages={6761-6774},
}

@InProceedings{pmlr-v119-xiong20b,
  title = 	 {On Layer Normalization in the Transformer Architecture},
  author =       {Xiong, Ruibin and Yang, Yunchang and He, Di and Zheng, Kai and Zheng, Shuxin and Xing, Chen and Zhang, Huishuai and Lan, Yanyan and Wang, Liwei and Liu, Tieyan},
  booktitle = 	 ICML,
  pages = 	 {10524--10533},
  year = 	 {2020},
  volume = 	 {119},
  month = 	 {13--18 Jul},
}

@inproceedings{NeuTRENO,
 author = {Nguyen, Tam and Nguyen, Tan and Baraniuk, Richard},
 booktitle = NIPS,
 pages = {80233--80256},
 title = {Mitigating Over-smoothing in Transformers via Regularized Nonlocal Functionals},
 volume = {36},
 year = {2023}
}

@inproceedings{ResFormer,
    title = "Value Residual Learning",
    author = "Zhou, Zhanchao  and
      Wu, Tianyi  and
      Jiang, Zhiyun  and
      Obeid, Fares  and
      Lan, Zhenzhong",
    booktitle = ACL,
    month = jul,
    year = "2025",
    pages = "28341--28356",
}

@article{telgarsky2015representation,
  title={Representation benefits of deep feedforward networks},
  author={Telgarsky, Matus},
  journal={arXiv preprint arXiv:1509.08101},
  year={2015}
}

@article{li2018visualizing,
  title={Visualizing the loss landscape of neural nets},
  author={Li, Hao and Xu, Zheng and Taylor, Gavin and Studer, Christoph and Goldstein, Tom},
  journal=NIPS,
  volume={31},
  year={2018}
}

@article{tsallis-entropy,
  title={Possible generalization of {B}oltzmann-{G}ibbs statistics},
  author={Tsallis, Constantino},
  journal={Journal of statistical physics},
  volume={52},
  number={1},
  pages={479--487},
  year={1988},
  publisher={Springer}
}

@article{llama3,
  title={The {L}lama 3 herd of models},
  author={Grattafiori, Aaron and Dubey, Abhimanyu and Jauhri, Abhinav and Pandey, Abhinav and Kadian, Abhishek and Al-Dahle, Ahmad and Letman, Aiesha and Mathur, Akhil and Schelten, Alan and Vaughan, Alex and others},
  journal={arXiv preprint arXiv:2407.21783},
  year={2024}
}

@article{srivastava2015highway,
  title={Highway networks},
  author={Srivastava, Rupesh Kumar and Greff, Klaus and Schmidhuber, J{\"u}rgen},
  journal={arXiv preprint arXiv:1505.00387},
  year={2015}
}

@inproceedings{he2016identity,
  title={Identity mappings in deep residual networks},
  author={He, Kaiming and Zhang, Xiangyu and Ren, Shaoqing and Sun, Jian},
  booktitle=ECCV,
  pages={630--645},
  year={2016},
  organization={Springer}
}

@inproceedings{huang2017densely,
  title={Densely connected convolutional networks},
  author={Huang, Gao and Liu, Zhuang and Van Der Maaten, Laurens and Weinberger, Kilian Q},
  booktitle=CVPR,
  pages={4700--4708},
  year={2017}
}

@article{qwen3,
  title={Qwen3 technical report},
  author={Yang, An and Li, Anfeng and Yang, Baosong and Zhang, Beichen and Hui, Binyuan and Zheng, Bo and Yu, Bowen and Gao, Chang and Huang, Chengen and Lv, Chenxu and others},
  journal={arXiv preprint arXiv:2505.09388},
  year={2025}
}

@article{team2025gemma,
  title={Gemma 3 technical report},
  author={Team, Gemma and Kamath, Aishwarya and Ferret, Johan and Pathak, Shreya and Vieillard, Nino and Merhej, Ramona and Perrin, Sarah and Matejovicova, Tatiana and Ram{\'e}, Alexandre and Rivi{\`e}re, Morgane and others},
  journal={arXiv preprint arXiv:2503.19786},
  year={2025}
}

@article{zagoruyko2016wide,
  title={Wide residual networks},
  author={Zagoruyko, Sergey and Komodakis, Nikos},
  journal={arXiv preprint arXiv:1605.07146},
  year={2016}
}

@inproceedings{xie2017aggregated,
  title={Aggregated residual transformations for deep neural networks},
  author={Xie, Saining and Girshick, Ross and Doll{\'a}r, Piotr and Tu, Zhuowen and He, Kaiming},
  booktitle=CVPR,
  pages={1492--1500},
  year={2017}
}

@inproceedings{balduzzi2017shattered,
  title={The shattered gradients problem: If resnets are the answer, then what is the question?},
  author={Balduzzi, David and Frean, Marcus and Leary, Lennox and Lewis, JP and Ma, Kurt Wan-Duo and McWilliams, Brian},
  booktitle=ICML,
  pages={342--350},
  year={2017},
}

@inproceedings{touvron2021going,
  title={Going deeper with image transformers},
  author={Touvron, Hugo and Cord, Matthieu and Sablayrolles, Alexandre and Synnaeve, Gabriel and J{\'e}gou, Herv{\'e}},
  booktitle=ICCV,
  pages={32--42},
  year={2021}
}

@article{hardt2016identity,
  title={Identity matters in deep learning},
  author={Hardt, Moritz and Ma, Tengyu},
  journal={arXiv preprint arXiv:1611.04231},
  year={2016}
}

@article{haber2017stable,
  title={Stable architectures for deep neural networks},
  author={Haber, Eldad and Ruthotto, Lars},
  journal={Inverse problems},
  volume={34},
  number={1},
  pages={014004},
  year={2017},
  publisher={IOP Publishing}
}

@article{wu2019global,
  title={Global convergence of gradient descent for deep linear residual networks},
  author={Wu, Lei and Wang, Qingcan and Ma, Chao},
  journal=NIPS,
  volume={32},
  year={2019}
}

@inproceedings{du2019width,
  title={Width provably matters in optimization for deep linear neural networks},
  author={Du, Simon and Hu, Wei},
  booktitle=ICML,
  pages={1655--1664},
  year={2019},
}

@article{zou2020global,
  title={On the global convergence of training deep linear resnets},
  author={Zou, Difan and Long, Philip M and Gu, Quanquan},
  journal={arXiv preprint arXiv:2003.01094},
  year={2020}
}

@article{zhang2019fixup,
  title={Fixup initialization: Residual learning without normalization},
  author={Zhang, Hongyi and Dauphin, Yann N and Ma, Tengyu},
  journal={arXiv preprint arXiv:1901.09321},
  year={2019}
}

@article{de2020batch,
  title={Batch normalization biases residual blocks towards the identity function in deep networks},
  author={De, Soham and Smith, Sam},
  journal=NIPS,
  volume={33},
  pages={19964--19975},
  year={2020}
}

@inproceedings{kai2025,
  title={PoLAR: Polar-decomposed Low-Rank Adapter Representation},
  author={Lion, Kai and Zhang, Liang and Li, Bingcong and He, Niao},
  booktitle=NIPS,
  year={2025}
}

@inproceedings{hu2021lora,
  title={Lo{RA}: Low-Rank Adaptation of Large Language Models},
  author={Edward Hu and Yelong Shen and Phillip Wallis and Zeyuan Allen-Zhu and Yuanzhi Li and Shean Wang and Lu Wang and Weizhu Chen},
  booktitle=ICLR,
  year={2022},
}

@article{lambert2024tulu,
  title={Tulu 3: Pushing frontiers in open language model post-training},
  author={Lambert, Nathan and Morrison, Jacob and Pyatkin, Valentina and Huang, Shengyi and Ivison, Hamish and Brahman, Faeze and Miranda, Lester James V and Liu, Alisa and Dziri, Nouha and Lyu, Shane and others},
  journal={arXiv preprint arXiv:2411.15124},
  year={2024}
}

@article{zhang2025reflora,
  title={RefLoRA: Refactored Low-Rank Adaptation for Efficient Fine-Tuning of Large Models},
  author={Zhang, Yilang and Li, Bingcong and Giannakis, Georgios B},
  journal={arXiv preprint arXiv:2505.18877},
  year={2025}
}
\bibliographystyle{plain}

\clearpage
\appendix

\section{Additional related work}
\label{apdx:related-work}
\textbf{Scaling deep.}
Broadly speaking, our work is also related to mechanisms for scaling neural networks to greater depth. Normalization techniques are often believed to facilitate training at depth. For example, the normalization in pre-activation ResNets~\cite{he2016identity} and transformers~\cite{pmlr-v119-xiong20b} has been shown to improve training stability. Moreover, it is shown in~\cite{de2020batch} that batch normalization downscales the hidden activations on the residual branch by an order of square root of the network depth at initialization. Proper normalization techniques, in conjunction with residual connections, enable deeper LLMs~\cite{DeepNet} and vision transformers~\cite{touvron2021going}. 
At the same time, there are also works suggesting that normalization is not strictly necessary, as some of its early training benefits can be reproduced by carefully designed initialization; see e.g., \cite{zhang2019fixup,de2020batch}. Nevertheless, residual connections remain a central ingredient for modern deep neural networks.

\textbf{Architecture-optimizer co-design.} Recently, there has been a growing trend toward opening up the “black box” of neural networks and co-designing architectures and optimizers for more efficient training. For example, \cite{zhang2025reflora} studies the gauge invariance underlying the optimization of LoRA \cite{hu2021lora}, while \cite{kai2025} redesigns the architecture so that fine-tuning performance scales more favorably with width (i.e., adapter rank). Our work aligns with this trend by showing that residual connections can substantially reshape the loss landscape, and therefore should be treated as a design choice rather than a fixed architectural default. 

\section{Broader impact}
\label{apdx:broader-impact}
The theoretical analysis and residual topology learning of ANCRe are broadly applicable across diverse large-scale pre-training settings. ANCRe improves the efficiency of scaling network depth in foundation models, yielding accelerated convergence and stronger performance in applications such as LLMs and diffusion transformers. These gains translate into more efficient real-world systems while significantly reducing the total GPU hours required for pre-training, thereby lowering the energy consumption and carbon footprint of AI development. Caution is warranted when applying the method to generative tasks; model outputs should be carefully evaluated, and appropriate safeguards should be implemented to ensure the safety and reliability of the generated content.

\section{Missing proofs}
\label{apdx:proof}
This appendix offers the missing proofs for theories in the main paper. Our analysis leverages the following notations. 

\textbf{Additional notation.} $\langle \cdot, \cdot \rangle_\fro$, $\tr (\cdot)$, $\rank(\cdot)$ stand for Frobenius inner product, trace, and rank; $\lambda_i(\cdot)$ and $\sigma_i(\cdot)$ are the $i$-th largest eigenvalue and singular value.
$\diag(\bfv)$ is the diagonal matrix whose diagonal entries are from vector $\bfv$, while $\diag(\bfM)$ refers to the vector formed by the diagonals of matrix $\bfM$. For simplicity, the subscripts of $\loss_1$ and $\loss_2$ are dropped when the objective is clearly specified. 

Notably, rather than aiming to analyze the full dynamics of modern architectures, our convergence analysis is intended to support the conjecture that residual topology is an important architectural lever and to motivate the subsequent topology-learning approach. This conjecture is validated through experiments in Section~\ref{sec:experiments}. While our current analysis omits nonlinearity in order to better isolate the effects of network depth and residual topology, we view it as a starting point for future extensions to more complex models.

\subsection{Proof of Theorem~\ref{thm:LB}}
Before presenting the proof, we first simplify the optimization objective and look into the associated dynamics. 

Recall that $\bfX$ is orthogonal, so~\eqref{eq:obj-LB} can be equivalently rewritten as
\begin{equation}
\label{eq.prob-bad}
    \loss (t) = \frac{1}{2} \Bigg\| \bfW_3(t) \bfW_2(t) \big( \bfW_1(t) + \bfI_d \big) \bfX - \bfY \bigg\|_\fro^2 = \frac{1}{2} \Bigg\| \bfW_3(t) \bfW_2(t) \big( \bfW_1(t) + \bfI_d \big) - \bfY \bfX^\top \bigg\|_\fro^2.
\end{equation}
Defining $\bfA := \bfY \bfX^\top$, the optimization dynamics then boil down to
\begin{align*}
	\bfE(t) &= \bfW_3(t) \bfW_2(t) \big(\bfW_1(t) + \bfI_d \big) - \bfA, \\
	\nabla_{\bfW_1} \loss(t) &= \bfW_2(t)^\top \bfW_3^\top(t) \bfE(t),  \\
	\nabla_{\bfW_2} \loss(t) &= \bfW_3(t)^\top \bfE(t) \big(\bfW_1^\top(t) + \bfI_d\big), \\
	\nabla_{\bfW_3} \loss(t) &= \bfE(t) \big(\bfW_1^\top(t) + \bfI_d\big) \bfW_2^\top(t).
\end{align*}

Given these, the gradient flow~\eqref{eq:GF} can thus be written as
\begin{align}
\label{eq:GF-3layers}
	\dt{\bfW_1(t)} = - \nabla_{\bfW_1} \loss(t), ~\dt{\bfW_2(t)} = - \nabla_{\bfW_2} \loss(t) , ~\dt{\bfW_3(t)} = - \nabla_{\bfW_3} \loss(t). 
\end{align}

\begin{theorem}[Formal restatement of Theorem~\ref{thm:LB}]
\label{thm:LB-formal}
	There exists a problem instance of~\eqref{eq.prob-bad} with diagonally initialized weights, such that if $\bfW_1[d,d](0) \in [-0.5, 0]$ and $\bfW_2[d,d](0) = \bfW_3[d,d](0) \in (0, 0.5]$, then gradient flow~\eqref{eq:GF-3layers} can only converge sublinearly at
    \begin{equation*}
		\loss(t)  \ge \Omega(1/t^2), ~\forall t > 0.
    \end{equation*}
\end{theorem}

\begin{proof}
We construct a specific instance of $\bfA$ and analyze the resultant gradient flow dynamics. In particular, consider a diagonal matrix $\bfA$ with rank $r_A < d$. That is, $\bfA = \diag(\sigma_1, \sigma_2, \ldots, \sigma_{r_A}, 0, \ldots,  0)$, where the last singular value $\sigma_d = 0$. Furthermore, since $\bfW_1(0)$, $\bfW_2(0)$ and $\bfW_3(0)$ are diagonally initialized, it can be verified that the dynamics of gradient flow is diagonally separable throughout training. Specifically, letting $w_i(t)$, $u_i(t)$ and $v_i(t)$ be the $i$-th diagonal of $\bfW_1(t)+\bfI_d$, $\bfW_2(t)$ and $\bfW_3(t)$, the dynamics for $\forall i \in \{1, 2 \ldots, d\}$ can be written in a decoupled and coordinate-wise form
\begin{subequations}
\label{eq:diag-dynamics}
	\begin{align}
		a_i(t) & := w_i(t) u_i(t) v_i(t), \label{eq.a} \\
		\dt{w_i(t)} &= - u_i(t) v_i(t) \big( a_i(t) - \sigma_i \big), \\
		\dt{u_i(t)} &= - w_i(t) v_i(t) \big( a_i(t) - \sigma_i \big), \label{eq.u} \\
		\dt{v_i(t)} &= - w_i(t) u_i(t) \big( a_i(t) - \sigma_i \big).
	\end{align}
\end{subequations}

We now focus on the coordinate $i=d$, for which $\sigma_d = 0$. Moreover, since we initialize $u_d(0) = v_d(0)$, Lemma \ref{lemma.u=v} leads to $u_d(t)= v_d(t)$ for $\forall t \ge 0$. Combining this with~\eqref{eq:diag-dynamics} we obtain
\begin{align*}
	\dt{a_d(t)} &= u_d(t) v_d(t) \dt{w_d(t)} + w_d(t) v_d(t) \dt{u_d(t)} + w_d(t) u_d(t) \dt{v_d(t)} \\
	& = - u_d^4(t) a_d(t) - 2 w_d^2(t) u_d^2(t) a_d(t)  \\
	& = - \frac{a_d^3(t)}{w_d^2(t)} - 2 w_d(t) a_d^2(t).
\end{align*}

Moreover, the initialization yields $w_d(0) = 1 + \bfW_1[d,d](0) \in [0.5, 1]$ and $u_d(0) = v_d(0) \in (0, 0.5]$. It then follows from Lemma~\ref{lemma.w} that $0 \le u_d(t) = v_d(t) \le w_d(t) \leq 1$. 

Therefore, it holds
\begin{equation*}
    a_d(t) = w_d(t) u_d(t) v_d(t) = w_d(t) u^2_d(t) \in [0,1]. 
\end{equation*}
Using this, we can further obtain
\begin{align*}	
	\dt{a_d(t)} & = - \frac{a_d^3(t)}{w_d^2(t)} - 2 w_d(t) a_d^2(t) \\
	& \overset{(a)}{\geq} - \frac{a_d^3(t)}{w_d^2(t)} - 2 a_d^2(t) \\
	& \overset{(b)}{\geq} - a_d^{7/3}(t) - 2 a_d^2(t) \\
	& \overset{(c)}{\geq}  - 3 a_d^2(t) 
\end{align*}
where (a) is by Lemmas \ref{lemma.w}; (b) utilizes Lemma \ref{lemma.w2} with $a_d(t) > 0$; and (c) follows from $a_d(t) \leq 1$. 

The inequality above implies that $a_d(t) \geq \frac{a_d(0)}{1 + 3a_d(0) t}$ as a result of Lemma \ref{apdx.aux.lemma3}. Hence, we arrive at
\begin{align*}
	\loss(t) = \frac{1}{2} \sum_{i=1}^d \big(a_i(t) - \sigma_i \big)^2 \geq  \frac{1}{2}	\big(a_d(t) \big)^2 \geq \frac{a_d^2(0)}{2(1 + 3a_d(0) t)^2} = \Omega (1/t^2)
\end{align*}
where $a_d(0) = w_d(0) u_d^2(0) > 0$. The proof is thus completed.
\end{proof}

We remark that, while the proof is constructed using diagonal matrices, the optimization dynamics are equivalent to $\min_{\bfW_i} || (\mathbf{U} \bfW_3) \bfW_2 (\bfW_1 + \bfI_d) (\bfX \mathbf{V}^\top) - \mathbf{U} \bfA \mathbf{V}^\top || = || \tilde{\bfW}_3 \bfW_2 (\bfW_1 + \bfI_d) \tilde{\bfX} - \tilde{\bfA} ||$ for any orthogonal $\mathbf{U}$ and $\mathbf{V}$. In fact, this corresponds to a simplified setting where the singular directions of $\tilde{\bfA}$ are already known. In fact, the experiments in Figure~\ref{fig:LNN-converge} do not use diagonal $\mathbf{A}$ or initialization, yet the sublinear rate is still clearly observed.

\begin{lemma}\label{lemma.u=v}
	If $u_d(0) = v_d(0)$, then $u_d(t) = v_d(t),\,\forall t \ge 0$.
\end{lemma}
\begin{proof}
	This symmetry follows immediately from the observation that $\dt{u_d(t)} = \dt{v_d(t)}$ whenever $u_d(t) = v_d(t)$. As $u_d(0) = v_d(0)$, this symmetry always holds throughout the optimization. 
\end{proof}

\begin{lemma}\label{lemma.descent2}
	Under gradient flow~\eqref{eq:GF-3layers}, it holds that
	\begin{align*}
		\dt{a_d^2(t)}  \leq 0.
	\end{align*}
\end{lemma}
\begin{proof}
	It is straightforward to see that 
	\begin{align*}
		\frac{1}{2}\dt{a_d^2(t)} &= a_d(t) u_d(t) v_d(t) \dt{w_d(t)} + a_d(t) w_d(t) v_d(t) \dt{u_d(t)} + w_d(t) u_d(t) a_d(t) \dt{v_d(t)} \\
	& = - u_d^4(t) a_d^2(t) - 2 w_d^2(t) u_d^2(t) a_d^2(t)  \leq 0
	\end{align*}	
	where we used $u_d(t) = v_d(t)$ given by Lemma~\ref{lemma.u=v}, together with the dynamics in~\eqref{eq:diag-dynamics}. 
\end{proof}

\begin{lemma}
\label{lemma.w}
	If $0 < u_d(0) = v_d(0) \le w_d(0) \le 1$, then $0 \le u_d(t) = v_d(t) \le w_d(t) \leq 1$.
\end{lemma}
\begin{proof}
	Leveraging~\eqref{eq:diag-dynamics} and Lemma~\ref{lemma.u=v}, it can be readily verified that
	\begin{align*}
		\dt{w_d^2(t)} = \dt{u_d^2(t)} = -2w_d (t) u_d(t) v_d(t) a_d(t) = -2 a_d^2(t) \le 0.
	\end{align*}
	This implies 
    \begin{equation*}
        w_d^2(t) - w_d^2(0) = u_d^2(t) - u_d^2(0) := \Delta(t) \le 0,
    \end{equation*}
    which proves $w_d^2(t) \ge u_d^2(t)$ as $w_d(0) \ge u_d(0) > 0$. 

    Note that
    \begin{equation*}
        \Big| \dt{w_d(t)} \Big| = u_d^2(t) | a_d(t)| = u_d^4(t) | w_d(t)| \le u_d^4(0) | w_d(0)| \le 1. 
    \end{equation*}
    This suggests $w_d(t)$ is a polynomial ODE that is locally Lipschitz, it has no jump according to Picard–Lindel\"of theorem. As $w_d(0) \in (0, 1]$ and $w_d^2(t)$ is non-increasing, we must have $0 \le w_d(t) \le w_d(0) \le 1$. Likewise, it also holds $0 \le u_d(t) \le u_d(0) < 1$. 

    Combining these results, we conclude that $0 \le u_d(t) = v_d(t) \le w_d(t) \le 1$, which completes the proof.
    \end{proof}

\begin{lemma}\label{lemma.w2}
	If $0 < u_d(0) = v_d(0) \le w_d(0) \le 1$, then $w_d(t) \geq a_d^{1/3}(t)$.
\end{lemma}
\begin{proof}
    It directly follows from Lemmas~\ref{lemma.u=v} and~\ref{lemma.w} that
    \begin{equation*}
        a_d^2(t) = u_d^4(t) w_d^2(t) \le w_d^6(t).
    \end{equation*}
    Given that $w_d(t) \geq 0$ by Lemma \ref{lemma.w}, we reach at $w_d (t) \geq a_d^{1/3}(t)$. 
\end{proof}

\subsection{Proof of Theorem~\ref{thm:UB}}
Similar to~\eqref{eq.prob-bad}, the objective~\eqref{eq:obj-UB} can be simplified as
\begin{equation}
\label{eq.prob-good}
    \loss (t) = \frac{1}{2} \Bigg\| \bfW_3(t) \big( \bfW_2(t) \bfW_1(t) + \bfI_d \big) \bfX - \bfY \bigg\|_\fro^2 = \frac{1}{2} \Bigg\| \bfW_3(t) \big( \bfW_2(t) \bfW_1(t) + \bfI_d \big) - \bfA \bigg\|_\fro^2.
\end{equation}
The associated optimization dynamics are
\begin{subequations}
\begin{align}
	\bfE(t) &:= \bfW_3(t) \big(\bfW_2(t) \bfW_1(t) + \bfI_d \big) - \bfA, \\
	\nabla_{\bfW_1} \loss(t) &= \bfW_2(t)^\top \bfW_3^\top(t) \bfE(t),  \\
	\nabla_{\bfW_2} \loss(t) &= \bfW_3(t)^\top \bfE(t) \bfW_1^\top(t), \\
	\nabla_{\bfW_3} \loss(t) &= \bfE(t) \big(\bfW_2(t)\bfW_1(t) + \bfI_d \big)^\top.
\end{align}
\end{subequations}
And the gradient flow is the same as~\eqref{eq:GF-3layers}. 

\begin{theorem}[Formal restatement of Theorem~\ref{thm:UB}]
\label{thm:UB-formal}
	If $0 < \| \bfW_i (0) \|_\fro \le \delta := \sqrt{\frac{\lambda}{2}}  \exp\{ - \frac{\sqrt{\lambda \loss(0)}}{(1-\lambda)^2} - \frac{\sqrt{2\pi} \loss(0)}{2(1-\lambda)^3} \},\,i=1,2,3$ for some $\lambda \in (0,1)$, then gradient flow~\eqref{eq:GF-3layers} achieves linear convergence
	\begin{align*}
		\loss(t) \leq \loss(0) e^{-2(1 - \lambda)^2 t}, ~\forall t > 0.
	\end{align*}
\end{theorem}

\begin{proof}
    Recall from the last subsection that the ODEs involved are polynomial and locally Lipschitz, and thus have no jump. Then we can define $T := \min \{ t \ge 0 \mid \| \bfW_2(t) \bfW_1(t) \| \ge \lambda \}$ to be the first time satisfying $\| \bfW_2 \bfW_1 \| \ge \lambda$. Applying Lemma \ref{lemma.linear_before_T1}, it follows that
	\begin{align*}
		\loss(t) \leq \loss(0) e^{-2(1 - \lambda)^2 t},~\forall t \leq T.
	\end{align*}
	
	Next applying Lemma \ref{lemma.W1W2}, we acquire for all $t \leq T$ that
	\begin{align*}
		\| \bfW_2(t) \bfW_1(t) \|_2  & \leq 	\| \bfW_1(t) \|_\fro \| \bfW_2(t) \|_\fro \leq \frac{1}{2} \Big( \| \bfW_1(t) \|_\fro^2 +  \| \bfW_2(t) \|_\fro^2 \Big) \\
        &\le \frac{1}{2} \Big( \| \bfW_1(0) \|_\fro^2 + \| \bfW_2(0) \|_\fro^2  \Big) \exp\Big\{ \frac{2 \delta \sqrt{2 \loss(0)}  }{(1 - \lambda)^2 } + \frac{\sqrt{2\pi} \loss(0)}{ (1 - \lambda)^3 }   \Big\} \\
		& \le \delta^2 \exp\Big\{ \frac{2 \delta \sqrt{2 \loss(0)}  }{(1 - \lambda)^2 } + \frac{\sqrt{2\pi} \loss(0)}{ (1 - \lambda)^3 }   \Big\} \\
        &= \frac{\lambda}{2} \exp\Bigg\{ \frac{2 \sqrt{2 \loss(0)}  }{(1 - \lambda)^2 } \bigg(\delta - \sqrt{\frac{\lambda}{2}} \bigg) \Bigg\} \\
        &\overset{(a)}{\le} \frac{\lambda}{2} < \lambda
	\end{align*}
    where $(a)$ is due to $\delta = \sqrt{\frac{\lambda}{2}}  \exp\{ - \frac{\sqrt{\lambda \loss(0)}}{(1-\lambda)^2} - \frac{\sqrt{2\pi} \loss(0)}{2(1-\lambda)^3} \} \le \sqrt{\frac{\lambda}{2}}$. 
	As $\| \bfW_2(t) \bfW_1(t) \|_2$ has no jump, this indicates our choice of $\delta$ and $\lambda$ ensures $\| \bfW_2(t) \bfW_1(t) \|_2 < \lambda$ always holds; that is, $T = +\infty$.
\end{proof}

\begin{lemma}\label{lemma.descent}
	Under gradient flow~\eqref{eq:GF-3layers}, we have that
	\begin{align*}
		\dt{\loss} = - \sum_{i=1}^3 \| \nabla_{\bfW_i} \loss(t) \|_\fro^2 \leq 0.
	\end{align*}
\end{lemma}
\begin{proof}
	It is straightforward to see that 
	\begin{align*}
		\dt{\loss} & = \Big\langle \nabla_{\bfW_1} \loss(t), \dt{\bfW_1} \Big\rangle +  \Big\langle \nabla_{\bfW_2} \loss(t), \dt{\bfW_2} \Big\rangle +  \Big\langle \nabla_{\bfW_3} \loss(t), \dt{\bfW_3} \Big\rangle	 \\
		& = - \| \nabla_{\bfW_1} \loss(t) \|_\fro^2  - \| \nabla_{\bfW_2} \loss(t) \|_\fro^2  - \| \nabla_{\bfW_3} \loss(t) \|_\fro^2 \le 0. 
	\end{align*}	
	This concludes the proof.
\end{proof}

\begin{lemma}\label{lemma.linear_before_T1}
For $\lambda \in (0,1)$, defining $T := \min \{ t \ge 0 \mid \| \bfW_2(t) \bfW_1(t) \| \ge \lambda \}$, then the loss converges linearly at 
	\begin{align*}
		\loss(t) \leq \loss(0) e^{-2(1 - \lambda)^2 t},~ t \le T.
	\end{align*}
\end{lemma}
\begin{proof}
	We start with bounding
	\begin{align*}
		\| \nabla_{\bfW_3} \loss (t) \|_\fro^2 \overset{(a)}{\geq} \| \bfE(t) \|_\fro^2  \sigma_d^2 \big( \bfI_d + \bfW_2(t)\bfW_1(t) \big) \overset{(b)}{\geq} (1 - \lambda)^2 \| \bfE(t) \|_\fro^2
	\end{align*}
	where (a) is by Lemma~\ref{apdx.aux.lemma1}; and (b) comes from $\sigma_d( \bfI_d + \bfW_2(t) \bfW_1(t)) \geq 1 - \|\bfW_2(t) \bfW_1(t) \| \ge 1 - \lambda$ when $t \le T$.
	
	Now based on Lemma \ref{lemma.descent}, it follows that
	\begin{align*}
		\dt{\loss} =  - \sum_{i=1}^3 \| \nabla_{\bfW_i} \loss \|_\fro^2	\leq - \| \nabla_{\bfW_3} \loss \|_\fro^2 \leq -(1 - \lambda)^2 \| \bfE(t) \|_\fro^2 = -2(1 - \lambda)^2  \loss(t).
	\end{align*}
	This directly results in $\loss(t) \leq \loss(0) e^{-2(1 - \lambda)^2 t}$.
\end{proof}

\begin{lemma}\label{lemma.W3}
	Under gradient flow~\eqref{eq:GF-3layers}, it holds that
	\begin{align*}
		\| \bfW_3(t)\|_\fro \leq  \| \bfW_3(0) \|_\fro + \sqrt{ t  \loss(0)}. 
	\end{align*}
\end{lemma}

\begin{proof}
	From the gradient flow, it can be seen that
	\begin{align*}
		\bfW_3(t) = \bfW_3(0) - \int_0^t \nabla_{\bfW_3} \loss(s)  \dd s.	
	\end{align*}
	Further taking norm on both sides, we have that 
	\begin{align*}
		\| \bfW_3(t)\|_\fro & \leq \| \bfW_3(0) \|_\fro + \Big{\|} \int_0^t  \nabla_{\bfW_3} \loss(s)  \dd s  \Big{\|} \\
		& \leq \| \bfW_3(0) \|_\fro +  \int_0^t 1 \times \| \nabla_{\bfW_3} \loss(s)\|_\fro \dd s   \\
		& \overset{(a)}{\leq} \| \bfW_3(0) \|_\fro + \sqrt{ t \int_0^t   \| \nabla_{\bfW_3} \loss(s)\|_\fro^2 \dd s  } \\
		& \leq \| \bfW_3(0) \|_\fro + \sqrt{ t \int_0^t  \sum_{i=1}^3 \| \nabla_{\bfW_i} \loss(s) \|_\fro^2  \dd s  } \\
		& \overset{(b)}{=} \| \bfW_3(0) \|_\fro + \sqrt{ t  \big(\loss(0) - \loss(t) \big)}  \\
		& \overset{(c)}{\leq} \| \bfW_3(0) \|_\fro + \sqrt{ t  \loss(0)} 
	\end{align*}
	where (a) comes from Cauchy–Schwarz inequality; (b) uses Lemma \ref{lemma.descent}; (c) is because the loss is always greater than 0.
\end{proof}

\begin{lemma}\label{lemma.dW1=dW2}
	Under gradient flow~\eqref{eq:GF-3layers}, it follows that 
	\begin{equation*}
		\dt{\| \bfW_1(t) \|_\fro^2}	= \dt{\| \bfW_2(t) \|_\fro^2}.
	\end{equation*}
\end{lemma}
\begin{proof}
	Note that
	\begin{align*}
		\dt{\| \bfW_1(t) \|_\fro^2} = 2 \Big\langle 	 \bfW_1, \dt{\bfW_1(t)}  \Big\rangle = -2 \tr\big( \bfW_1^\top(t) \bfW_2^\top(t) \bfW_3^\top(t) \bfE(t) \big).
	\end{align*}
	Similarly, we have
	\begin{align*}
		\dt{\| \bfW_2(t) \|_\fro^2} & = 2 \Big\langle 	 \bfW_2, \dt{\bfW_2(t)}  \Big\rangle = -2 \tr\big( \bfW_2^\top(t) \bfW_3^\top(t) \bfE(t) \bfW_1^\top(t)\big) \\
		& = -2 \tr\big( \bfW_1^\top(t) \bfW_2^\top(t) \bfW_3^\top(t) \bfE(t) \big).
	\end{align*}
	The proof is thus completed.
\end{proof}

\begin{lemma}\label{lemma.W1W2}
	When $t \leq T$, it holds that
	\begin{align*}
		\| \bfW_1(t) \|_\fro^2 + \| \bfW_2(t) \|_\fro^2 \leq \Big( \| \bfW_1(0) \|_\fro^2 + \| \bfW_2(0) \|_\fro^2  \Big) e^M
	\end{align*}
	where $M$ is given by
	\begin{align*}
		M := \frac{2 \sqrt{2 \loss(0)}\| \bfW_3(0) \|_\fro  }{(1 - \lambda)^2 } + \frac{\sqrt{2\pi} \loss(0)}{ (1 - \lambda)^3 }.
	\end{align*}
\end{lemma}
\begin{proof}
    Using Lemma~\ref{lemma.dW1=dW2}, we obtain
	\begin{align*}
		\dt{\|\bfW_1 \|_\fro^2 + \|\bfW_2 \|_\fro^2}	 
        & = -4 \tr\big( \bfW_1^\top(t) \bfW_2^\top(t) \bfW_3^\top(t) \bfE(t) \big) \\
		& \leq 4 \| \bfW_1 (t) \|_\fro \| \bfW_2 (t) \|_\fro \| \bfW_3(t) \|_\fro \|  \bfE(t)\|_\fro \\
		& \leq 2 \| \bfW_3 (t) \|_\fro \|  \bfE(t)\|_\fro \cdot \Big( \| \bfW_1 (t) \|_\fro^2 +  \| \bfW_2 (t) \|_\fro^2  \Big) \\
		& \overset{(a)}{\leq} 2 \| \bfW_3 (t) \|_\fro \sqrt{2 \loss(0)} e^{- (1 -\lambda)^2 t} \cdot \Big( \| \bfW_1 (t) \|_\fro^2 +  \| \bfW_2 (t) \|_\fro^2  \Big) \\
		& \overset{(b)}{\leq} 2 \Big( \| \bfW_3(0) \|_\fro + \sqrt{ t  \loss(0)}  \Big)  \cdot \sqrt{2 \loss(0)} e^{- (1 -\lambda)^2 t} \cdot \Big( \| \bfW_1 (t) \|_\fro^2 +  \| \bfW_2 (t) \|_\fro^2  \Big)
	\end{align*}
	where (a) comes from Lemma \ref{lemma.linear_before_T1}, which shows that $ \loss(t) = \frac{1}{2}\| \bfE (t) \|_\fro^2 \leq \loss(0)e^{-2 (1 -\lambda)^2 t} $; and (b) is by Lemma \ref{lemma.W3}.

	Now applying Lemma \ref{apdx.aux.lemma2}, with $c_1 = 2 \sqrt{2 \loss(0)}\| \bfW_3(0) \|_\fro$, $c_2 = 2\sqrt{2} \loss(0) $, $c_3 = (1 - \lambda)^2$, and $x(t) = \|\bfW_1(t) \|_\fro^2 + \|\bfW_2(t) \|_\fro^2$, we obtain for $\forall t\le T$ that
    \begin{equation*}
        \|\bfW_1(t) \|_\fro^2 + \|\bfW_2(t) \|_\fro^2 \leq \Big( \|\bfW_1(0) \|_\fro^2 + \|\bfW_2(0) \|_\fro^2\Big) e^M,~~M = \frac{2 \sqrt{2 \loss(0)}\| \bfW_3(0) \|_\fro}{(1 - \lambda)^2} + \frac{\sqrt{2\pi} \loss(0)}{(1 - \lambda)^3}.
    \end{equation*}
    The proof is thus completed. 
\end{proof}

\subsection{Additional auxiliary lemmas}
This subsection provides useful lemmas for proving our main results. 

\begin{lemma}\label{apdx.aux.lemma3}
	If $\dt{a(t)} \geq -3 a^2(t)$, it holds that $a(t) \geq \frac{a(0)}{1 + 3 a(0) t}$.	
\end{lemma}
\begin{proof}
	We can rewrite $\dt{a(t)} \geq -3 a^2(t)$ as
	\begin{equation*}
		\frac{1}{a^2(t)} \dt{a(t)} \geq -3.
	\end{equation*}
	Integrating both sides renders
	\begin{equation*}
		- \frac{1}{a(s)} \Big|_0^t \geq - 3t
	\end{equation*}
    which concludes the proof. 
\end{proof}

\begin{lemma}\label{apdx.aux.lemma1}
	Given two matrices $\bfA \in \mathbb{R}^{r \times r}$ and $\bfB \in \mathbb{R}^{r \times r}$, it holds that $\| \bfA\bfB \|_\fro^2 \geq \sigma_r^2(\bfA)  \| \bfB \|_\fro^2$. 
\end{lemma}
\begin{proof}
	Let $\mathbf{b}_i$ be the $i$th column of $\bfB$. It then follows that
	\begin{align*}
		\| \bfA \bfB \|_\fro^2 & = \sum_{i=1}^r \| \bfA \mathbf{b}_i \|_2^2 = \sum_{i=1}^r \big\| \bfA \frac{\mathbf{b}_i  }{ \| \mathbf{b}_i \|} \big\|_2^2 \| \mathbf{b}_i \|_2^2 \\
		& \geq \min_{\| \mathbf{u} \| = 1}  \| \bfA \mathbf{u} \|^2 \sum_{i=1}^r  \| \mathbf{b}_i \|_2^2 = \sigma_r^2(\bfA)  \| \bfB \|_\fro^2.
	\end{align*}
    The proof is thus completed. 
\end{proof}

\begin{lemma}\label{apdx.aux.lemma2}
If $x(t) > 0,\, \forall t \le T$, and it holds for constant $c_1 \ge 0, c_2 \ge 0, c_3 > 0$ that
\begin{equation*}
	\dt{x} \leq ( c_1 + c_2 \sqrt{t}	 ) e^{-c_3 t} x(t),~\forall t \le T,
\end{equation*}
then we have
\begin{equation*}
    x(t) \leq x(0) e^{M}, \,\forall t \le T
\end{equation*}
where $M := \frac{c_1}{c_3} + \frac{\sqrt{\pi}}{2}c_2 c_3^{-3/2}$.
\end{lemma}
\begin{proof}
	To start with, since $x(t) > 0,\, \forall t \le T$, it follows
	\begin{equation*}
		\dt{\ln x(t)} = \frac{1}{x(t)} \dt{x(t)} \leq 	( c_1 + c_2 \sqrt{t}	 ) e^{-c_3 t}.
	\end{equation*}
	Integrating both sides, we acquire
	\begin{align*}
		\ln x(t) - \ln x(0) & \leq \int_0^t ( c_1 + c_2 \sqrt{s}) e^{-c_3 s } \dd s  \\
		& \leq \int_0^\infty 	( c_1 + c_2 \sqrt{s}	 ) e^{-c_3 s } \dd s  \\
		& \leq \frac{c_1}{c_3} + c_2 \int_0^\infty  \sqrt{s}	 e^{-c_3 s } \dd s.
	\end{align*}

	Next, notice that
	\begin{equation*}
		\int_0^\infty \sqrt{s} e^{-c_3 s } \dd s \overset{(a)}{=} \frac{1}{c_3^{3/2}} \int_0^\infty \sqrt{u} e^{-u } \dd u \overset{(b)}{=} \frac{1}{c_3^{3/2}}  \Gamma(\frac{3}{2}) \overset{(c)}{=} c_3^{-3/2}  \frac{1}{2}\Gamma(\frac{1}{2}) = \frac{\sqrt{\pi}}{2} c_3^{-3/2}
	\end{equation*}
	where (a) is by change of variable $u = c_3 s $; (b) uses the definition of Gamma function $\Gamma(z) = \int_{0}^\infty t^{z-1} e^{-t} \dd t$; and (c) relies on $\Gamma(z+1) = z \Gamma(z)$ and $\Gamma(1/2) = \sqrt{\pi}$.
	
	Combining these together offers
	\begin{equation*}
		\ln x(t) - \ln x(0) \leq \frac{c_1}{c_3} + \frac{\sqrt{\pi}}{2}c_2 c_3^{-3/2} := M.
	\end{equation*}
	The proof is thus completed.
\end{proof}

\subsection{Extension to more than 3 layers}
\label{apdx:extension-morelayers}

This subsection extends Theorems~\ref{thm:LB-formal} and~\ref{thm:UB-formal} to LNNs with arbitrary depth $K \ge 3$. Following the preceding subsections, we define $\bfA := \bfY\bfX^\top$ and work with the corresponding reduced objectives after whitening. 

First consider the slow-convergence case, namely a $K$-layer LNN with the 0:1 shortcut
\begin{equation} 
\label{eq:prob-bad-K} 
    \loss(t) := \frac{1}{2} \Bigg\| \bfW_K(t)\cdots\bfW_2(t) \big(\bfW_1(t)+\bfI_d\big) - \bfA \Bigg\|_\fro^2. 
\end{equation} 

\begin{theorem}[Extension of Theorem~\ref{thm:LB-formal}] \label{thm:LB-formal-K} 
For any $K \ge 3$, there exists a problem instance of~\eqref{eq:prob-bad-K} with diagonally initialized weights, such that if 
\begin{equation*} 
    \bfW_1[d,d](0) \in [-0.5,0], \qquad \bfW_2[d,d](0) = \cdots = \bfW_K[d,d](0) \in (0,0.5], 
\end{equation*} 
then gradient flow~\eqref{eq:GF} satisfies
\begin{equation} 
\label{eq:LB-K} 
    \loss(t) \ge \Omega\left( 1 / t^\frac{K-1}{K-2} \right), ~\forall t > 0.
\end{equation}
\end{theorem} 

\begin{proof} 
Similar to the proof of Theorem~\ref{thm:LB-formal}, consider a rank-deficient diagonal matrix 
\begin{equation*} 
    \bfA = \diag(\sigma_1,\ldots,\sigma_{r_A},0,\ldots,0), \qquad r_A<d. 
\end{equation*}
Since all weights are diagonally initialized, the gradient flow dynamics remain diagonal throughout training. 

Likewise, we focus on the last coordinate, for which $\sigma_d=0$, and define similarly 
\begin{equation*} 
    w(t) := 1+\bfW_1[d,d](t), \qquad u_k(t) := \bfW_k[d,d](t), \quad k=2,\ldots,K. 
\end{equation*} 
The corresponding end-to-end scalar is
\begin{equation} 
    \label{eq:a-K} a(t) := w(t)\prod_{k=2}^K u_k(t). 
\end{equation} 
The coordinate-wise gradient flow dynamics are
\begin{subequations}
\begin{align} 
    \dt{w(t)} &= - w(t) \prod_{k=2}^K u_k^2(t), \label{eq:w-K-dynamics} \\ 
    \dt{u_j(t)} &= - w^2(t)u_j(t) \prod_{\substack{k=2\\k\neq j}}^K u_k^2(t), \qquad j=2,\ldots,K. \label{eq:u-K-dynamics} 
\end{align} 
\end{subequations}
Since $u_2(0)=\cdots=u_K(0)$ and the dynamics in~\eqref{eq:u-K-dynamics} are symmetric with respect to $u_2,\ldots,u_K$, uniqueness of the solution gives 
\begin{equation} \label{eq:u-K-equal} 
    u_2(t) = \cdots = u_K(t) := u(t), \qquad \forall t\ge 0. 
\end{equation} 
Consequently, it holds
\begin{equation} \label{eq:a-K-simple} 
    a(t)=w(t)u^{K-1}(t), 
\end{equation} 
and~\eqref{eq:w-K-dynamics}--\eqref{eq:u-K-dynamics} reduce to
\begin{equation} \label{eq:wu-K-simple}
    \dt{w(t)} = -w(t)u^{2K-2}(t), \qquad \dt{u(t)} = -w^2(t)u^{2K-3}(t). 
\end{equation} 
It follows immediately that 
\begin{equation} \label{eq:balance-K-slow} 
    \dt{w^2(t)} = \dt{u^2(t)} = -2w^2(t)u^{2K-2}(t) = -2a^2(t).
\end{equation} 
As a result, we obtain $w^2(t)-u^2(t) = w^2(0)-u^2(0) \ge0$, where the last inequality follows from $ w(0) = 1+\bfW_1[d,d](0) \in[0.5,1] \ge u (0) \in (0, 0.5]$.

Moreover, following the same process of Lemma~\ref{lemma.w}, it can be inferred from~\eqref{eq:wu-K-simple} and~\eqref{eq:balance-K-slow} that 
\begin{equation} \label{eq:ordering-K-slow} 
    0<u(t)\le w(t)\le1, \qquad 0<a(t)=w(t)u^{K-1}(t)\le 1, \qquad \forall t\ge0. 
\end{equation} 

Next, we derive the dynamics of $a(t)$. Using~\eqref{eq:a-K-simple} and \eqref{eq:wu-K-simple}, we obtain 
\begin{align} 
    \dt{a(t)} &= u^{K-1}(t)\dt{w(t)} + (K-1)w(t)u^{K-2}(t)\dt{u(t)} \nonumber\\ 
    &= -a(t) \left[ u^{2K-2}(t) + (K-1)w^2(t)u^{2K-4}(t) \right]. \label{eq:a-K-dynamics} 
\end{align} 
Since $u(t)\le w(t)$, we have 
\begin{equation*} 
    a(t) = w(t)u^{K-1}(t) \le w^K(t), 
\end{equation*}
and therefore 
\begin{equation} \label{eq:w-lower-K} 
    w(t)\ge a^{1/K}(t). 
\end{equation} 
Denote by $q := \frac{2(K-2)}{K-1}$. Using~\eqref{eq:w-lower-K}, it holds that
\begin{align*} 
    u^{2K-2}(t) &= \frac{a^2(t)}{w^2(t)} \le a^{2-2/K}(t) \le a^q(t)
\end{align*} 
where the last inequality follows from $0<a(t)\le1$ and 
$2-\frac{2}{K} \ge \frac{2(K-2)}{K-1}$. 

Furthermore, it follows from $0<w(t)\le1$ that
\begin{equation*} 
    w^2(t)u^{2K-4}(t) \le w^{\frac{2(K-2)}{K-1}}(t) u^{2K-4}(t) = \big[ w(t)u^{K-1}(t) \big]^{\frac{2(K-2)}{K-1}} = a^q(t). 
\end{equation*} 
Substituting these two bounds into~\eqref{eq:a-K-dynamics} yields 
\begin{equation*} \label{eq:a-K-diff-ineq} 
    \dt{a(t)} \ge -K a^{1+q}(t). 
\end{equation*} 
Since $a(t)>0$, this implies 
\begin{equation*} 
    \dt{a^{-q}(t)} = -q a^{-q-1}(t)\dt{a(t)} \le Kq. 
\end{equation*}
Integrating both sides gives 
\begin{equation} \label{eq:a-K-lower} 
    a(t) \ge \left( a^{-q}(0)+Kqt \right)^{-1/q}. 
\end{equation} 
Finally, since the last diagonal entry of $\bfA$ is zero, it holds
\begin{align*} 
    \loss(t) &\ge \frac12 a^2(t) \\ 
    &\ge \frac12 \left( a^{-q}(0)+Kqt \right)^{-2/q} \\ 
    &= \Omega\left( \frac{1}{t^{(K-1)/(K-2)}} \right)
\end{align*} 
where we used $\frac{2}{q} = \frac{K-1}{K-2}$. This completes the proof. 
\end{proof} 

We next consider the fast-convergence case, namely a $K$-layer LNN with the 0:$K\!-\!1$ shortcut, whose loss is
\begin{equation} \label{eq:prob-good-K} 
    \loss(t) := \frac12 \left\| \bfW_K(t) \left( \prod_{k=1}^{K-1}\bfW_k(t) + \bfI_d \right) - \bfA \right\|_\fro^2. 
\end{equation} 
For notational convenience, define 
\begin{equation} \label{eq:P-K} 
    \bfP(t) := \prod_{k=1}^{K-1}\bfW_k(t) = \bfW_{K-1}(t)\cdots\bfW_1(t),
\end{equation} 
and error matrix
\begin{equation*} 
    \bfE(t) := \bfW_K(t)\big(\bfP(t)+\bfI_d\big)-\bfA. 
\end{equation*} 
Then, the following theorem establishes the linear convergence of this residual topology.

\begin{theorem}[Extension of Theorem~\ref{thm:UB-formal}] \label{thm:UB-formal-K} 
Let $K\ge3$ and $\lambda\in(0,1)$. Define constants
\begin{subequations}
\begin{align} 
    M_K &:= \frac{2\sqrt{2\loss(0)}}{(1-\lambda)^2} + \frac{\sqrt{2\pi}\loss(0)}{(1-\lambda)^3}, \label{eq:M-K} \\ 
    R_K &:= \left(\frac{\lambda}{2}\right)^{\frac{2}{K-1}}, \label{eq:R-K} \\
    \delta_K &:= \min\left\{ 1,\, \sqrt{\frac{R_K}{2}} \exp\left\{ -\frac{M_K}{2} R_K^{(K-3)/2} \right\} \right\} \label{eq:delta-K}. 
\end{align}
\end{subequations}
If $0< \|\bfW_k(0)\|_\fro \le \delta_K, \, k=1,\ldots,K$, then gradient flow~\eqref{eq:GF} achieves linear convergence 
\begin{equation} 
    \label{eq:UB-K} \loss(t) \le \loss(0) e^{-2(1-\lambda)^2t},~ \forall t \ge 0. 
\end{equation} 
\end{theorem} 

\begin{proof} 
We first establish the linear convergence rate as long as $\|\bfP(t)\|_2<\lambda$. From~\eqref{eq:prob-good-K}, it follows
\begin{equation*} 
    \nabla_{\bfW_K}\loss(t) = \bfE(t) \big(\bfP(t)+\bfI_d\big)^\top. 
\end{equation*} 
Hence, by Lemma~\ref{apdx.aux.lemma1}, it follows
\begin{align*}
\|\nabla_{\bfW_K}\loss(t)\|_\fro^2 \ge \sigma_d^2\big(\bfI_d+\bfP(t)\big) \|\bfE(t)\|_\fro^2 \ge (1-\lambda)^2 \|\bfE(t)\|_\fro^2
\end{align*} 
where we used $\sigma_d\big(\bfI_d+\bfP(t)\big) \ge 1-\|\bfP(t)\|_2 > 1-\lambda$. 

Together with the gradient flow descent identity (cf. Lemma~\ref{lemma.descent})
\begin{equation*} 
    \dt{\loss(t)} = - \sum_{k=1}^K \|\nabla_{\bfW_k}\loss(t)\|_\fro^2, 
\end{equation*} 
we can establish the linear convergence
\begin{equation} \label{eq:linear-before-T-K} 
    \loss(t) \le \loss(0)e^{-2(1-\lambda)^2t} 
\end{equation} 
whenever $\|\bfP(t)\|_2<\lambda$. 

We next control the norms of the first $K-1$ layers. For $j\in\{1,\ldots,K-1\}$, define 
\begin{equation*} 
    \bfP_{>j}(t) := \bfW_{K-1}(t)\cdots\bfW_{j+1}(t), \qquad \bfP_{<j}(t) := \bfW_{j-1}(t)\cdots\bfW_1(t). 
\end{equation*} 
Then we can write the gradient as
\begin{equation*} 
    \nabla_{\bfW_j}\loss(t) = \bfP_{>j}^\top(t) \bfW_K^\top(t) \bfE(t) \bfP_{<j}^\top(t). 
\end{equation*} 
Therefore, it holds
\begin{align} 
    \dt{\|\bfW_j(t)\|_\fro^2} &= -2 \tr\left( \bfW_j^\top(t) \bfP_{>j}^\top(t) \bfW_K^\top(t) \bfE(t) \bfP_{<j}^\top(t) \right) \nonumber\\ 
    &= -2 \tr\left( \bfP^\top(t) \bfW_K^\top(t) \bfE(t) \right), \label{eq:balanced-K} 
\end{align} 
where the last equality follows from the cyclic property of trace. 

Thus, we obtain similarly
\begin{equation} \label{eq:balanced-K-difference} 
    \dt{\|\bfW_1(t)\|_\fro^2} = \cdots = \dt{\|\bfW_{K-1}(t)\|_\fro^2}. 
\end{equation} 
Now define $y(t) := \max_{1\le k\le K-1} \|\bfW_k(t)\|_\fro^2$. Since all squared norms in~\eqref{eq:balanced-K-difference} have the same derivative, their pairwise differences remain constant, and hence $y(t)$ is differentiable. Using~\eqref{eq:balanced-K} leads to
\begin{align} 
    \dt{y(t)} &\le 2 \|\bfP(t)\|_\fro \|\bfW_K(t)\|_\fro \|\bfE(t)\|_\fro \nonumber\\ 
    &\le 2 \|\bfW_K(t)\|_\fro \|\bfE(t)\|_\fro y^\frac{K-1}{2}(t). \label{eq:y-K-dynamics} 
\end{align} 

Akin to Lemma~\ref{lemma.W3}, the gradient flow descent identity yields 
\begin{equation} \label{eq:WK-K-bound} 
    \|\bfW_K(t)\|_\fro \le \|\bfW_K(0)\|_\fro + \sqrt{t\loss(0)} \le 1+\sqrt{t\loss(0)}, 
\end{equation} where the last inequality follows from $\delta_K\le1$. 

We next use a bootstrap argument similar to the proof of Theorem~\ref{thm:UB-formal}. In particular, let $T$ be the first time at which either $\|\bfP(t)\|_2=\lambda$ or $y(t)=R_K$. For all $t\le T$, combining~\eqref{eq:linear-before-T-K}, \eqref{eq:y-K-dynamics}, and~\eqref{eq:WK-K-bound} gives 
\begin{align*} 
    \dt{y(t)} &\le 2 \big( 1+\sqrt{t\loss(0)} \big) \sqrt{2\loss(0)} e^{-(1-\lambda)^2t} y^{(K-1)/2}(t) \\ 
    &\le 2 \big( 1+\sqrt{t\loss(0)} \big) \sqrt{2\loss(0)} e^{-(1-\lambda)^2t} R_K^{(K-3)/2} y(t). 
\end{align*} 

Applying Lemma~\ref{apdx.aux.lemma2} yields 
\begin{equation} \label{eq:y-K-bound} 
    y(t) \le y(0) \exp\left\{ M_K R_K^{(K-3)/2} \right\}, ~ t\le T. 
\end{equation} 
Since $y(0)\le\delta_K^2$, the choice~\eqref{eq:delta-K} gives 
\begin{equation} \label{eq:y-K-half} 
    y(t) \le \frac{R_K}{2}, ~ t\le T. 
\end{equation} 
Therefore, $y(t)$ cannot reach $R_K$. Moreover, \begin{align*} \|\bfP(t)\|_2 &\le \prod_{k=1}^{K-1} \|\bfW_k(t)\|_\fro \\ &\le y^{(K-1)/2}(t) \\ &\le \left( \frac{R_K}{2} \right)^{(K-1)/2} \\ &< R_K^{(K-1)/2} = \frac{\lambda}{2} < \lambda. \end{align*} Hence $\|\bfP(t)\|_2$ cannot reach $\lambda$ either. By continuity, such a finite $T$ cannot exist, and therefore 
\begin{equation*} 
    \|\bfP(t)\|_2<\lambda, ~ \forall t\ge0. 
\end{equation*} 
Applying~\eqref{eq:linear-before-T-K} for all $t\ge0$ proves the desired linear convergence.
\end{proof} 

Theorems~\ref{thm:LB-formal-K} and~\ref{thm:UB-formal-K} extend the exponential discrepancy established for the 3-layer case to arbitrary depth. In particular, for every $K\ge3$, there exist problem instances and sufficiently small initializations for which the 0:1 shortcut exhibits only polynomial convergence, whereas the 0:$K\!-\!1$ shortcut guarantees linear convergence under sufficiently small initialization. Interestingly, the lower bound in~\eqref{eq:LB-K} becomes slower as the network grows deeper, with exponent
\begin{equation*}
    \lim_{K\to\infty} (K-1)/(K-2) = 1. 
\end{equation*}

\section{Integration with gradient checkpointing and pipeline parallelism}
\label{apdx:gradckpt-pipeline}
In practice, gradient checkpointing and pipeline parallelism/model sharding are standard solutions to training large models that do not fit into a single GPU. This appendix describes how to integrate ANCRe with these techniques. 

For gradient checkpointing, a practical remedy is to restrict ingoing shortcuts to originate only from checkpointed layers, together with the immediately preceding layer. In our formulation, this corresponds to fixing $p_{ij} = 0$ strategically, thereby reducing the checkpointing burden while preserving the flexibility of learned residual topology. We believe this is more scalable to larger models and could still outperform vanilla training. 

Furthermore, one can readily add to the training loss function a Tsallis entropy regularization on the shortcut parameters $c_{ij}$, which encourages sparse connectivity for softmax. For model sharding, this can be used to reduce the inter-GPU communication by sparsifying the cross-GPU connectivity. 

Due to resource limitations, pre-training 1B models is at our full capacity, and we leave larger experiments as a future direction.

\section{Experimental setups}
\label{apdx:experiment-setups}
This appendix offers the detailed setups for reproducing the results reported in the main paper. All our codes are written in python. The LLM pretraining and RL tests are performed on a computing node with 4$\times$A100 GPUs, while the diffusion models are trained on 2$\times$H100 GPUs. 

\subsection{Case study}
The case study is based on synthetic data, where $\mathbf{A}$ is generated from standard Gaussian with a low-rank constraint $r=4$, $\mathbf{W}_i$ relies on small initialization $\mathcal{N}(0, 10^{-4})$, $d = n = 50$, the optimizer is standard GD with fixed step size of $10^{-3}$. 

\subsection{Datasets}
In the following, we provide a brief introduction to the datasets used in our tests. 

\textbf{Colossal Clean Crawled Corpus (C4)~\cite{C4}} is a large-scale English text dataset derived from the April 2019 Common Crawl snapshot consisting of cleaned web text. It was introduced as part of the development of Google's T5 model~\cite{C4}, and has been widely used as a standard pretraining corpus for LLMs due to its scale, diversity, and high quality. The dataset is available online from TensorFlow datasets\footnote{\url{https://www.tensorflow.org/datasets/catalog/c4}}. 

\textbf{ImageNet-1K~\cite{ImageNet}} is a large-scale image classification dataset containing approximately 1.28 million training images and 50,000 validation images annotated across 1,000 object categories. The dataset is a subset of the full ImageNet (ILSVRC-12) and is organized according to the WordNet hierarchy, with images collected from the web and manually labeled. ImageNet-1K has become a standard benchmark for training and evaluating visual recognition models, particularly in large-scale image classification and representation learning. The dataset is available online\footnote{\url{https://image-net.org}} for non-commercial research and/or educational purposes. 

\subsection{Models}
Next, we elaborate the models adopted in our evaluation. 

\textbf{LLaMA~\cite{llama}} is a family of decoder-only transformer-based large language models introduced by Meta. In this work, LLaMA is implemented using the Hugging Face \texttt{transformers} library, following the official architecture definition\footnote{\url{https://github.com/huggingface/transformers/blob/main/src/transformers/models/llama/modeling_llama.py}}. 

\textbf{Diffusion Transformer (DiT)~\cite{DiT}} is a generative image model that replaces conventional convolutional U-Nets in diffusion models with a vision transformer operating on image patches. Reference implementations of DiT are publicly available\footnote{\url{https://github.com/facebookresearch/DiT}} and released under the CC-BY-NC license.

\textbf{Residual Networks (ResNets)~\cite{ResNet}} are a class of deep convolutional neural networks, characterized by residual connections that enable the effective training of very deep architectures. ResNets have become a foundational backbone for a wide range of tasks. Standard ResNet implementations\footnote{\url{https://github.com/KaimingHe/deep-residual-networks}} are released under permissive MIT license. Our test relies on the variants used in~\cite{GCRL}, where cascaded shortcuts are added to every block of four layers. 

\subsection{Hyperparameters}
We next list the hyperparameters for running the experiments.

\begin{table}[t]
\setlength{\tabcolsep}{0.1cm}
   \caption{Hyperparameters for pre-training LLMs.}
   \label{tab:LLM-hyperparams}
   \begin{center}
     \begin{small}
         \begin{tabular}{lcccccccc}
           \toprule
           \multirow{2}{*}{Hyperparameter}  & \multicolumn{4}{c}{FullPT(+ANCRe)}  & \multicolumn{4}{c}{GaLore(+ANCRe)} \\
            &  60M & 130M  & 350M   & 1B & 60M & 130M  & 350M   & 1B \\
           \midrule
           Max seq. len. & \multicolumn{8}{c}{256} \\
           Learning rate & $5\times 10^{-3}$ & $1\times 10^{-3}$ & $1\times 10^{-3}$ & $5\times 10^{-4}$ & $1 \times 10^{-2}$ & $5 \times 10^{-2}$ & $5 \times 10^{-2}$ & $5 \times 10^{-3}$ \\
           Training steps & 10K & 20K & 60K & 100K  & 10K & 20K & 60K & 100K \\
           Warmup steps & 1K & 2K & 6K & 10K & 1K & 2K & 6K & 10K \\
           Global batch size & \multicolumn{8}{c}{512} \\
           Micro Batch size & 256 & 128 & 64 & 16 & 256 & 128 & 64 & 16 \\
           Rank & - & - & - & - & 128 & 256 & 256 & 512 \\
           GPUs & 1$\times$A100 & 1$\times$A100 & 4$\times$A100 & 2$\times$H100 & 1$\times$A100 & 1$\times$A100 & 4$\times$A100 & 2$\times$H100 \\
           \bottomrule
         \end{tabular}
     \end{small}
   \end{center}
\end{table}

\textbf{Pre-training of LLaMAs.} The experimental setup follows~\cite{ReLoRA,GaLore}, with all hyperparameters except the learning rate fixed to their default values. We use the AdamW optimizer~\cite{AdamW} with gradient accumulation. The learning rate is tuned over $\{5\times10^{-4}, 1\times10^{-3}, 5\times10^{-3}, 1\times10^{-2}, 5\times10^{-2}\}$; see Table~\ref{tab:LLM-hyperparams} for further details. For GaLore, the update projection gap and scaling factor are set to their default values of 200 and 0.25, respectively. For ANCRe, the softmax temperature $\tau$ in~\eqref{eq:softmax} is fixed to 0.01 without additional tuning.

\textbf{Pre-training of DiTs.} The experimental settings are from~\cite{DiT} without any modifications, including the learning rate. Specifically, we use images of size $256\times256$, a global batch size of 256, a learning rate of $1\times10^{-4}$, a patch size of $2\times 2$, and train for 80 epochs (i.e., 400K iterations) using BF16 precision. The ANCRe temperature is set to $\tau = 0.1$ for both DiT-S/2 and DiT-B/2. All experiments are conducted on two H100 GPUs.

\begin{table}[t]
   \caption{ANCRe temperature $\tau$ for RL tests.}
   \label{tab:RL-temperature}
   \begin{center}
     \begin{tabular}{lcc}
       \toprule
       Environment & ResNet-16  & ResNet-64 \\
       \midrule
       Humanoid & 0.01 & 0.01 \\
       Ant Big Maze & 0.01 & 0.01 \\
       Arm Push Hard & 0.1 & 0.1 \\
       Arm Binpick Hard & 0.01 & 0.1 \\
       \bottomrule
     \end{tabular}
   \end{center}
\end{table}

\textbf{RL with ResNets.} The numerical experiments follow the experimental setup of~\cite{GCRL}. Specifically, the episode length is set to 1000, and training is conducted for 100M environment steps. The number of training epochs and the batch size are 100 and 512, respectively, with 512 parallel environments. The learning rate is fixed at $3 \times 10^{-4}$. Both the actor and critic networks use a hidden dimension of 256, while the network depth is selected from $\{16, 64\}$. For ANCRe, the softmax temperature $\tau$ is coarsely tuned from $\{ 10^{-3}, 10^{-2}, 10^{-1} \}$, which can be found in Table~\ref{tab:RL-temperature}. 

At the start of pre-training, all models are randomly initialized in the default regime. ANCRe parameters $c_{ij}$ are uniformly initialized. 

\begin{table}[t]
    \caption{Mean PPL and standard deviation across random seeds.}
    \label{tab:seeds}
    \begin{center}
    \begin{tabular}{lccc}
    \toprule
    PPL & 60M & 130M & 350M \\
    \midrule
    FullPT & 30.69$\pm$0.37 & 25.04$\pm$0.03 & 19.32$\pm$0.40  \\
    FullPT+ANCRe & 29.58$\pm$0.05 & 24.50$\pm$0.03 & 18.46$\pm$0.15 \\
    \bottomrule
    \end{tabular}
    \end{center}
\end{table}

\subsection{Runs across random seeds}
\label{apdx:seeds}
To assess the robustness of evaluation metric, we train ANCRe and FullPT with three random seeds \{0,1,2\} for the 60M, 130M, and 350M LLaMA models, and report the mean and standard deviation of the resultant PPLs in Table~\ref{tab:seeds}. We observed that ANCRe outperforms FullPT for every size and every seed, and the mean PPLs closely match Table~\ref{tab:LLM-PPL}. Notably, the error bars between ANCRe and FullPT do not overlap, which confirms that the gain is consistent rather than an effect of seed noise. Moreover, the learned shortcut coefficients are highly correlated across independent runs. For the three random seeds, the mean pairwise Pearson correlation of the learned coefficients is 0.92, 0.75, and 0.81 for the 60M, 130M, and 350M models, respectively. Moreover, the most heavily weighted incoming shortcut for each layer agrees across 73\%–81\% of layers. These results indicate that the learned topology is stable rather than an artifact of random initialization.


\newpage
\section*{NeurIPS Paper Checklist}

The checklist is designed to encourage best practices for responsible machine learning research, addressing issues of reproducibility, transparency, research ethics, and societal impact. Do not remove the checklist: {\bf The papers not including the checklist will be desk rejected.} The checklist should follow the references and follow the (optional) supplemental material.  The checklist does NOT count towards the page
limit. 

Please read the checklist guidelines carefully for information on how to answer these questions. For each question in the checklist:
\begin{itemize}
    \item You should answer \answerYes{}, \answerNo{}, or \answerNA{}.
    \item \answerNA{} means either that the question is Not Applicable for that particular paper or the relevant information is Not Available.
    \item Please provide a short (1--2 sentence) justification right after your answer (even for \answerNA). 
\end{itemize}

{\bf The checklist answers are an integral part of your paper submission.} They are visible to the reviewers, area chairs, senior area chairs, and ethics reviewers. You will also be asked to include it (after eventual revisions) with the final version of your paper, and its final version will be published with the paper.

The reviewers of your paper will be asked to use the checklist as one of the factors in their evaluation. While \answerYes{} is generally preferable to \answerNo{}, it is perfectly acceptable to answer \answerNo{} provided a proper justification is given (e.g., error bars are not reported because it would be too computationally expensive'' or ``we were unable to find the license for the dataset we used''). In general, answering \answerNo{} or \answerNA{} is not grounds for rejection. While the questions are phrased in a binary way, we acknowledge that the true answer is often more nuanced, so please just use your best judgment and write a justification to elaborate. All supporting evidence can appear either in the main paper or the supplemental material, provided in appendix. If you answer \answerYes{} to a question, in the justification please point to the section(s) where related material for the question can be found.

IMPORTANT, please:
\begin{itemize}
    \item {\bf Delete this instruction block, but keep the section heading ``NeurIPS Paper Checklist"},
    \item  {\bf Keep the checklist subsection headings, questions/answers and guidelines below.}
    \item {\bf Do not modify the questions and only use the provided macros for your answers}.
\end{itemize}


\begin{enumerate}

\item {\bf Claims}
    \item[] Question: Do the main claims made in the abstract and introduction accurately reflect the paper's contributions and scope?
    \item[] Answer: \answerYes{} 
    \item[] Justification: The abstract and introduction clearly state the paper's contributions, including the theoretical analysis of residual topologies and advocated ANCRe method. They are supported by rigorous proofs in Appendix~\ref{apdx:proof} and extensive numerical experiments in Section~\ref{sec:experiments}.

\item {\bf Limitations}
    \item[] Question: Does the paper discuss the limitations of the work performed by the authors?
    \item[] Answer: \answerYes{} 
    \item[] Justification: We clearly stated in Sections~\ref{sec:case-study} and~\ref{sec:conclusion} that our theoretical analysis focuses only on deep LNNs, and future work is needed to take into account the nonlinearities in modern architectures.

\item {\bf Theory assumptions and proofs}
    \item[] Question: For each theoretical result, does the paper provide the full set of assumptions and a complete (and correct) proof?
    \item[] Answer: \answerYes{} 
    \item[] Justification: Formal theorems and detailed proofs are provided in Appendix~\ref{apdx:proof}.

    \item {\bf Experimental result reproducibility}
    \item[] Question: Does the paper fully disclose all the information needed to reproduce the main experimental results of the paper to the extent that it affects the main claims and/or conclusions of the paper (regardless of whether the code and data are provided or not)?
    \item[] Answer: \answerYes{} 
    \item[] Justification: Models, datasets, hyperparameters, and implementation details are all provided in Appendix~\ref{apdx:experiment-setups}.

\item {\bf Open access to data and code}
    \item[] Question: Does the paper provide open access to the data and code, with sufficient instructions to faithfully reproduce the main experimental results, as described in supplemental material?
    \item[] Answer: \answerYes{} 
    \item[] Justification: Links to models and datasets are included in Appendix~\ref{apdx:experiment-setups}. Codes are uploaded alongside with the submission. 

\item {\bf Experimental setting/details}
    \item[] Question: Does the paper specify all the training and test details (e.g., data splits, hyperparameters, how they were chosen, type of optimizer) necessary to understand the results?
    \item[] Answer: \answerYes{} 
    \item[] Justification: Appendix~\ref{apdx:experiment-setups} specifies training and test setups per experiment, including hyperparameters and type of optimizers. 

\item {\bf Experiment statistical significance}
    \item[] Question: Does the paper report error bars suitably and correctly defined or other appropriate information about the statistical significance of the experiments?
    \item[] Answer: \answerYes{} 
    \item[] Justification: Standard deviation under multiple random seeds is reported in Appendix~\ref{apdx:seeds}. 

\item {\bf Experiments compute resources}
    \item[] Question: For each experiment, does the paper provide sufficient information on the computer resources (type of compute workers, memory, time of execution) needed to reproduce the experiments?
    \item[] Answer: \answerYes{} 
    \item[] Justification: The computational and memory efficiency of ANCRe are provided in Table~\ref{tab:time}. Compute resource details are offered in Section~\ref{apdx:experiment-setups}.
    
\item {\bf Code of ethics}
    \item[] Question: Does the research conducted in the paper conform, in every respect, with the NeurIPS Code of Ethics \url{https://neurips.cc/public/EthicsGuidelines}?
    \item[] Answer: \answerYes{} 
    \item[] Justification: The research is a study on neural network architecture and optimization, which adheres to the NeurIPS ethical code.

\item {\bf Broader impacts}
    \item[] Question: Does the paper discuss both potential positive societal impacts and negative societal impacts of the work performed?
    \item[] Answer: \answerYes{} 
    \item[] Justification: The broader impacts of this work, including both positive and negative sides, are explicitly discussed in Appendix~\ref{apdx:broader-impact}.
    
\item {\bf Safeguards}
    \item[] Question: Does the paper describe safeguards that have been put in place for responsible release of data or models that have a high risk for misuse (e.g., pre-trained language models, image generators, or scraped datasets)?
    \item[] Answer: \answerNA{} 
    \item[] Justification: This work focuses on generic model architectural learning, rather than releasing specific model checkpoints with risks for misuse. 

\item {\bf Licenses for existing assets}
    \item[] Question: Are the creators or original owners of assets (e.g., code, data, models), used in the paper, properly credited and are the license and terms of use explicitly mentioned and properly respected?
    \item[] Answer: \answerYes{} 
    \item[] Justification: Datasets and models are properly credited, cited, and used under their corresponding licenses and terms.

\item {\bf New assets}
    \item[] Question: Are new assets introduced in the paper well documented and is the documentation provided alongside the assets?
    \item[] Answer: \answerNA{} 
    \item[] Justification: The paper focuses on the analysis and model architecture of pre-training. No new assets are introduced. 

\item {\bf Crowdsourcing and research with human subjects}
    \item[] Question: For crowdsourcing experiments and research with human subjects, does the paper include the full text of instructions given to participants and screenshots, if applicable, as well as details about compensation (if any)? 
    \item[] Answer: \answerNA{} 
    \item[] Justification: The paper does not involve crowdsourcing nor human subjects. 

\item {\bf Institutional review board (IRB) approvals or equivalent for research with human subjects}
    \item[] Question: Does the paper describe potential risks incurred by study participants, whether such risks were disclosed to the subjects, and whether Institutional Review Board (IRB) approvals (or an equivalent approval/review based on the requirements of your country or institution) were obtained?
    \item[] Answer: \answerNA{} 
    \item[] Justification: The paper does not involve crowdsourcing nor human subjects, so no IRB approval is required. 

\item {\bf Declaration of LLM usage}
    \item[] Question: Does the paper describe the usage of LLMs if it is an important, original, or non-standard component of the core methods in this research? Note that if the LLM is used only for writing, editing, or formatting purposes and does \emph{not} impact the core methodology, scientific rigor, or originality of the research, declaration is not required.
    \item[] Answer: \answerNA{} 
    \item[] Justification: LLMs are experimental subjects rather than core method components. 

\end{enumerate}

\end{document}